\def\ARXIVBUILD{1}
\documentclass{article}

\newif\ifarxiv
\ifdefined\ARXIVBUILD
  \arxivtrue
\fi

\ifarxiv
  \usepackage[preprint]{neurips_2026}
\else
  \usepackage{neurips_2026}
\fi

\usepackage[utf8]{inputenc}
\usepackage{amsmath,amssymb,amsthm,amsfonts}
\usepackage{xcolor}
\usepackage{graphicx}
\usepackage{tikz}
\usetikzlibrary{positioning,calc,arrows.meta}
\usepackage{booktabs}
\usepackage{enumitem}
\usepackage{mathtools}
\usepackage{fontspec}
\usepackage[english,provide=*]{babel}
\usepackage[breaklinks=true,hidelinks]{hyperref}
\usepackage{url}
\ifarxiv
  \definecolor{yatlink}{HTML}{246875}
  \hypersetup{colorlinks=true,linkcolor=yatlink,citecolor=yatlink,urlcolor=yatlink,
    bookmarksnumbered=true,bookmarksopen=true,bookmarksopenlevel=1,
    pdfsubject={Kernel foundations, directional structure, and neural layers},
    pdfkeywords={reproducing kernel Hilbert spaces, inverse multiquadric, Yat kernel, approximation theory}}
\fi

\newtheorem{theorem}{Theorem}
\newtheorem{proposition}{Proposition}
\newtheorem{lemma}{Lemma}
\newtheorem{corollary}{Corollary}
\newtheorem{remark}{Remark}

\newtheorem*{theorem*}{Theorem}
\newtheorem*{proposition*}{Proposition}
\newtheorem*{corollary*}{Corollary}
\newtheorem*{lemma*}{Lemma}

\babelprovide[import,onchar=ids fonts]{english}
\babelprovide[onchar=ids fonts]{tifinagh}
\ifarxiv
\babelfont[tifinagh]{rm}{DejaVuSans.ttf}
\newfontfamily\tifinaghfont{DejaVuSans.ttf}
\else
\babelfont[tifinagh]{rm}[Path=fonts/, Extension=.ttf, UprightFont=*]{ebrima}
\newfontfamily\tifinaghfont[Path=fonts/, Extension=.ttf, UprightFont=*]{ebrima}
\fi
\DeclareRobustCommand{\yat}{\text{\normalfont\tifinaghfont ⵟ}}
\DeclareRobustCommand{\kYat}{k_{\yat}}

\newcommand{\R}{\mathbb{R}}
\newcommand{\IMQ}{\mathrm{IMQ}}
\newcommand{\Sph}{\mathbb{S}}

\title{Action at a Distance: A Universal Reproducing Kernel Hilbert Space from Polynomial Alignment and IMQ Distance}

\ifarxiv
  \author{%
    Taha Bouhsine\\
    Azetta AI\\
    \texttt{taha@azetta.ai}%
  }
  \hypersetup{%
    pdftitle={Action at a Distance: A Universal Reproducing Kernel Hilbert Space from Polynomial Alignment and IMQ Distance},
    pdfauthor={Taha Bouhsine}%
  }
\else
  \author{%
    Anonymous Authors\\
    Anonymous Affiliation\\
    \texttt{anonymous@example.com}%
  }
\fi

\begin{document}
\maketitle

\begin{abstract}
Inverse-multiquadric (IMQ) kernel sections decay at infinity. Can polynomial
alignment add persistent directional responses while preserving universal
approximation? The \yat{} (Yat) kernel answers this question by multiplying IMQ
distance by a biased squared inner product. With shared positive bias and
regularization, its RKHS continuously contains the IMQ RKHS and is universal on
every compact domain. Alignment strictly enlarges the global function space:
Yat sections retain a quadratic directional trace at infinity.
We quantify a second distinction through high-dimensional activation covariance.
For a fixed number of orthogonal unit centers and fixed kernel parameters,
isotropic Gaussian inputs yield a rank-one rescaled limiting covariance for IMQ
and a full-rank limit for Yat. On the radius-$\sqrt d$ sphere, both limits are
full rank, isolating shared input-norm fluctuations as the Gaussian collapse
mechanism. The connection is also constructive: three positive-bias Yat atoms
recover one IMQ section exactly, and three are necessary at nonzero centers when
regularization is shared and bias varies across atoms. These results identify how
polynomial alignment preserves radial approximation while adding directional
structure.
\end{abstract}

\ifarxiv
\begin{center}
\small
\href{https://github.com/mlnomadpy/flaxchat}{Training code}
\quad$\cdot$\quad
\href{https://huggingface.co/mlnomad}{Model hub}
\quad$\cdot$\quad
\hyperref[app:public-resources]{Checkpoints and run links}
\quad$\cdot$\quad
\hyperref[arxiv:contents]{Contents}
\end{center}
\fi

\section{Introduction}
\label{sec:intro}

Neural units built from centers and units built from directions offer complementary
geometries. Learned-center RBF and IMQ units are symmetric kernel sections
\citep{broomhead1988rbf,moody1989rbf,park1991rbf}: they measure locality through
$\|\mathbf{x}-\mathbf{w}\|$, and a layer with shared kernel parameters is a finite
expansion in one RKHS. Their response, however, is constant on every sphere about
$\mathbf{w}$. Behind a learned feature map, purely radial responses are governed
entirely by the relation between feature-space distances and bandwidth; activation or
Gram-spectrum concentration can follow when those scales are mismatched. Standard
scalar activations such as ReLU and GELU
\citep{hendrycks2016gelu} instead use $\mathbf{w}$ as a direction through
$\sigma(\mathbf{w}^\top\mathbf{x}+b)$, but a unit is not canonically a symmetric
Mercer section on the shared input--parameter domain. This leaves a basic design
question: can a single kernel section retain an IMQ distance factor while adding a
preferred direction?

We introduce the \emph{\yat{} (Yat) kernel}
\begin{equation}
\label{eq:yat-def}
k_{\yat,b,\varepsilon}(\mathbf{w},\mathbf{x}) \;=\; \frac{(\mathbf{w}^\top\mathbf{x}+b)^2}{\|\mathbf{x}-\mathbf{w}\|^2+\varepsilon},
\qquad b\ge 0,\ \varepsilon>0.
\end{equation}
We abbreviate $k_{\yat,b,\varepsilon}$ to $k_{b,\varepsilon}$ below.
It is the product of the degree-two polynomial kernel
$(\mathbf w^\top\mathbf x+b)^2$ and the inverse-multiquadric (IMQ) kernel
$h_\varepsilon(\mathbf w,\mathbf x):=(\|\mathbf x-\mathbf w\|^2+\varepsilon)^{-1}$.
Closure of positive-semidefinite kernels under products immediately certifies this
construction~\citep{scholkopf2002learning,steinwart2008support}. Our focus is the
induced function space: which IMQ guarantees survive the product, and which functions
become available through alignment.

The first question is what the kernel preserves from IMQ. For a
fixed positive bias, $b^2h_\varepsilon\preceq k_{b,\varepsilon}$ in the Loewner order,
so the universal IMQ RKHS embeds continuously into $\mathcal H_{b,\varepsilon}$. Hence
one fixed Yat kernel is universal on every compact subset of $\mathbb R^d$ when $b>0$.
At $b=0$, every section vanishes at the origin. We prove the exact compact-domain
boundary: the zero-bias kernel is universal on a compact set if and only if that set
does not contain the origin.

The second question is what alignment changes. Expanding the numerator separates
radial, linear-alignment, and quadratic-alignment channels. Along a ray $r\mathbf u$,
a Yat section converges to $(\mathbf u^\top\mathbf w)^2$, while every finite IMQ
expansion converges to zero. The resulting far-field trace gives a concrete witness
of directional functions absent from finite radial spans.

The third question is how the atom families represent functions. For every
$t>0$,
\begin{equation}
\label{eq:intro-imq-identity}
k_{3t,\varepsilon}(\mathbf w,\cdot)
-2k_{2t,\varepsilon}(\mathbf w,\cdot)
+k_{t,\varepsilon}(\mathbf w,\cdot)
=2t^2h_\varepsilon(\mathbf w,\cdot).
\end{equation}
The three biases in~\eqref{eq:intro-imq-identity} are different, so this identity
acts across the variable-bias atom family. We prove that three atoms are also necessary
at every nonzero center. Fixed bias therefore gives continuous RKHS containment;
variable bias gives an exact, width-sharp finite representation of every IMQ section.

\paragraph{Contributions.}
The paper establishes the following kernel-specific results.
\begin{enumerate}[leftmargin=2em,itemsep=2pt,topsep=3pt]
\item \emph{Radial structure is retained.} For fixed $b>0$, Loewner domination yields
continuous IMQ-RKHS containment and therefore fixed-kernel universality,
characteristicness, and strict positive definiteness on compact domains
(Section~\ref{sec:psd}).
\item \emph{Directional structure is added.} A quadratic far-field trace gives
strict global RKHS enlargement. For a fixed number of unit centers, the Gaussian
and spherical covariance limits identify the contributions of input-norm fluctuations
and center geometry (Section~\ref{sec:extension}).
\item \emph{Finite representations are related constructively.} Three variable-bias
Yat atoms recover one IMQ atom exactly, and three are necessary at nonzero centers.
Rank-$r$ positive-semidefinite quadratic targets admit $r$-atom Yat approximations; in dimension $d\ge5$,
accurate IMQ approximation requires an explicit coefficient-variation budget
(Section~\ref{sec:universality}).
\end{enumerate}

For a layer with shared $(b,\varepsilon)$, these structural results sit inside the
standard fixed-kernel framework. If
$f=\sum_j\alpha_jk_{b,\varepsilon}(\mathbf w_j,\cdot)$, then
$\|f\|_{\mathcal H_{b,\varepsilon}}^2=\boldsymbol\alpha^\top\mathbf K\boldsymbol\alpha$.
Together with the Yat diagonal
$k_{b,\varepsilon}(\mathbf x,\mathbf x)=(\|\mathbf x\|^2+b)^2/\varepsilon$,
this identity yields an explicit layer-local Rademacher bound.

The layer-local geometry extends through composition in two complementary forms. Fixing
the preceding layers produces a prefix-dependent pullback RKHS with an exact quotient-norm
characterization. For fixed finite-depth architectures with prescribed center, bias,
and coefficient budgets and regularization bounded away from zero, a Sobolev argument
places the resulting functions in a parameter-independent ambient RKHS
(Appendix~\ref{app:layer-theory}). These are two distinct geometries: the pullback
depends on the trained prefix, while the ambient space is fixed before training.

\section{What Distance Retains: Fixed-Kernel Universality}
\label{sec:psd}

\noindent\textbf{Notation.} Throughout, $d\ge 1$, $\mathcal{X}\subset\R^d$ is compact, $b\ge 0$, $\varepsilon>0$. Write $D_\varepsilon(\mathbf{x},\mathbf{w}):=\|\mathbf{x}-\mathbf{w}\|^2+\varepsilon$ for the regularized squared distance, and use the shorthands
\[
\begin{aligned}
g_\varepsilon(\mathbf{x};\mathbf{w},b)
&:= \frac{(\mathbf{x}^\top\mathbf{w}+b)^2}{D_\varepsilon(\mathbf{x},\mathbf{w})}
&&\text{(biased Yat atom)},\\
h_\varepsilon(\mathbf{x},\mathbf{w})
&:= k_\IMQ^\varepsilon(\mathbf{x},\mathbf{w})
:= D_\varepsilon(\mathbf{x},\mathbf{w})^{-1}
&&\text{(IMQ)}.
\end{aligned}
\]
The unbiased section $\kYat(\mathbf{w},\mathbf{x})=g_\varepsilon(\mathbf{x};\mathbf{w},0)$ corresponds to $b=0$. We adopt throughout the convention that the RKHS of the identically-zero kernel is the zero subspace $\{0\}$ with trivial norm; this makes the channel decomposition of Theorem~\ref{thm:channel-decomp} and the limit $b\to 0^+$ well-defined when individual channels collapse.

\begin{proposition}[Kernel validity and inherited approximation guarantees]
\label{prop:inherited-guarantees}
Fix $b\ge0$ and $\varepsilon>0$. The kernel $k_{b,\varepsilon}$ is symmetric
positive semidefinite on $\R^d$ and continuous Mercer on every compact restriction.
It satisfies $b^2h_\varepsilon\preceq k_{b,\varepsilon}$. For $b>0$,
\[
\mathcal H_{h_\varepsilon}\subseteq\mathcal H_{b,\varepsilon},
\qquad
\|f\|_{\mathcal H_{b,\varepsilon}}\le b^{-1}\|f\|_{\mathcal H_{h_\varepsilon}}.
\]
For $b>0$, on every compact $\mathcal X\subset\R^d$ the fixed kernel is universal, its mean
embedding is injective on finite signed Borel measures, and its MMD metrizes weak
convergence of Borel probability measures. Its Gram matrix on any finite set of
distinct points is strictly positive definite.
At $b=0$, the kernel is universal on a nonempty compact $\mathcal X$ if and only if
$\mathbf0\notin\mathcal X$.
\end{proposition}

Appendix~\ref{app:distance-proofs} gives the individual statements and proofs.
The kernel-order constant $b^2$ is sharp on $\R^d$; the corresponding norm bound
scales as $1/b$. At zero bias every RKHS function vanishes at the origin.
Shared $(b,\varepsilon)$ define one kernel and one RKHS; atom-dependent parameters
define the broader span $F_\yat^{\ge0}$.

The construction separates three channels:
\[
k_{b,\varepsilon}
=\underbrace{b^2h_\varepsilon}_{\text{radial}}
+\underbrace{2b(\mathbf x^\top\mathbf w)h_\varepsilon}_{\text{linear alignment}}
+\underbrace{(\mathbf x^\top\mathbf w)^2h_\varepsilon}_{\text{quadratic alignment}}.
\]
The associated RKHS is the sum of these channel spaces, with the minimum sum of
squared component norms (Theorem~\ref{thm:channel-decomp}).
Only the quadratic channel remains at $b=0$.
Table~\ref{tab:kernel-cmp} compares the resulting structure with classical alternatives.

\begin{table}[t]
\centering
\footnotesize
\caption{Geometry and guarantees of kernel-valued hidden units. ``Fixed UAT'' means
universality of one fixed kernel on every compact domain; ``trace'' records a finite,
nonzero directional limit at infinity. The exact squared Gram norm is shared by every fixed
kernel expansion and is therefore not used as a distinguishing column.}
\label{tab:kernel-cmp}
\setlength{\tabcolsep}{2.5pt}
\begin{tabular}{@{}lcccccc@{}}
\toprule
Primitive & Locality & Alignment & \shortstack{Fixed\\UAT} & \shortstack{Each fixed\\section bounded} &
\shortstack{Nonzero\\trace} & \shortstack{IMQ\\relation} \\
\midrule
\textbf{Yat} ($b>0$) & \checkmark & \checkmark & \checkmark & \checkmark & \checkmark &
\shortstack{sub-RKHS} \\
\textbf{Yat} ($b=0$) & \checkmark & \checkmark & \shortstack{iff\\$0\notin\mathcal X$}$^\ddagger$ & \checkmark & \checkmark & --- \\
Variable-$b$ Yat span & \checkmark & \checkmark & \shortstack{family\\UAT} & \checkmark & \checkmark & \shortstack{three-atom\\exact$^\dagger$} \\
IMQ & \checkmark & --- & \checkmark & \checkmark & --- & baseline \\
RBF & \checkmark & --- & \checkmark & \checkmark & --- & --- \\
Quadratic poly. & --- & \checkmark & --- & --- & diverges & --- \\
Univ.\ dot-product$^\S$ & --- & \checkmark & \checkmark & --- & diverges & --- \\
\bottomrule
\end{tabular}
\parbox{\linewidth}{\footnotesize
$^\dagger$Exact recovery uses three atoms with different positive biases; fixed $b>0$
gives IMQ-RKHS containment. $^\ddagger$Proposition~\ref{prop:zero-bias-boundary}.
$^\S$Example:
$e^{\gamma\mathbf{x}^\top\mathbf{w}}$~\citep{smola2001regularization}. All listed
fixed kernels are PSD in their stated parameter regimes. The variable-bias row
describes a family of atoms, rather than one Mercer kernel.}
\end{table}

\section{What Alignment Adds: Nonradial Structure}
\label{sec:extension}

The Loewner embedding identifies the IMQ native space as a subspace of
$\mathcal{H}_{b,\varepsilon}$ for $b>0$. The inclusion is strict. Every function in
$\mathcal{H}_{h_\varepsilon}$ vanishes at infinity, whereas a Yat section has the
directional limit $(\mathbf{u}^\top\mathbf{w})^2$ along
$\mathbf{x}=r\mathbf{u}$. This section formalizes the distinction through an RKHS
sum decomposition and the asymptotic trace operator $T_\infty$.

We begin with the radial baseline: every function in the IMQ RKHS vanishes at
infinity.

\begin{lemma}[IMQ RKHS functions vanish at infinity]
\label{lem:imq-rkhs-c0}
Every $f\in\mathcal{H}_{h_\varepsilon}$ belongs to $C_0(\mathbb{R}^d)$, i.e.,
$f(\mathbf{x})\to0$ as $\|\mathbf{x}\|\to\infty$.
\end{lemma}

\noindent
The lemma is a \emph{global} statement on $\R^d$, distinct from compact-domain
universality: the IMQ RKHS restricts densely to $C(\mathcal{X})$ on every compact
$\mathcal{X}$~\citep{micchelli2006}, yet its ambient functions vanish at infinity.
A Yat section with nonzero directional trace therefore witnesses strict containment.

\begin{corollary}[Strict containment and non-reversibility]
\label{cor:strict-order}
For $b>0$, $\varepsilon>0$: $\mathcal{H}_{h_\varepsilon}\subsetneq\mathcal{H}_{b,\varepsilon}$.
Moreover, no constant $C>0$ satisfies $k_{b,\varepsilon}\preceq C\,h_\varepsilon$ on $\mathbb{R}^d$.
\end{corollary}

The quadratic channel carries the directional trace. The following proposition
identifies its range with the homogeneous quadratic forms on $\Sph^{d-1}$.

\begin{proposition}[Directional asymptotic trace separates Yat from IMQ]
\label{prop:trace-sep}
Let $\mathcal{D}_{T_\infty}$ denote the linear subspace of continuous functions $F:\R^d\to\R$ for
which the radial limit $\lim_{r\to\infty}F(r\mathbf{u})$ exists for every
$\mathbf{u}\in\Sph^{d-1}$, and define on $\mathcal{D}_{T_\infty}$ the asymptotic trace
operator $T_\infty F(\mathbf{u}):=\lim_{r\to\infty}F(r\mathbf{u})$. Both $F_{\yat}^{\ge 0}$
and $F_{\IMQ,\varepsilon}$ lie in $\mathcal{D}_{T_\infty}$. For every $\mathbf{w}\in\R^d$, every $b\in\R$, and every $\mathbf{u}\in\Sph^{d-1}$, the biased Yat atom satisfies
\begin{equation}
\label{eq:trace-yat}
T_\infty\,g_\varepsilon(\cdot;\mathbf{w},b)(\mathbf{u}) \;=\; (\mathbf{u}^\top\mathbf{w})^2,
\end{equation}
independent of $b$ and $\varepsilon$, and the convergence is uniform in $\mathbf{u}$. By contrast, $T_\infty F\equiv 0$ for every finite IMQ combination $F\in F_{\IMQ,\varepsilon}$. Consequently the induced linear map $T_\infty:F_{\yat}^{\ge 0}\to C(\Sph^{d-1})$ is surjective onto the space of restrictions of homogeneous quadratic forms,
\(
\mathcal{Q}_2(\Sph^{d-1})=\{\mathbf{u}\mapsto\mathbf{u}^\top\mathbf{A}\mathbf{u}:\mathbf{A}\in\mathrm{Sym}(d)\},
\)
of dimension $d(d{+}1)/2$, while $F_{\IMQ,\varepsilon}\subseteq\ker T_\infty$.
\end{proposition}

The trace is a global function-space witness. A second consequence appears in the
centered activation spectrum. For a fixed number of centers, the shared fluctuation of
$\|\mathbf{x}\|^2$ produces a rank-one leading covariance for IMQ activations. The Yat
numerator preserves the covariance of the projected coordinates.

\begin{theorem}[Activation covariance under Gaussian and fixed-norm inputs]
\label{thm:activation-covariance}
\label{thm:fixed-norm-activation-covariance}
Fix $m\in\mathbb N$, $b\ge0$, and $\varepsilon>0$. For each $d\ge m$, let
$\mathbf w_{1,d},\ldots,\mathbf w_{m,d}\in\R^d$ be unit vectors whose Gram matrices
$\mathbf C_d=[\mathbf w_{i,d}^{\top}\mathbf w_{j,d}]$ converge to
$\mathbf C\in\R^{m\times m}$. For either input law below, define
\[
\mathbf H_d:=\bigl(h_\varepsilon(\mathbf x_d,\mathbf w_{j,d})\bigr)_{j=1}^m,
\qquad
\mathbf Y_d:=\bigl(k_{b,\varepsilon}(\mathbf w_{j,d},\mathbf x_d)\bigr)_{j=1}^m.
\]
The following limits hold entrywise and in operator norm, where $\circ$ denotes
the Hadamard product:
\begin{enumerate}[leftmargin=2em,itemsep=2pt,topsep=2pt]
\item[(i)] For $\mathbf x_d\sim\mathcal N(\mathbf0,\mathbf I_d)$,
\begin{equation}
\label{eq:activation-covariance-limits}
d^3\operatorname{Cov}(\mathbf H_d)\longrightarrow2\mathbf1\mathbf1^\top,
\qquad
d^2\operatorname{Cov}(\mathbf Y_d)\longrightarrow
2(\mathbf C\circ\mathbf C)+4b^2\mathbf C.
\end{equation}
\item[(ii)] For $\mathbf x_d$ uniform on $\sqrt d\,\Sph^{d-1}$,
\begin{equation}
\label{eq:fixed-norm-activation-covariance-limits}
d^4\operatorname{Cov}(\mathbf H_d)\longrightarrow4\mathbf C,
\qquad
d^2\operatorname{Cov}(\mathbf Y_d)\longrightarrow
2(\mathbf C\circ\mathbf C)+4b^2\mathbf C.
\end{equation}
\end{enumerate}
For orthonormal centers, the Gaussian IMQ limit has eigenvalues
$(2m,0,\ldots,0)$ and the spherical IMQ limit is $4\mathbf I_m$.
The Yat limit is $(2+4b^2)\mathbf I_m$ in both cases.
Thus the limiting effective ranks are $1$ versus $m$ under Gaussian input
and $m$ for both kernels under spherical input.
\end{theorem}

Fixing the input norm removes the shared radial fluctuation that drives the leading
Gaussian IMQ covariance. The unscaled covariances vanish with $d$ in both models;
the limits describe their leading nonzero terms.
Both proofs appear in Appendix~\ref{app:covariance-proof}.

\paragraph{What alignment preserves after normalization.}
Removing input-norm fluctuations leaves a distinction between linear and quadratic
center geometry. The spherical IMQ limit is $4\mathbf C$, so its rank is the dimension
spanned by the limiting center coordinates. The Yat limit additionally contains
$2(\mathbf C\circ\mathbf C)$, the Gram matrix of their quadratic features. These
features can distinguish activation directions that are linearly dependent in the
original center coordinates.

\begin{corollary}[Rank criterion for general limiting center geometry]
\label{cor:covariance-rank-geometry}
Under the assumptions of Theorem~\ref{thm:activation-covariance}, factor
$\mathbf C=\mathbf V^\top\mathbf V$ with columns
$\mathbf v_1,\ldots,\mathbf v_m\in\R^r$, where $r=\operatorname{rank}\mathbf C$.
The common rescaled Yat limit is the Gram matrix of the lifted vectors
\[
\boldsymbol\psi_b(\mathbf v_j)
=\bigl(\sqrt2\,\mathbf v_j\otimes\mathbf v_j,\;2b\,\mathbf v_j\bigr).
\]
Consequently, its rank equals
$\dim\operatorname{span}\{\boldsymbol\psi_b(\mathbf v_j):1\le j\le m\}$,
whereas the rescaled spherical IMQ limit has rank $r$.
The Yat limit is full rank exactly when the lifted vectors are linearly independent.
Its rank is at most $\min\{m,r(r+1)/2\}$ for $b=0$, and at most
$\min\{m,r(r+1)/2+r\}$ for $b>0$.
\end{corollary}

\noindent
The proof is in Appendix~\ref{app:rank-geometry-proof}.
For a concrete example, take three centers in a fixed two-dimensional subspace,
embedded in $\R^d$ as $d$ grows:
\[
\mathbf v_1=\mathbf e_1,\qquad \mathbf v_2=\mathbf e_2,\qquad
\mathbf v_3=(\mathbf e_1+\mathbf e_2)/\sqrt2.
\]
Here $\operatorname{rank}\mathbf C=2$, but
$\det(\mathbf C\circ\mathbf C)=1/2>0$. At $b=0$, the spherical IMQ covariance
limit therefore has rank two and the Yat limit has rank three
(Figure~\ref{fig:quadratic-lift}). The rank-three conclusion also holds for $b>0$.
Normalization removes shared radial fluctuations; polynomial alignment can still
add independent covariance directions. These rank statements concern the leading
rescaled covariance limits under the theorem's fixed-center asymptotics.

\begin{figure}[ht]
\centering
\begingroup
\definecolor{liftTeal}{HTML}{2E7D8A}
\definecolor{liftAmber}{HTML}{D87F26}
\definecolor{liftGray}{HTML}{6B6B6B}
\begin{tikzpicture}[x=1cm,y=1cm,every node/.style={font=\footnotesize},
  lift/arrow/.style={-{Stealth[length=2mm]},semithick}]
\node[font=\small\bfseries] at (2.0,2.8) {Linear center geometry};
\node[font=\small\bfseries] at (9.2,2.8) {Quadratic feature geometry};
\coordinate (origin) at (0.6,0.35);
\draw[liftGray!60] (origin) -- ++(2.55,0);
\draw[liftGray!60] (origin) -- ++(0,2.0);
\draw[liftGray!50,dashed] (2.5,0.35) arc[start angle=0,end angle=90,radius=1.9];
\draw[lift/arrow,liftTeal] (origin) -- (2.5,0.35) node[right] {$\mathbf v_1$};
\draw[lift/arrow,liftTeal] (origin) -- (0.6,2.25) node[above right] {$\mathbf v_2$};
\draw[lift/arrow,liftAmber,thick] (origin) -- (1.9435,1.6935)
  node[above right] {$\mathbf v_3$};
\node at (2.0,-0.12) {$\mathbf v_3=(\mathbf v_1+\mathbf v_2)/\sqrt2$};
\node[font=\small\bfseries,liftTeal] at (2.0,-0.62) {Span dimension 2};
\draw[lift/arrow,liftGray] (4.1,1.35) -- (6.0,1.35)
  node[midway,above=4pt] {$\mathbf v\mapsto q(\mathbf v)$};
\node[align=center] at (9.2,1.35) {$\begin{array}{c|ccc}
 & q(\mathbf v_1)&q(\mathbf v_2)&q(\mathbf v_3)\\[3pt]\hline
v_1^2 &1&0&1/2\\[3pt]
v_2^2 &0&1&1/2\\[3pt]
\color{liftAmber}\sqrt2v_1v_2&0&0&\color{liftAmber}1/\sqrt2
\end{array}$};
\node[liftAmber] at (9.2,-0.12) {Mixed coordinate adds an independent direction};
\node[font=\small\bfseries,liftTeal] at (9.2,-0.62) {Span dimension 3};
\end{tikzpicture}
\endgroup
\caption{\textbf{Quadratic alignment separates linearly dependent centers.}
The three unit centers span two dimensions. Their quadratic coordinates
$q(\mathbf v)=(v_1^2,v_2^2,\sqrt2v_1v_2)$ span three: only the third center has
a nonzero mixed coordinate. Since $q(\mathbf v_i)^\top q(\mathbf v_j)
=(\mathbf v_i^\top\mathbf v_j)^2$, this gives spherical covariance-limit ranks
two for IMQ and three for Yat, already at $b=0$.}
\label{fig:quadratic-lift}
\end{figure}

For $n$ independent inputs, the empirical centered activation covariance converges
almost surely to its population counterpart as $n\to\infty$. Thus, sequentially in
$n$ and $d$, the participation-ratio effective rank under Gaussian input
$r_{\mathrm{eff}}(\mathbf{A}):=\operatorname{tr}(\mathbf{A})^2/
\operatorname{tr}(\mathbf{A}^2)$ converges to $1$ for IMQ and to $m$ for orthogonal
Yat centers; for orthogonal centers under fixed-norm spherical input it converges to
$m$ for both. These are fixed-$m$ sequential limits, not claims for growing numbers
of centers or arbitrary learned center geometry. If
$\mathbf{A}\in\R^{n\times m}$ is the activation matrix and
$\mathbf{P}_n=\mathbf{I}_n-\mathbf{1}\mathbf{1}^\top/n$, the centered sample Gram
$\mathbf{P}_n\mathbf{A}\mathbf{A}^\top\mathbf{P}_n/n$ and activation covariance
$\mathbf{A}^\top\mathbf{P}_n\mathbf{A}/n$ have the same nonzero eigenvalues. The
theorem therefore describes the leading centered Gram spectrum as well.
Figure~\ref{fig:hd-spectrum} visualizes the Gaussian and fixed-norm limits with exact
sampling; Appendix~\ref{app:spectral-protocol} gives the diagnostic protocol.

\begin{figure}[b]
\centering
\includegraphics[width=0.94\linewidth]{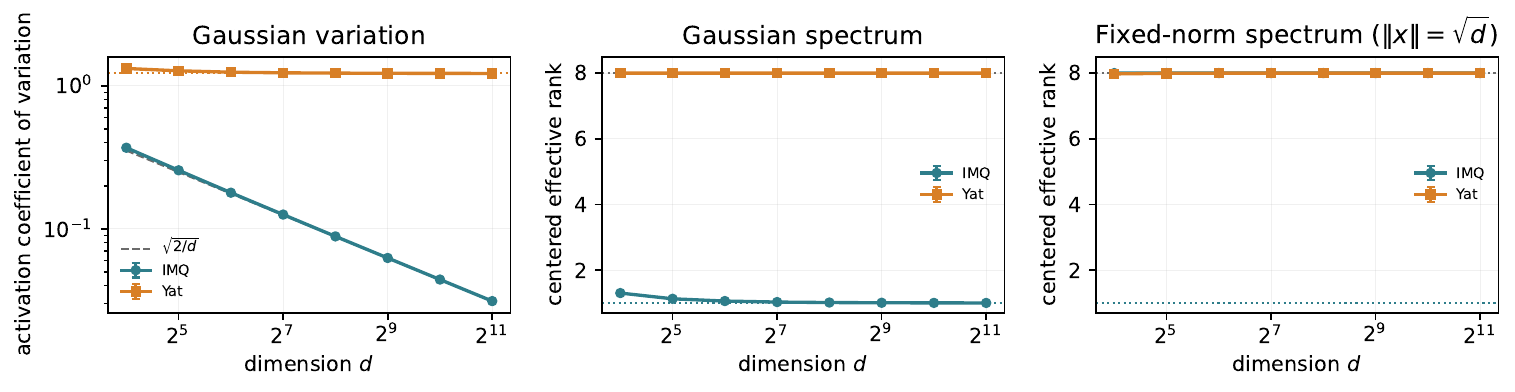}
\caption{\textbf{Input normalization determines whether IMQ covariance collapses.}
Mean $\pm$ standard deviation over five Monte Carlo seeds, with $10^5$ isotropic
Gaussian or fixed-norm spherical inputs per seed, $m=8$ orthogonal unit centers, and
$b=\varepsilon=1$. Under Gaussian input, IMQ follows the predicted $\sqrt{2/d}$
coefficient of variation and its leading centered covariance approaches rank one,
while Yat retains effective rank $8$. On the sphere $\|\mathbf x\|=\sqrt d$, both
effective ranks approach $8$, matching
Theorem~\ref{thm:fixed-norm-activation-covariance}.}
\label{fig:hd-spectrum}
\end{figure}

\paragraph{Trace diagnostic.}
The far-field formula is directly visible in a two-dimensional synthetic test. Finite
IMQ expansions with $50$ and $200$ learned centers are fitted to one Yat atom on an
annulus and evaluated farther along the aligned ray. On $r\ge400$, their mean absolute
errors are $1.008$ and $1.007$, respectively, against the limiting value $1.000$
predicted by Corollary~\ref{cor:directional-gap}. The diagnostic illustrates the
asymptotic obstruction quantitatively. Appendix~\ref{app:directional} gives the
construction and training protocol.

\section{Constructive Representation and Approximation}
\label{sec:universality}

Fixed positive bias gives continuous RKHS containment. Allowing bias to vary across
atoms gives an exact finite representation: a second finite difference recovers one
IMQ section, with sharp three-atom complexity at nonzero centers.

\begin{theorem}[Exact IMQ recovery with sharp three-atom complexity]
\label{thm:imq-embed}
\label{thm:tightness-main}
Fix $d\ge1$ and $\varepsilon>0$.
\begin{enumerate}[leftmargin=2em,itemsep=2pt,topsep=2pt]
\item[(i)] For every $\mathbf w\in\R^d$ and $h>0$, the positive-bias atoms satisfy
\begin{equation}
\label{eq:imq-identity}
g_\varepsilon(\mathbf{x};\mathbf{w},3h)
\;-\; 2\,g_\varepsilon(\mathbf{x};\mathbf{w},2h)
\;+\; g_\varepsilon(\mathbf{x};\mathbf{w},h)
\;=\; \frac{2h^2}{\|\mathbf{x}-\mathbf{w}\|^2+\varepsilon}
\;=\; 2h^2\,k_\IMQ^\varepsilon(\mathbf{x},\mathbf{w}),
\end{equation}
for every $\mathbf{x}\in\R^d$. Consequently
$F_{\IMQ,\varepsilon}\subseteq F_{\yat,\varepsilon}^{+}$, where
$F_{\yat,\varepsilon}^{+}=\mathrm{span}\{g_\varepsilon(\cdot;\mathbf{w},b):\mathbf{w}\in\R^d, b>0\}$
and $F_{\IMQ,\varepsilon}=\mathrm{span}\{k_\IMQ^\varepsilon(\cdot,\mathbf{w}):\mathbf{w}\in\R^d\}$.
The containment is strict.
\item[(ii)] For every
$\mathbf w_0\neq\mathbf0$, no linear combination of at most two biased Yat atoms,
all using this same $\varepsilon$ but with arbitrary centers and biases, equals a
nonzero scalar multiple of $k_\IMQ^\varepsilon(\cdot,\mathbf w_0)$.
\end{enumerate}
At $\mathbf w_0=\mathbf0$, one atom suffices:
$g_\varepsilon(\cdot;\mathbf0,b)/b^2=h_\varepsilon(\cdot,\mathbf0)$ for $b>0$.
\end{theorem}

Appendix~\ref{app:imq-proof} proves the identity and minimality, using irreducible
quadrics for $d\ge2$ and pole matching for $d=1$.
For approximate recovery on a compact set $\mathcal X$, put $M=\sup_{\mathbf{x}\in\mathcal X}|\mathbf{x}^\top\mathbf w|$. Then, as $B\to\infty$,
\[
\left\|B^{-2}g_\varepsilon(\cdot;\mathbf w,B)
-h_\varepsilon(\cdot,\mathbf w)\right\|_{L^\infty(\mathcal X)}
\le \frac{2M/B+M^2/B^2}{\varepsilon}\longrightarrow0.
\]
Thus exact recovery at a nonzero center requires three atoms, while one atom achieves
arbitrary compact-domain accuracy as its bias grows. Parameter budgets determine the
corresponding approximation question.

\paragraph{Approximation transfer.}
For fixed $\varepsilon>0$ and $h>0$, every $m$-atom IMQ expansion has an identical
Yat representation with at most $3m$ atoms and biases in $\{h,2h,3h\}$.
Thus an IMQ error rate $\phi(m)$ transfers to
$\phi(\lfloor M/3\rfloor)$ at Yat budget $M\ge3$. This also gives compact-domain
universality of the variable-bias span. Corollaries~\ref{cor:family-uat}
and~\ref{cor:rate-transfer} in Appendix~\ref{app:proofs-universality-extra}
state these consequences in full.

\paragraph{Quantitative separation.}\label{sec:exterior-shell}%
Exact recovery preserves finite IMQ expansions. Alignment also provides a direct
construction for quadratic targets, with a coefficient-budget obstruction for IMQ. Define the coefficient-variation class with
unrestricted centers
\[
\mathcal{V}_{\IMQ}(\Lambda):=
\Bigl\{\textstyle\sum_{j=1}^m a_j h_\varepsilon(\cdot,\mathbf{v}_j):
m\in\mathbb{N},\ \sum_j|a_j|\le\Lambda,\ \mathbf{v}_j\in\R^d\Bigr\}.
\]

\begin{theorem}[Quadratic approximation and coefficient-budget obstruction]
\label{thm:main-atom-separation}
Let $\mathbf{A}\succeq0$ have rank $r$ and largest eigenvalue
$\lambda_{\max}>0$. Fix $R>0$, $\varepsilon>0$, and
$0<\delta<\lambda_{\max}R^2/4$.
\begin{enumerate}[leftmargin=2em,itemsep=2pt,topsep=2pt]
\item[(i)] There is an unbiased Yat expansion with at most $r$ atoms that approximates
$\mathbf{x}^\top\mathbf{A}\mathbf{x}$ uniformly on $[-R,R]^d$ to error $\delta$.
\item[(ii)] Suppose $d\ge5$. If $F\in\mathcal{V}_{\IMQ}(\Lambda)$ attains the
same error, then necessarily
\begin{equation}
\label{eq:main-imq-budget}
\Lambda\ge
\sup_{0\le t\le1}
\bigl(\lambda_{\max}R^2t-\delta\bigr)_+
\bigl(R^2(1-t)+\varepsilon\bigr).
\end{equation}
In particular, taking $t=1/4$ gives
\[
\Lambda\ge
\left(\frac{\lambda_{\max}R^2}{4}-\delta\right)
\left(\frac{3R^2}{4}+\varepsilon\right).
\]
\end{enumerate}
\end{theorem}

The IMQ bound controls total coefficient variation with unrestricted width and
center locations. The Yat construction uses at most $r$ atoms, with center scale
and RKHS-norm cost depending on the target accuracy. Appendix~\ref{app:yat-native-atom-count}
gives the construction, the spherical-average proof, a worked budget example, and
the extension to mixed quadratic--IMQ targets. Related exterior-shell bounds appear
in Appendix~\ref{subsec:exterior-shell-asymptotic}.

\paragraph{Frozen-feature diagnostics.}
A no-training scale diagnostic on frozen CLIP ViT-B/32 ImageNet-1k representations
finds that at the widest tested bandwidth the median unit coefficient of variation is
$6.8\times10^{-4}$ for RBF and $1.36\times10^{-3}$ for IMQ, versus
$9.72\times10^{-2}$ for Yat. A separate, fully recorded width-$128$ Yat head with
fixed $b=\varepsilon=1$, Adam, and $30$ epochs yields a descriptive per-class norm--accuracy
correlation. Appendix~\ref{app:clip} gives the complete protocols and results.

\section{From Kernel Sections to Trained Layers}
\label{sec:capacity}

For a layer with shared $b\ge0$ and $\varepsilon>0$, write
$f=\sum_j\alpha_j k_{b,\varepsilon}(\mathbf w_j,\cdot)$ and
$\mathbf K_{ij}=k_{b,\varepsilon}(\mathbf w_i,\mathbf w_j)$. Its global RKHS norm is
\[
\|f\|_{\mathcal H_{b,\varepsilon}}^2=\boldsymbol\alpha^\top\mathbf K\boldsymbol\alpha,
\]
including repeated centers (Proposition~\ref{prop:rkhs-norm}). This quantity is
computable from trained centers and readout coefficients; restricting the input
domain gives a minimum-extension norm that can be smaller.
For multiclass logits, centering the class coefficients removes common-logit shifts
(Corollary~\ref{cor:centered-class-norm}).

For $n$ inputs of Euclidean norm at most $R$ and fixed $(b,\varepsilon)$, the kernel
diagonal bounds the empirical Rademacher complexity of a radius-$B$ RKHS ball by
$B(R^2+b)/\sqrt{n\varepsilon}$.
Appendix~\ref{app:layer-theory} gives the precise sample assumptions, data-dependent
refinement, and uniform control over prescribed learned-parameter ranges.
The shared-kernel interpretation applies to shared bias and regularization;
per-neuron biases instead select atoms from the variable-bias family.

A fixed continuous prefix induces a pullback RKHS on the original input domain.
For positive shared bias, this kernel is universal on a compact input domain exactly
when the prefix is injective: functions in the pullback space are constant on each
prefix fiber (Appendix~\ref{app:pullback-loewner}).

A width-$m$ layer in dimension $d$ has $O(md)$ forward cost.
Evaluating its squared Gram norm costs $O(m^2d)$ time and $O(m^2)$ memory,
with blockwise evaluation reducing memory.

\paragraph{Depth diagnostic.}
The kernel results give exact geometry for one layer; the deep experiment tests whether the
same primitive remains optimizable throughout a realistic stack.
We replace GELU in all $12$ MLP blocks of a matched $261$M decoder-only language model
and train on $5.22$B C4 tokens. Both Yat variants with learned output scaling obtain
lower WikiText and long-range perplexity than GELU in the three-seed runs
(Table~\ref{tab:main-lm}). Training throughput is approximately $655$K tokens/s for
Yat and $715$K for GELU on a TPU v6e-8 slice. The result establishes trainability in
this setting and measures an approximately $8\%$ systems cost. Appendix~\ref{app:lm-poc} reports all four Yat variants, six
zero-shot tasks, and the recorded compute estimates.

\begin{table}[t]
\centering
\caption{Matched $261$M causal LMs trained on $5.22$B C4 tokens; mean $\pm$ standard
deviation over three seeds. Perplexity is lower-is-better. Throughput is measured at
the training batch size on a TPU v6e-8 slice.}
\label{tab:main-lm}
\setlength{\tabcolsep}{7pt}
\begin{tabular}{lccc}
\toprule
Primitive & WikiText PPL & Long-range PPL & Tokens/s \\
\midrule
GELU & $44.23\pm1.44$ & $114.08\pm2.82$ & $715$K \\
Yat (pn$+\alpha$) & $39.04\pm1.03$ & $104.20\pm3.40$ & $655$K \\
Yat (sb$+\alpha$) & $39.07\pm0.39$ & $101.09\pm3.89$ & $655$K \\
\bottomrule
\end{tabular}
\end{table}

\section{Conclusion and Discussion}
\label{sec:discussion}

The central conclusion is that polynomial alignment can be added to an IMQ kernel
without sacrificing its radial approximation structure. For fixed positive bias, the
IMQ RKHS remains continuously embedded and supplies universality. Beyond that inherited
space, the numerator creates alignment channels with a nonzero directional trace that
finite radial expansions cannot reproduce. With variable bias, the relationship is
constructive in the reverse direction: three Yat atoms recover one IMQ atom exactly,
and the factor three is sharp away from the origin.

For a fixed number of unit centers and fixed $(b,\varepsilon)$ under isotropic
Gaussian inputs, the rescaled leading IMQ activation covariance is rank one;
orthogonal Yat centers give a full-rank limit. For the same orthogonal centers under
uniform inputs on the radius-$\sqrt d$ sphere, both rescaled limits are full rank.
The two cases of Theorem~\ref{thm:activation-covariance} identify shared input-norm fluctuations as the source of the
Gaussian leading-order IMQ collapse in this regime.
On bounded cubes, the same polynomial component approximates a rank-$r$ positive-semidefinite quadratic with
$r$ atoms, while accurate IMQ expansions obey an explicit necessary coefficient budget
without any restriction on width or center locations. These are distinct claims: the
covariance theorem is distributional, whereas the budget theorem is a uniform
approximation obstruction inside a variation class.

This function-space picture carries directly to neural layers. A
shared-$(b,\varepsilon)$ layer is a finite expansion in one fixed RKHS, with a
computable squared Gram norm and a layer-local Rademacher bound. The accompanying experiments
track the same progression: the Gaussian and frozen-feature diagnostics measure
activation and Gram-spectrum concentration, the synthetic test displays the trace
obstruction, and the frozen-CLIP diagnostics measure scale sensitivity and a descriptive
norm--accuracy association under their recorded configurations. Matched causal LMs establish end-to-end optimization at
$261$M parameters across three seeds, with approximately $8\%$ lower throughput.

The results establish a precise foundation for kernel-defined neural layers. The exact
squared Gram norm applies to a layer with shared $(b,\varepsilon)$; fixed-kernel
universality is qualitative; and the IMQ obstruction controls coefficient variation.
The covariance theorems describe Gaussian and fixed-norm spherical inputs, the
exterior-shell results describe extrapolation, and the frozen-feature experiments
measure the corresponding geometry under their recorded optimization grids. The causal-LM runs extend the study
to end-to-end trainability. Two natural next questions are whether the resulting
components support faithful interventions and whether specialized implementations can
improve latency, memory, or energy at matched quality. The first calls for held-out
intervention tests; the second calls for end-to-end systems comparisons.

Depth defines the next theoretical problem. Fixing a trained prefix yields an exact
pullback RKHS for the next layer, while the parameter-independent Sobolev RKHS contains
every fixed finite-depth stack under the prescribed parameter budgets, with rapidly
growing norm bounds. Closing the gap between
these exact local and coarse global descriptions would produce a useful capacity theory
for constrained deep Yat networks.

\section{Reproducibility and Public Artifacts}
\label{sec:reproducibility}

Appendices~\ref{app:clip}--\ref{app:lm-poc} give the complete recorded evidence set:
the frozen-CLIP diagnostics, directional-tail protocol, and causal-LM reports. The
supplement contains executable implementations, per-configuration results, and fitted
states for the retained CLIP experiments. The causal-LM experiments use the JAX/Flax
codebase FlaxChat.
\ifarxiv
The code is public at \url{https://github.com/mlnomadpy/flaxchat}; checkpoints are
archived on Hugging Face at \url{https://huggingface.co/mlnomad}, and training curves
and logged metrics are retained in W\&B at
\url{https://wandb.ai/irf-sic/flaxchat}. The public supplementary artifact maps each
available table row and seed to an immutable Hugging Face revision and W\&B training
run.
\else
The training codebase is externally public, and checkpoints, training curves, and
logged metrics are retained in external Hugging Face and W\&B projects. Identity-bearing
links are omitted from the anonymous version.
\fi
The public map covers all $15$ paper row/seed combinations with both an exact
paper-time checkpoint revision and a W\&B training run. The release also includes a declarative launch
specification for every reported configuration and per-seed evaluation histories
that reproduce the reported task metrics. Historical code, environment, and
checkpoint linkage are detailed in Appendix~\ref{app:lm-poc}.

\bibliographystyle{unsrtnat}
\bibliography{references}

\ifarxiv
\clearpage
\phantomsection
\label{arxiv:contents}
\pdfbookmark[0]{Contents}{arxivcontents}
\begingroup
\setcounter{tocdepth}{1}
\tableofcontents
\endgroup
\medskip
\noindent\textbf{Reading paths.}
For the kernel foundations, read Sections~\ref{sec:intro}--\ref{sec:capacity}
and use the appendix guide below to locate each proof.
For the empirical record, go to Appendix~\ref{app:experiments};
for released checkpoints and run histories, go to Appendix~\ref{app:public-resources}.
\clearpage
\fi

\appendix

\paragraph{Appendix guide.}
Appendix~\ref{app:background} collects notation, RKHS background, and related work.
The mathematical development then follows three questions:
what distance retains (Appendix~\ref{app:distance-proofs}),
what alignment changes (Appendix~\ref{app:alignment-proofs}), and
how the atom families represent and approximate functions
(Appendix~\ref{app:representation-proofs}).
Appendix~\ref{app:layer-theory} collects layer norms, capacity, and composition.
Experimental protocols are collected separately in Appendix~\ref{app:experiments}.
For the three main theorems, the covariance proofs are in
Appendix~\ref{app:covariance-proof}, exact recovery and its sharpness proof are in
Appendix~\ref{app:imq-proof}, and the quadratic approximation and coefficient-budget
proofs are in Appendix~\ref{app:yat-native-atom-count}.

\section{Background, Notation, and Related Work}
\label{app:background}

\subsection{Preliminaries and Notation}
\label{app:notation}

This section consolidates the notation used throughout the paper, both in the main
body and in the appendices. Generic RKHS background (Mercer's theorem, Moore--Aronszajn,
universality / characteristicness / SPD, Schur products, Laplace representation) is in
Appendix~\ref{sec:prelim}; named results from the literature we cite by label
(Steinwart--Christmann product-kernel containment, Micchelli IMQ universality) are in
Appendix~\ref{app:bg}.

\begin{table}[h]
\centering
\small
\caption{Quick reference for the notation. The table is grouped by category; longer definitions and standing assumptions follow.}
\label{tab:notation}
\setlength{\tabcolsep}{6pt}
\renewcommand{\arraystretch}{1.15}
\begin{tabular}{l l}
\toprule
Symbol & Definition \\
\midrule
\multicolumn{2}{l}{\emph{Domains and dimensions}}\\
$d, d_0, d_\ell$ & input dimension; depth-$\ell$ width \\
$\mathcal{X}\subset\R^d$ & compact input domain \\
$\Sph^{d-1}$ & unit sphere $\{\mathbf{u}:\|\mathbf{u}\|=1\}$ \\
$B_R$ & closed ball $\{\mathbf{x}:\|\mathbf{x}\|\le R\}$ \\
\midrule
\multicolumn{2}{l}{\emph{Kernels and atoms}}\\
$D_\varepsilon(\mathbf{x},\mathbf{w})$ & $\|\mathbf{x}-\mathbf{w}\|^2+\varepsilon$ (regularized squared distance) \\
$g_\varepsilon(\mathbf{x};\mathbf{w},b)$ & $(\mathbf{x}^\top\mathbf{w}+b)^2/D_\varepsilon$ (biased Yat atom) \\
$k_{b,\varepsilon}, \kYat$ & Yat kernel $g_\varepsilon$; unbiased $k_{0,\varepsilon}$ \\
$h_\varepsilon, k_\IMQ^\varepsilon$ & IMQ kernel $1/D_\varepsilon$ \\
\midrule
\multicolumn{2}{l}{\emph{RKHS objects}}\\
$\mathcal{H}_{b,\varepsilon}$, $\mathcal{H}_\yat$, $\mathcal{H}_\IMQ^\varepsilon$ & RKHSs of $k_{b,\varepsilon}$, $\kYat$, $h_\varepsilon$ \\
$\boldsymbol{\alpha}^\top\mathbf{K}\boldsymbol{\alpha}$ & squared norm of $\sum_j\alpha_j k(\mathbf{w}_j,\cdot)$ \\
$\preceq$ & Loewner PSD order on kernels / Gram matrices \\
\midrule
\multicolumn{2}{l}{\emph{Spans of atoms}}\\
$F_{\yat,\varepsilon}^{+}$ & $\mathrm{span}\{g_\varepsilon(\cdot;\mathbf{w},b):\mathbf{w}\in\R^d,\,b>0\}$ \\
$F_{\yat}^{+}, F_{\yat}^{\ge 0}$ & broader Yat spans over $b>0$ resp.\ $b\ge 0$, $\varepsilon>0$ \\
$F_{\IMQ,\varepsilon}$ & $\mathrm{span}\{h_\varepsilon(\cdot,\mathbf{w}):\mathbf{w}\in\R^d\}$ \\
$\mathcal{V}_\IMQ(\Lambda)$ & IMQ class with $\sum|a_j|\le\Lambda$ and unrestricted centers \\
$\mathcal{R}_\IMQ(\Lambda,W)$ & bounded-variation IMQ class: $\sum|a_j|\le\Lambda$, $\|\mathbf{v}_j\|\le W$ \\
$\mathcal{R}_\yat(\Lambda,W)$ & Yat analogue of $\mathcal{R}_\IMQ$ \\
\midrule
\multicolumn{2}{l}{\emph{Operators}}\\
$T_\infty F(\mathbf{u})$ & $\lim_{r\to\infty} F(r\mathbf{u})$ (asymptotic-trace operator) \\
$\mathcal{Q}_2(\Sph^{d-1})$ & quadratic forms $\mathbf{u}\mapsto\mathbf{u}^\top\mathbf{A}\mathbf{u}$ on the sphere \\
\midrule
\multicolumn{2}{l}{\emph{Network and pullback}}\\
$\Phi_\ell, T_\ell, m_\ell$ & prefix map, layer map, layer width \\
$z_{\ell,c}, \boldsymbol\alpha_{\ell,c}$ & layer-$\ell$ output coordinate $c$ and its read-out vector \\
$\widetilde{k}_\ell(\mathbf{x},\mathbf{x}')$ & $k_\ell(\Phi_{\ell-1}(\mathbf{x}),\Phi_{\ell-1}(\mathbf{x}'))$ (pullback) \\
\bottomrule
\end{tabular}
\end{table}

\paragraph{Standing assumptions.}
Throughout, $d\ge 1$, $\mathcal{X}\subset\R^d$ is compact, $b\ge 0$, $\varepsilon>0$. The
sphere is $\Sph^{d-1}=\{\mathbf{u}\in\R^d:\|\mathbf{u}\|=1\}$ and the closed ball
$B_R=\{\mathbf{x}\in\R^d:\|\mathbf{x}\|\le R\}$.

\paragraph{Kernels.} The \emph{regularized squared distance} and the three kernels we
work with:
\begin{align*}
D_\varepsilon(\mathbf{x},\mathbf{w}) &:= \|\mathbf{x}-\mathbf{w}\|^2+\varepsilon, \\
\text{biased Yat atom: }\quad g_\varepsilon(\mathbf{x};\mathbf{w},b) &:= (\mathbf{x}^\top\mathbf{w}+b)^2/D_\varepsilon(\mathbf{x},\mathbf{w}), \\
\text{Yat kernel: }\quad k_{b,\varepsilon}(\mathbf{w},\mathbf{x}) &:= g_\varepsilon(\mathbf{x};\mathbf{w},b), \quad \kYat(\mathbf{w},\mathbf{x}) := k_{0,\varepsilon}(\mathbf{w},\mathbf{x}), \\
\text{IMQ kernel: }\quad h_\varepsilon(\mathbf{x},\mathbf{w}) &:= k_\IMQ^\varepsilon(\mathbf{x},\mathbf{w}) := D_\varepsilon(\mathbf{x},\mathbf{w})^{-1}.
\end{align*}

\paragraph{RKHS objects.}
$\mathcal{H}_{b,\varepsilon}$, $\mathcal{H}_\yat$, $\mathcal{H}_{h_\varepsilon}=\mathcal{H}_\IMQ^\varepsilon$ denote the RKHSs of
$k_{b,\varepsilon}$, $\kYat$, $h_\varepsilon$ respectively. The Loewner PSD order is denoted $\preceq$.
The RKHS $\mathcal{H}_k$ associated with a continuous PSD kernel $k$ is the unique Hilbert
space of functions on its domain for which $f(\mathbf{x})=\langle f,k(\cdot,\mathbf{x})\rangle_{\mathcal{H}_k}$
holds for all $f\in\mathcal{H}_k$. The squared norm of a finite expansion
$f=\sum_j\alpha_j k(\mathbf{w}_j,\cdot)$ in the global RKHS on $\R^d$ is
$\|f\|_{\mathcal{H}_k}^2=\boldsymbol{\alpha}^\top\mathbf{K}\boldsymbol{\alpha}$, where
$\mathbf{K}_{ij}:=k(\mathbf{w}_i,\mathbf{w}_j)$. The identity also holds with repeated
centers; the Gram matrix may then be singular. Restriction to $\mathcal{X}$ gives a quotient norm, which can be smaller when centers lie outside $\mathcal{X}$.

\paragraph{Spans of atoms.}
The atom families that appear repeatedly:
\begin{align*}
F_{\yat,\varepsilon}^{+} &:= \mathrm{span}\{g_\varepsilon(\cdot;\mathbf{w},b):\mathbf{w}\in\R^d,\ b>0\}, \\
F_{\yat}^{+} &:= \mathrm{span}\{k_{b,\varepsilon}(\cdot,\mathbf{w}):\mathbf{w}\in\R^d,\ b>0,\ \varepsilon>0\}, \\
F_{\yat}^{\ge 0} &:= \mathrm{span}\{k_{b,\varepsilon}(\cdot,\mathbf{w}):\mathbf{w}\in\R^d,\ b\ge 0,\ \varepsilon>0\}, \\
F_{\IMQ,\varepsilon} &:= \mathrm{span}\{k_\IMQ^\varepsilon(\cdot,\mathbf{w}):\mathbf{w}\in\R^d\}.
\end{align*}
The coefficient-variation class and its bounded-center subclass are
\begin{align*}
\mathcal{V}_\IMQ(\Lambda)
&:=\Bigl\{\textstyle\sum_j a_j h_\varepsilon(\cdot,\mathbf{v}_j):
\sum_j|a_j|\le\Lambda,\ \mathbf{v}_j\in\R^d\Bigr\},\\
\mathcal{R}_\IMQ(\Lambda,W)
&:=\Bigl\{\textstyle\sum_j a_j h_\varepsilon(\cdot,\mathbf{v}_j):
\sum_j|a_j|\le\Lambda,\ \|\mathbf{v}_j\|\le W\Bigr\},
\end{align*}
with $\mathcal{R}_\yat(\Lambda,W)$ defined analogously for Yat atoms. The unrestricted
class $\mathcal{V}_\IMQ(\Lambda)$ is used in the bounded-domain coefficient-budget obstruction;
the center restriction in $\mathcal{R}_\IMQ(\Lambda,W)$ is used only for exterior-shell
decay.

\paragraph{Operators on the spans.}
The directional asymptotic-trace operator on the radial limit is
\[
T_\infty F(\mathbf{u}) \;:=\; \lim_{r\to\infty} F(r\mathbf{u}), \qquad \mathbf{u}\in\Sph^{d-1},
\]
defined on the linear subspace $\mathcal{D}_{T_\infty}\subset C(\R^d)$ where the limit
exists for every direction. The image of $T_\infty$ on $F_\yat^{\ge 0}$ is the space of
quadratic forms restricted to $\Sph^{d-1}$,
$\mathcal{Q}_2(\Sph^{d-1}):=\{\mathbf{u}\mapsto\mathbf{u}^\top\mathbf{A}\mathbf{u}:\mathbf{A}\in\mathrm{Sym}(d)\}$.
\paragraph{Network notation.}
For a depth-$L$ Yat stack with input $\mathbf{x}\in X\subset\R^{d_0}$ and per-layer widths $m_\ell$,
\[
\Phi_0(\mathbf{x})=\mathbf{x},\qquad \Phi_\ell(\mathbf{x})=T_\ell(\Phi_{\ell-1}(\mathbf{x})), \qquad
[T_\ell(\mathbf{z})]_c = \sum_{j=1}^{m_\ell}\alpha_{\ell,j,c}\,k_{b_\ell,\varepsilon_\ell}(\mathbf{w}_{\ell,j},\mathbf{z}),
\]
where $\Phi_{\ell-1}$ is the prefix map, $T_\ell$ the layer map, $c$ a coordinate index, and
$z_{\ell,c}:=[\Phi_\ell]_c$ the layer-$\ell$ output coordinate. The pulled-back kernel on $X$ is
$\widetilde{k}_\ell(\mathbf{x},\mathbf{x}'):=k_\ell(\Phi_{\ell-1}(\mathbf{x}),\Phi_{\ell-1}(\mathbf{x}'))$.

\paragraph{Conventions.}
$\Sph^{d-1}$ carries the uniform surface measure where unspecified; $\mathrm{Sym}(d)$ is
the space of real $d\times d$ symmetric matrices; $\mathrm{vec}(\cdot)$ is column-major
matrix vectorisation; $\langle\cdot,\cdot\rangle$ without subscript is the Euclidean
inner product on $\R^d$. The Tifinagh letter $\yat$ (Unicode U+2D5F) is used as the
kernel-specific glyph $\kYat$.

\subsection{RKHS Background and External Results}
\label{sec:prelim}

\paragraph{Reproducing kernel Hilbert spaces.}
A symmetric function $k:\mathcal{X}\times\mathcal{X}\to\R$ is
\emph{positive semidefinite (PSD)} if $\sum_{i,j}c_ic_j\,k(\mathbf{x}_i,\mathbf{x}_j)\ge 0$
for every finite set $\{\mathbf{x}_i\}\subset\mathcal{X}$ and coefficients $\{c_i\}\subset\R$.
Every continuous PSD $k$ on a compact metric space $\mathcal{X}$ is a \emph{Mercer kernel}.
More precisely, fix a finite Borel measure $\rho$ with full support on $\mathcal{X}$.
The associated integral operator on $L^2(\rho)$ is compact, self-adjoint, and positive;
Mercer's theorem gives
$k(\mathbf{x},\mathbf{y})=\sum_i\lambda_i\phi_i(\mathbf{x})\phi_i(\mathbf{y})$
with $\lambda_i\ge0$, $\{\phi_i\}$ orthonormal in $L^2(\rho)$, and absolute uniform
convergence on $\mathcal{X}\times\mathcal{X}$~\citep{mercer1909}.
The Moore--Aronszajn theorem~\citep{aronszajn1950} associates to every PSD $k$ a unique
\emph{reproducing kernel Hilbert space (RKHS)} $\mathcal{H}_k$ whose reproducing property
$f(\mathbf{x})=\langle f,k(\cdot,\mathbf{x})\rangle_{\mathcal{H}_k}$ holds for all $f\in\mathcal{H}_k$.
Two properties of RKHSs are used throughout. First, if $c>0$ and $c^2 k_1 \preceq k_2$ in the
Loewner PSD order (meaning $k_2 - c^2 k_1$ is PSD), then
$\mathcal{H}_{k_1} \subseteq \mathcal{H}_{k_2}$ with
$\|f\|_{\mathcal{H}_{k_2}} \le c^{-1}\|f\|_{\mathcal{H}_{k_1}}$~\citep{paulsen2016introduction}.
Second, the RKHS sum rule: if $k = k_1 + k_2$ then $\mathcal{H}_k = \mathcal{H}_{k_1} + \mathcal{H}_{k_2}$
with squared infimal-convolution norm
$\|f\|_{\mathcal{H}_k}^2 = \inf_{f_1+f_2=f}(\|f_1\|_{\mathcal{H}_{k_1}}^2 + \|f_2\|_{\mathcal{H}_{k_2}}^2)$.

\paragraph{Universality, characteristicness, and strict positive definiteness.}
A continuous PSD kernel on compact $\mathcal{X}$ is \emph{universal} if its RKHS is dense
in $C(\mathcal{X})$~\citep{micchelli2006,steinwart2008support}. We say $k$ is
\emph{characteristic} if the kernel mean embedding
$\mu\mapsto\int k(\cdot,\mathbf{x})\,d\mu(\mathbf{x})$ is
injective on finite signed Borel measures on $\mathcal{X}$ (this finite-signed
formulation is the one we use throughout). A kernel is \emph{strictly positive definite
(SPD)} if for every $n\ge 1$ and every set of pairwise distinct points
$\{\mathbf{x}_1,\dots,\mathbf{x}_n\}$ the Gram matrix is positive definite. The
implications we use for continuous kernels on compact $\mathcal{X}$ are
\[
\text{universal}\;\Longrightarrow\;\text{characteristic}\;\Longrightarrow\;\text{SPD},
\]
where the first arrow is~\citet{sriperumbudur2011universality} and the second follows
because if $\mathbf{c}^\top\mathbf{K}\mathbf{c}=0$ for distinct nodes then the atomic
measure $\mu=\sum_i c_i\delta_{\mathbf{x}_i}$ has zero kernel mean embedding. (Under the
finite-signed-measure definition above, characteristicness on a compact domain is in
fact equivalent to universality by a Hahn--Banach/Riesz argument. The implication used
here is universality $\Rightarrow$ injectivity on finite signed measures; injectivity
restricted to probability measures is the weaker convention.)
For a \emph{bounded continuous}
characteristic kernel on a compact metric space, $\mathrm{MMD}_k$ metrizes weak
convergence of probability measures~\citep[Thm.~23]{sriperumbudur2010hilbert}; we use this
endpoint of the chain throughout.

\paragraph{Schur products and Laplace integrals.}
Two closure properties are used repeatedly.
\textbf{(S)} The \emph{Schur product theorem}~\citep[Sec.~7.5]{horn2012matrix}: the
entrywise (Hadamard) product of two PSD kernels is PSD.
\textbf{(L)} The Laplace representation $a^{-1}=\int_0^\infty e^{-ta}\,dt$ for
$a>0$~\citep[Ch.~IV]{widder1941laplace}: the IMQ kernel $h_\varepsilon(\mathbf{x},\mathbf{w})$
$=(\|\mathbf{x}-\mathbf{w}\|^2+\varepsilon)^{-1} = \int_0^\infty e^{-t\varepsilon}e^{-t\|\mathbf{x}-\mathbf{w}\|^2}\,dt$
is a nonnegative mixture of Gaussian PSD kernels, hence PSD. Combining (S) and (L),
any Schur product $p(\mathbf{x},\mathbf{w})\cdot h_\varepsilon(\mathbf{x},\mathbf{w})$ with
a PSD polynomial kernel $p$ is PSD --- this is the core of the Yat PSD proof. For the
Yat numerator $p_b(\mathbf{x},\mathbf{w})=(\mathbf{x}^\top\mathbf{w}+b)^2$, the feature
map $\Phi_b(\mathbf{x})=(\mathbf{x}\otimes\mathbf{x},\sqrt{2b}\,\mathbf{x},b)$ requires
$b\ge 0$ for $\sqrt{2b}$ to be real; the sign condition is necessary, as
Theorem~\ref{thm:psd-mercer} exhibits a $2\times 2$ counterexample at $b=-1$.

\subsubsection{External theorems used in the proofs}
\label{app:bg}

This appendix collects the named results invoked by label in the proofs of
Sections~\ref{sec:psd}--\ref{sec:capacity}. Standard background (Mercer's theorem,
Moore--Aronszajn, universality, characteristicness, strict positive definiteness,
empirical Rademacher complexity, the Schur product theorem, and the Laplace
representation) is given in Section~\ref{sec:prelim} and not repeated here.

\begin{theorem}[Product-kernel RKHS containment~{\citep[Lemma~4.6 and Thm.~4.21]{steinwart2008support}}]
\label{thm:steinwart421}
Let $k_1,k_2$ be continuous PSD kernels on $\mathcal{X}$, and let $\mathcal{H}_1,\mathcal{H}_2$ be their RKHSs. Then the Schur product $k=k_1\cdot k_2$ is a continuous PSD kernel with RKHS $\mathcal{H}_k$ satisfying: for every $f_1\in\mathcal{H}_1$ and $f_2\in\mathcal{H}_2$, the pointwise product $f_1\cdot f_2\in\mathcal{H}_k$, with $\|f_1\cdot f_2\|_{\mathcal{H}_k}\le\|f_1\|_{\mathcal{H}_1}\|f_2\|_{\mathcal{H}_2}$.
\end{theorem}

The product feature map is
$\Phi(x)=k_1(\cdot,x)\otimes k_2(\cdot,x)$ in
$\mathcal H_1\widehat\otimes\mathcal H_2$ (Lemma~4.6).
Theorem~4.21 identifies its RKHS norm with the minimum representing-vector norm.
The vector $f_1\otimes f_2$ represents $f_1f_2$ and has norm
$\|f_1\|_{\mathcal H_1}\|f_2\|_{\mathcal H_2}$, giving the stated bound.

\begin{theorem}[Micchelli IMQ universality~{\citep[Thm.~17]{micchelli2006}}]
\label{thm:micchelli-imq}
For every $\varepsilon>0$ and every compact $\mathcal{X}\subset\R^d$, the IMQ span $F_{\IMQ,\varepsilon}=\mathrm{span}\{(\|\cdot-\mathbf{w}\|^2+\varepsilon)^{-1}:\mathbf{w}\in\R^d\}$ is dense in $C(\mathcal{X})$ in the uniform norm; equivalently, the IMQ kernel is universal on every compact $\mathcal{X}\subset\R^d$.
\end{theorem}

\subsection{Related Work}
\label{sec:related}

\paragraph{RBF and learned-center kernel networks.}
Radial-basis-function networks~\citep{broomhead1988rbf,moody1989rbf} train finite
expansions of locally-tuned radial kernels and are universal
approximators~\citep{park1991rbf}; the classical MLP approximation theory underlying
this is surveyed by~\citet{pinkus1999approximation}. The classical theory of native spaces and IMQ
universality is developed in~\citet{wendland2004scattered}. With a shared bandwidth, an
IMQ or Gaussian RBF layer is already a finite learned-center RKHS expansion admitting
the same $\boldsymbol{\alpha}^\top\mathbf{K}\boldsymbol{\alpha}$ norm and the same
diagonal-based Rademacher bound as Yat. The Yat primitive differs in two respects: it
replaces the radial numerator with the polynomial alignment factor
$(\mathbf{w}^\top\mathbf{x}+b)^2$, producing a nonzero directional far-field trace that
all pure-radial combinations lack; and the bias $b>0$ places a constant-coordinate
feature in the polynomial factor's feature map, establishing fixed-kernel universality
via kernel-order domination rather than through a variable-bandwidth construction.

\paragraph{Neural kernels and arc-cosine constructions.}
\citet{cho2009kernel} introduced arc-cosine kernels associated with threshold
activations. Both arc-cosine and Yat constructions relate neural units to
positive-semidefinite kernels, but they use different parameterizations. Yat is a
finite rational section for which the kernel order, RKHS sum decomposition, and
finite-difference identity are explicit.

\paragraph{Kernel machines and product kernels.}
Polynomially modulated radial kernels and product or nonstationary kernel constructions
are classical in kernel machines and Gaussian
processes~\citep{scholkopf2002learning,steinwart2008support}. The tensor-product construction and feature-space characterization of
Steinwart--Christmann~\citep[Lemma~4.6 and Thm.~4.21]{steinwart2008support} provide
one route to fixed-kernel universality; the kernel-order argument provides another.
Our analysis concerns the resulting rational kernel: its bias finite-difference
identity, asymptotic trace, and layer-local RKHS norm.

\paragraph{Deep composition and kernel viewpoints.}
For standard scalar-activation MLPs, layer-level capacity is controlled through spectral
or Lipschitz surrogates on the weight matrix~\citep{bartlett2017spectral,neyshabur2018pac},
which are function-level bounds rather than Hilbert-space objects attached to individual
units. The conjugate-kernel construction of \citet{daniely2016deeper} treats layered
networks through recursively defined kernels and is related to our prefix-pullback
formulation in Theorem~\ref{thm:pullback-rkhs}; deep kernel
learning~\citep{wilson2016dkl} places a GP on top of a learned feature extractor.
The representer theorem of~\citet{bohn2019representer} treats concatenations of
RKHS finite combinations and is directly aligned with this compositional viewpoint.
The pullback formalism proved in Appendix~\ref{app:pullback-loewner} is also related to the
convolutional kernel networks of~\citet{mairal2014convolutional}, where a kernel is
constructed by repeated pullback through learned feature maps.
RKHS penalties have also been used to regularize ordinary deep networks by matching
their functions to kernel-induced smoothness~\citep{bietti2019kernel}. That line attaches
an RKHS regularizer to a network function. Here the architectural primitive itself is a
kernel section, so a shared-parameter layer has an exact finite-expansion norm and the
Yat-specific finite-difference and directional identities remain accessible after training.
The Yat-specific results are the kernel inequality, the exact and minimal IMQ
representation, and the directional asymptotic trace.

\paragraph{RKHS generalization bounds.}
The single-layer Rademacher bound (Theorem~\ref{thm:rademacher}) follows
the RKHS-ball Rademacher framework of~\citet{bartlett2002}. The relationships among
universality, injective kernel mean embeddings, and strict positive definiteness depend
on the definition of characteristicness; throughout we use injectivity on finite
signed measures~\citep{sriperumbudur2011universality}. For Yat, the input-dependent
diagonal
$k_{b,\varepsilon}(\mathbf{x},\mathbf{x}) = (\|\mathbf{x}\|^2+b)^2/\varepsilon$
makes $b$ and $\varepsilon$ explicit in the bound, unlike the constant diagonal of the
Gaussian ($1$) or IMQ ($\varepsilon^{-1}$) kernels.

\paragraph{Approximation rates and shallow-network separations.}
The coefficient-budget obstruction in Theorem~\ref{thm:main-atom-separation} is distinct
from classical neural-network depth separations. \citet{eldan2016power} constructs an
approximately radial target that is efficiently represented by a depth-three network
but requires exponential width at depth two; it is not an IMQ-versus-quadratic-ridge
theorem. \citet{bach2017breaking} studies approximation in convex neural-network
function classes controlled by variation-type norms. Those results provide methodological
context, but neither is invoked to prove our IMQ budget inequality. The Barron-space MLP
capacity literature~\citep{ema2020priori} provides a related functional class for
two-layer networks. Yat instead places each shared-parameter layer in a fixed RKHS.

\paragraph{Rational and ratio-form kernels.}
The term ``rational kernel'' is also used for the sequence kernels of
\citet{cortes2004rational}, which are constructed by finite-state operations; that
framework is distinct from the rational function on $\R^d$ considered here.
Rational scalar activations are analyzed by~\citet{telgarsky2017rational},
\citet{molina2020pade}, and \citet{boulle2020rational}. Those models remain scalar
functions of $\mathbf{w}^\top\mathbf{x}$ rather than sections of a fixed kernel on
the joint input--parameter domain. Normalized polynomial kernels
~\citep{burges1998tutorial,scholkopf2002learning} divide a polynomial kernel by its
diagonal normalization, whereas Yat uses the regularized distance
$\|\mathbf{x}-\mathbf{w}\|^2+\varepsilon$. The native-space theory of IMQ kernels is
developed in~\citet{wendland2004scattered,buhmann2003rbf}.

\paragraph{Non-stationary and anisotropic kernels.}
The Yat kernel is non-stationary in the sense that
$k_{b,\varepsilon}(\mathbf{w},\mathbf{x})$ depends on $(\mathbf{w},\mathbf{x})$ rather
than only on $\mathbf{x}-\mathbf{w}$. It is jointly orthogonally invariant:
$k_{b,\varepsilon}(Q\mathbf w,Q\mathbf x)=k_{b,\varepsilon}(\mathbf w,\mathbf x)$
for every orthogonal $Q$. A fixed-center section is directionally selective through
$(\mathbf{x}^\top\mathbf{w}+b)^2$, with preferred direction $\mathbf w$. This places Yat in
a literature on non-stationary covariance constructions in Gaussian processes and
kernel machines: \citet{genton2001classes} surveys the design space of kernel families
including non-stationary and anisotropic constructions, \citet{paciorek2006nonstationary}
develop a class of non-stationary covariance functions parameterised by spatially
varying length scales, and \citet{wilson2013pattern} introduce spectral mixture kernels
that obtain alignment-style modulation through a sum over frequency components. The Yat
kernel differs from these constructions by being a single rational expression rather
than a mixture or a length-scale field, by carrying an exact algebraic IMQ embedding
(Theorem~\ref{thm:imq-embed}) absent from the spectral-mixture and varying-length-scale
families, and by admitting a closed-form layer-local RKHS norm that is generic to
shared-bandwidth Mercer-section layers.

\section{What Distance Retains: Proofs and Kernel Bounds}
\label{app:distance-proofs}

\label{app:proofs}
\subsection{Kernel certification and Loewner domination}
\label{app:psd-proof}

\begin{theorem}[PSD and Mercer for $b\ge 0$]
\label{thm:psd-mercer}
For every $b\ge 0$, $\varepsilon>0$, the kernel $k_{b,\varepsilon}$ in~\eqref{eq:yat-def}
is symmetric and positive semidefinite on $\R^d\times\R^d$, and on every
compact $\mathcal{X}\subset\R^d$ the restriction is continuous and Mercer. By
Moore--Aronszajn there exists a unique RKHS $\mathcal{H}_{b,\varepsilon}$
with reproducing kernel $k_{b,\varepsilon}$. The sign condition is
essential: for $b<0$ PSD can fail (explicit $2\times 2$ counterexample at
$b=-1$, $d=2$, nodes $\mathbf{e}_1,\mathbf{e}_2$).
\end{theorem}

\begin{proof}
Factor $k_{b,\varepsilon}=p\cdot h_\varepsilon$, where $p(\mathbf{x},\mathbf{w}):=(\mathbf{x}^\top\mathbf{w}+b)^2$ and $h_\varepsilon(\mathbf{x},\mathbf{w}):=(\|\mathbf{x}-\mathbf{w}\|^2+\varepsilon)^{-1}$.

\emph{Factor $p$.} Extend symmetrically by $\sqrt{b}$: define $\tilde{\mathbf{x}}:=(\mathbf{x},\sqrt{b})\in\R^{d+1}$ and $\tilde{\mathbf{w}}:=(\mathbf{w},\sqrt{b})\in\R^{d+1}$ (well-defined because $b\ge 0$). Then $\tilde{\mathbf{x}}^\top\tilde{\mathbf{w}}=\mathbf{x}^\top\mathbf{w}+b$, so $p(\mathbf{x},\mathbf{w})=(\tilde{\mathbf{x}}^\top\tilde{\mathbf{w}})^2$ is the square of a linear kernel in extended feature space. The linear kernel is PSD; by the Schur product theorem its Hadamard square is also PSD, so $p$ is PSD. Equivalently, $\Phi_b(\mathbf{x})=(\mathrm{vec}(\mathbf{x}\mathbf{x}^\top),\sqrt{2b}\,\mathbf{x},b)\in\R^{d^2+d+1}$ satisfies $\langle\Phi_b(\mathbf{x}),\Phi_b(\mathbf{w})\rangle=(\mathbf{x}^\top\mathbf{w})^2+2b\,\mathbf{x}^\top\mathbf{w}+b^2=(\mathbf{x}^\top\mathbf{w}+b)^2=p(\mathbf{x},\mathbf{w})$, which is real iff $b\ge 0$. The $b<0$ counter-example below shows this condition is necessary.

\emph{Factor $h_\varepsilon$.} By the Laplace identity $a^{-1}=\int_0^\infty e^{-ta}\,dt$ for $a>0$,
\begin{equation}
\label{eq:laplace-rep-app}
h_\varepsilon(\mathbf{x},\mathbf{w}) \;=\; \int_0^\infty e^{-t\varepsilon}\,e^{-t\|\mathbf{x}-\mathbf{w}\|^2}\,dt.
\end{equation}
For each $t>0$ the Gaussian $e^{-t\|\mathbf{x}-\mathbf{w}\|^2}$ is a PSD kernel and $e^{-t\varepsilon}\ge 0$. A nonnegative mixture of PSD kernels is PSD, so $h_\varepsilon$ is PSD.

\emph{Schur product.} By the Schur product theorem~\citep{horn2012matrix}, the entry-wise product of two PSD matrices is PSD. Applied to the Gram matrices of $p$ and $h_\varepsilon$ on an arbitrary finite point set, this shows $k_{b,\varepsilon}=p\cdot h_\varepsilon$ is PSD on $\R^d$.

\emph{Continuity on compact $\mathcal{X}$.} The denominator $\|\mathbf{x}-\mathbf{w}\|^2+\varepsilon$ is bounded below by $\varepsilon>0$, so $k_{b,\varepsilon}$ is a ratio of a continuous numerator and a strictly positive continuous denominator. It is therefore a continuous Mercer kernel on $\mathcal{X}\times\mathcal{X}$. Moore--Aronszajn gives the unique RKHS.

\emph{Counterexample at $b<0$.} Take $b=-1$, $d=2$, nodes $\mathbf{x}_1=\mathbf{e}_1$, $\mathbf{x}_2=\mathbf{e}_2$. The numerator Gram is
\[
[(\mathbf{x}_i^\top\mathbf{x}_j+b)^2] \;=\; \begin{pmatrix}(\,1-1)^2 & (0-1)^2 \\ (0-1)^2 & (1-1)^2\end{pmatrix}\;=\;\begin{pmatrix}0&1\\1&0\end{pmatrix},
\]
with eigenvalues $\{1,-1\}$, so the numerator is indefinite. The denominator at these nodes is $D_\varepsilon(\mathbf{x}_i,\mathbf{x}_j)\in\{\varepsilon,2+\varepsilon\}$, both positive, so the Schur product inherits the eigenvalue $-(2+\varepsilon)^{-1}<0$ and $k_{-1,\varepsilon}$ is not PSD.
\end{proof}

\begin{theorem}[Kernel-order domination]
\label{thm:kernel-order}
For every $b\ge 0$, $\varepsilon>0$,
\[
b^2\,h_\varepsilon\preceq k_{b,\varepsilon}
\]
in the Loewner PSD order on kernels. By Aronszajn's inclusion theorem~\citep{paulsen2016introduction},
\[
\mathcal{H}_{h_\varepsilon}\subseteq \mathcal{H}_{b,\varepsilon},
\qquad
\|f\|_{\mathcal{H}_{b,\varepsilon}}\le \frac{1}{b}\|f\|_{\mathcal{H}_{h_\varepsilon}}\quad(b>0).
\]
Equivalently, $B_{\mathcal{H}_{h_\varepsilon}}(R)\subseteq B_{\mathcal{H}_{b,\varepsilon}}(R/b)$.
\end{theorem}

\begin{proof}[Proof of Theorem~\ref{thm:kernel-order}]
By Theorem~\ref{thm:psd-mercer}, $k_{b,\varepsilon}$ defines the RKHS
$\mathcal{H}_{b,\varepsilon}$. Compute
\[
k_{b,\varepsilon}-b^2h_\varepsilon
=2b(\mathbf{x}^\top\mathbf{w})h_\varepsilon
+(\mathbf{x}^\top\mathbf{w})^2h_\varepsilon.
\]
The linear and quadratic dot-product kernels are PSD. Since $2b\ge0$ and
$h_\varepsilon$ is PSD, both terms on the right are PSD Schur products; their sum is
PSD by the Schur product theorem~\citep[Sec.~7.5]{horn2012matrix}. Thus
$b^2h_\varepsilon\preceq k_{b,\varepsilon}$, and Aronszajn's inclusion theorem yields
the stated RKHS containment and norm bound. At $b=0$ the kernel inequality becomes
$0\preceq k_{0,\varepsilon}$; the nontrivial IMQ-RKHS inclusion begins at $b>0$.
\end{proof}

\paragraph{Sharp constant and zero-bias endpoint.}
The constant $b^2$ is optimal on $\R^d$. At $\mathbf{x}=\mathbf{0}$,
$k_{b,\varepsilon}(\mathbf{0},\mathbf{0})/
h_\varepsilon(\mathbf{0},\mathbf{0})=b^2$; hence the $1\times1$ Gram matrix of
$k_{b,\varepsilon}-c h_\varepsilon$ is negative for every $c>b^2$. Therefore
\[
\sup\{c\ge0:c\,h_\varepsilon\preceq k_{b,\varepsilon}\}=b^2.
\]
On a compact domain separated from the origin, the optimal domain-restricted constant
may be larger.

Two related constants govern the $b\to 0^+$ behavior and should not be conflated. The kernel-domination constant $b^2$ enters squared because it compares kernels (and hence Gram matrices), while the RKHS embedding constant in Theorem~\ref{thm:kernel-order} is $1/b$, comparing norms. Consequently, norm estimates transferred through this embedding acquire a factor $1/b$, or
$1/b^2$ for squared norms, and therefore deteriorate as $b\downarrow0$.

The endpoint $b=0$ changes the RKHS. By Theorem~\ref{thm:channel-decomp},
$\mathcal{H}_{0,\varepsilon}=\mathcal{H}_{k_2}$, while the radial and
linear-alignment channels vanish. Moreover, every
$f\in\mathcal{H}_{0,\varepsilon}$ satisfies $f(\mathbf{0})=0$, whereas the IMQ section
$h_\varepsilon(\cdot,\mathbf{0})$ is nonzero there. Thus
$\mathcal{H}_{h_\varepsilon}\not\subset\mathcal{H}_{0,\varepsilon}$, even though
$\mathcal{H}_{h_\varepsilon}\hookrightarrow\mathcal{H}_{b,\varepsilon}$ for every
$b>0$. Positive bias gives universality on every compact domain; at zero bias the
exact domain dependence is given below.

\subsection{Fixed-kernel universality, characteristicness, and strict positive definiteness}
\label{app:singular-proof}

Kernel-order domination embeds the universal IMQ RKHS. The following argument gives an
equivalent product-space proof and then derives the measure-embedding and Gram-matrix
consequences.

\begin{proposition}[Fixed-kernel universality for $b_0>0$]
\label{prop:singular}
Let $\mathcal{X}\subset\R^d$ be compact and let $b_0>0$,
$\varepsilon_0>0$. Then the RKHS
$\mathcal{H}_{b_0,\varepsilon_0}$ of the fixed biased kernel
$k_{b_0,\varepsilon_0}(\mathbf{w},\mathbf{x})=(\mathbf{w}^\top\mathbf{x}+b_0)^2/(\|\mathbf{x}-\mathbf{w}\|^2+\varepsilon_0)$
is dense in $C(\mathcal{X})$ in the uniform norm. Hence the single fixed kernel
$k_{b_0,\varepsilon_0}$ is universal.
\end{proposition}

\begin{proof}
Write $k=p\cdot h$ with $h(\mathbf{x},\mathbf{w})=(\|\mathbf{x}-\mathbf{w}\|^2+\varepsilon_0)^{-1}$ the IMQ kernel and $p(\mathbf{x},\mathbf{w})=(\mathbf{x}^\top\mathbf{w}+b_0)^2$.

\emph{Step 1: the IMQ factor is universal.} The RKHS $\mathcal{H}_h$ of $h$ is dense in $C(\mathcal{X})$ for every compact $\mathcal{X}$ by the Micchelli universality theorem~\citep{micchelli2006}.

\emph{Step 2: the polynomial factor has a constant feature.} The numerator kernel $p$ has the explicit feature map
$\Phi_{b_0}(\mathbf{x})=(\mathrm{vec}(\mathbf{x}\mathbf{x}^\top),\sqrt{2b_0}\,\mathbf{x},b_0)\in\R^{d^2+d+1}$,
so $p(\mathbf{x},\mathbf{w})=\langle\Phi_{b_0}(\mathbf{x}),\Phi_{b_0}(\mathbf{w})\rangle$. One coordinate of $\Phi_{b_0}$ is the nonzero constant $b_0$.

\emph{Step 3: the polynomial RKHS contains constants.} The RKHS of $p$ is
$\mathcal{H}_p=\{\langle\mathbf{u},\Phi_{b_0}(\cdot)\rangle:\mathbf{u}\in\R^{d^2+d+1}\}$.
The dual coefficient $\mathbf{u}_1=(0,\ldots,0,1/b_0)$ yields
$\langle\mathbf{u}_1,\Phi_{b_0}(\mathbf{x})\rangle=b_0\cdot(1/b_0)=1$ for every
$\mathbf{x}$, so the constant function $\mathbf{1}_{\mathcal{X}}\in\mathcal{H}_p$ with
$\|\mathbf{1}_{\mathcal{X}}\|_{\mathcal{H}_p}\le\|\mathbf{u}_1\|=1/b_0$ (an upper bound
because $\Phi_{b_0}$ is non-minimal, so the RKHS norm is the infimum over representers).

\emph{Step 4: $\mathcal{H}_k$ contains $\mathcal{H}_h$.} By the product-kernel RKHS
containment theorem~\citep[Lemma~4.6 and Thm.~4.21]{steinwart2008support} (also
Theorem~\ref{thm:steinwart421} above), the RKHS of $k=p\cdot h$ contains every pointwise
product $f_1\cdot f_2$ with $f_1\in\mathcal{H}_h$ and $f_2\in\mathcal{H}_p$, and
$\|f_1\cdot f_2\|_{\mathcal{H}_k}\le\|f_1\|_{\mathcal{H}_h}\|f_2\|_{\mathcal{H}_p}$. Take
$f_2=\mathbf{1}_{\mathcal{X}}\in\mathcal{H}_p$ from Step~3, with
$\|\mathbf{1}_{\mathcal{X}}\|_{\mathcal{H}_p}\le 1/b_0$. Then for every
$f_1\in\mathcal{H}_h$ the pointwise product $f_1\cdot\mathbf{1}_{\mathcal{X}}=f_1$ lies
in $\mathcal{H}_k$, and
$\|f_1\|_{\mathcal{H}_k}\le\|f_1\|_{\mathcal{H}_h}\cdot(1/b_0)$. Thus the embedding
$\mathcal{H}_h\hookrightarrow\mathcal{H}_k$ has operator norm $\le 1/b_0$, recovering
the Loewner bound $\|f\|_{\mathcal{H}_{b_0,\varepsilon_0}}\le b_0^{-1}\|f\|_{\mathcal{H}_h}$
of Theorem~\ref{thm:kernel-order}. The kernel-order proof gives the same bound directly.

\emph{Step 5: density.} Let $g\in C(\mathcal{X})$ and $\eta>0$. By Step~1, there exists
$f\in\mathcal{H}_h$ with $\|g-f\|_\infty<\eta$. Step~4 gives $f\in\mathcal{H}_k$. Hence
$\mathcal{H}_k$ is dense in $C(\mathcal{X})$.
\end{proof}

\begin{proposition}[Exact zero-bias compact-domain boundary]
\label{prop:zero-bias-boundary}
Let $\mathcal X\subset\R^d$ be a nonempty compact set and fix $\varepsilon>0$.
The RKHS $\mathcal H_{0,\varepsilon}|_{\mathcal X}$ is dense in
$C(\mathcal X)$ if and only if $\mathbf0\notin\mathcal X$.
\end{proposition}

\phantomsection
\label{app:zero-bias-proof}
\begin{proof}[Proof of Proposition~\ref{prop:zero-bias-boundary}]
If $\mathbf0\in\mathcal X$, every section $k_{0,\varepsilon}(\mathbf w,\cdot)$
vanishes at $\mathbf0$, and continuity of point evaluation implies that every member
of the RKHS does too; the constant function cannot be approximated uniformly.
Now suppose $\mathbf0\notin\mathcal X$ and put $q(\mathbf x)=\|\mathbf x\|^2$.
The homogeneous quadratic kernel $(\mathbf x^\top\mathbf w)^2$ has $q$ in its
finite-dimensional RKHS. Since
$k_{0,\varepsilon}=((\mathbf x^\top\mathbf w)^2)h_\varepsilon$, the tensor-product
feature map shows that $qh$ belongs to $\mathcal H_{0,\varepsilon}$ whenever
$h\in\mathcal H_{h_\varepsilon}$. For any $f\in C(\mathcal X)$, the function
$f/q$ is continuous because $q$ is bounded below on $\mathcal X$. Universality of
the IMQ kernel supplies $h_n\in\mathcal H_{h_\varepsilon}$ with
$h_n\to f/q$ uniformly, and then $qh_n\to f$ uniformly with
$qh_n\in\mathcal H_{0,\varepsilon}$.
\end{proof}

\begin{remark}[Positive bias and origin-containing domains]
The constant feature $b_0$ in $\Phi_{b_0}$ is what makes Step~4 produce constants in $\mathcal{H}_p$. At $b_0=0$ this coordinate vanishes and the argument fails. When additionally $\mathbf{0}\in\mathcal{X}$, the simpler observation $k_{0,\varepsilon}(\mathbf{w},\mathbf{0})=0$ for every $\mathbf{w}$ forces every $f\in\mathcal{H}_{0,\varepsilon}$ to satisfy $f(\mathbf{0})=0$, ruling out density in $C(\mathcal{X})$.
\end{remark}

\subsubsection{Characteristicness and strict positive definiteness}
\label{app:char-spd}

\begin{corollary}[Characteristic kernel and SPD]
\label{cor:char-spd}
For every $b_0>0$, $\varepsilon_0>0$, and every compact
$\mathcal{X}\subset\R^d$:
(i) $k_{b_0,\varepsilon_0}$ is characteristic on $\mathcal{X}$, so the
kernel mean embedding $\mu\mapsto\int k_{b_0,\varepsilon_0}(\cdot,\mathbf{x})\,d\mu(\mathbf{x})$
is injective on finite signed Borel measures and the maximum mean
discrepancy $\mathrm{MMD}_{k_{b_0,\varepsilon_0}}$ metrizes weak convergence of
Borel probability measures on $\mathcal X$~\citep[Thm.~23]{sriperumbudur2010hilbert};
(ii) $k_{b_0,\varepsilon_0}$ is strictly positive definite in the
standard kernel-theoretic sense: for any $n\ge 1$ and any $n$ pairwise
distinct $\{\mathbf{x}_1,\dots,\mathbf{x}_n\}\subset\R^d$, the Gram
matrix $\mathbf{K}$ is strictly positive definite, hence invertible, and
the quadratic form $\boldsymbol{\alpha}^\top\mathbf{K}\boldsymbol{\alpha}$
is the squared RKHS norm, inducing a norm on coefficient space by its square root.
\end{corollary}

\begin{corollary*}[Characteristic kernel, restating Corollary~\ref{cor:char-spd}(i)]
\label{cor:characteristic-app}
For $b_0>0$, $\varepsilon_0>0$, the kernel $k_{b_0,\varepsilon_0}$ is characteristic on every compact $\mathcal{X}\subset\R^d$.
\end{corollary*}

\begin{proof}
By Proposition~\ref{prop:singular}, $k_{b_0,\varepsilon_0}$ is universal on $\mathcal{X}$. A universal continuous PSD kernel on compact $\mathcal{X}$ is characteristic: if $\mu_k:=\int k_{b_0,\varepsilon_0}(\cdot,\mathbf{x})\,d\mu(\mathbf{x})=0$ in $\mathcal{H}_{b_0,\varepsilon_0}$ for some finite signed Borel $\mu$, then by the reproducing property $\int f\,d\mu=\langle f,\mu_k\rangle=0$ for every $f\in\mathcal{H}_{b_0,\varepsilon_0}$. By density, $\int f\,d\mu=0$ for every $f\in C(\mathcal{X})$, so $\mu=0$ by the Riesz representation theorem. On Borel probability measures on $\mathcal X$, the metrization of weak convergence by $\mathrm{MMD}_{k_{b_0,\varepsilon_0}}$ follows from~\citet[Thm.~23]{sriperumbudur2010hilbert}.
\end{proof}

\begin{corollary*}[Strict positive definiteness and Gram invertibility, restating part of Cor.~\ref{cor:char-spd}]
For $b_0>0$, $\varepsilon_0>0$, and any $n$ pairwise distinct $\{\mathbf{x}_1,\dots,\mathbf{x}_n\}\subset\R^d$, the Gram matrix $\mathbf{K}_{ij}=k_{b_0,\varepsilon_0}(\mathbf{x}_i,\mathbf{x}_j)$ is strictly positive definite, hence invertible.
\end{corollary*}

\begin{proof}
Let $\mathbf{c}=(c_1,\dots,c_n)\in\R^n$ and set $\mu:=\sum_{i=1}^n c_i\delta_{\mathbf{x}_i}$, the finite signed Borel measure on $\mathcal{X}=\{\mathbf{x}_1,\dots,\mathbf{x}_n\}$ with atomic masses $c_i$. By bilinearity of the kernel mean embedding,
\[
\mathbf{c}^\top\mathbf{K}\mathbf{c}\;=\;\sum_{i,j}c_i c_j\,k_{b_0,\varepsilon_0}(\mathbf{x}_i,\mathbf{x}_j)\;=\;\Big\|\int k_{b_0,\varepsilon_0}(\cdot,\mathbf{x})\,d\mu(\mathbf{x})\Big\|_{\mathcal{H}_{b_0,\varepsilon_0}}^{2}.
\]
If $\mathbf{c}^\top\mathbf{K}\mathbf{c}=0$, the mean embedding of $\mu$ in $\mathcal{H}_{b_0,\varepsilon_0}$ is zero. Since $k_{b_0,\varepsilon_0}$ is characteristic on every compact $\mathcal{X}\subset\R^d$ (Corollary~\ref{cor:char-spd}(i)), the mean embedding is injective on finite signed Borel measures, hence $\mu=0$. Distinctness of $\{\mathbf{x}_i\}$ then forces $c_i=0$ for every $i$. Equivalently, $\mathbf{K}\succ 0$.
\end{proof}

\subsection{Section bounds and exact supremum}
\label{app:bounded-proof}

\begin{proposition}[Uniform boundedness with sharpness]
\label{prop:bounded}
For $\|\mathbf{x}\|\le R$, $\|\mathbf{w}\|\le W$, $b\ge 0$, $\varepsilon>0$,
\[
0 \;\le\; k_{b,\varepsilon}(\mathbf{w},\mathbf{x}) \;\le\; \frac{(RW+b)^2}{\varepsilon}.
\]
For every fixed $\mathbf{w}$, $b\ge 0$, $\varepsilon>0$, the section
$k_{b,\varepsilon}(\mathbf{w},\cdot)$ is globally bounded on $\R^d$, with exact supremum
\begin{equation}
\label{eq:section-supremum}
\sup_{\mathbf{x}\in\R^d}k_{b,\varepsilon}(\mathbf{w},\mathbf{x})
=\|\mathbf{w}\|^2+\frac{(\|\mathbf{w}\|^2+b)^2}{\varepsilon}.
\end{equation}
If $(\mathbf w,b)\ne(\mathbf0,0)$, equality is attained at
$\mathbf{x}^\ast=\mathbf w+\varepsilon\mathbf w/(\|\mathbf w\|^2+b)$; when
$\mathbf w=\mathbf0$ and $b>0$, this gives the maximizer $\mathbf x=\mathbf0$. For
$\mathbf w=\mathbf0$ and $b=0$, the section is identically zero.
\end{proposition}

\begin{proof}
\emph{Compact-domain bound.} By Cauchy--Schwarz, $|\mathbf{w}^\top\mathbf{x}|\le\|\mathbf{w}\|\|\mathbf{x}\|\le WR$, so $|\mathbf{w}^\top\mathbf{x}+b|\le WR+b$. Also $D_\varepsilon(\mathbf{x},\mathbf{w})\ge\varepsilon>0$, so $k_{b,\varepsilon}(\mathbf{w},\mathbf{x})\le(WR+b)^2/\varepsilon$. Nonnegativity follows from the squared numerator and positive denominator.

\emph{Exact global supremum.} Put $\mathbf y=\mathbf x-\mathbf w$ and
$c=\|\mathbf w\|^2+b$. Cauchy--Schwarz in $\R^{d+1}$ gives
\[
(\mathbf w^\top\mathbf y+c)^2
=\left\langle(\mathbf w,c/\sqrt\varepsilon),
(\mathbf y,\sqrt\varepsilon)\right\rangle^2
\le\left(\|\mathbf w\|^2+\frac{c^2}{\varepsilon}\right)
(\|\mathbf y\|^2+\varepsilon).
\]
Division by $\|\mathbf y\|^2+\varepsilon$ proves~\eqref{eq:section-supremum}.
If $c>0$, equality holds when
$(\mathbf y,\sqrt\varepsilon)=(\varepsilon/c)(\mathbf w,c/\sqrt\varepsilon)$,
namely at $\mathbf y=\varepsilon\mathbf w/c$. If $c=0$, nonnegativity of $b$ forces
$\mathbf w=\mathbf0$ and $b=0$, and the section is zero.
\end{proof}

The diagonal value $k_{b,\varepsilon}(\mathbf{x},\mathbf{x})=(\|\mathbf{x}\|^2+b)^2/\varepsilon$ used in the Rademacher proof is the special case $\mathbf{w}=\mathbf{x}$ of the kernel definition. Classical scalar activations such as ReLU and GELU are unbounded along positive-projection directions. They are not symmetric Mercer sections on the shared input--parameter domain, so this particular kernel-diagonal Rademacher calculation does not apply to them.

\section{What Alignment Changes: Proofs and Extensions}
\label{app:alignment-proofs}

\subsection{Directional structure and quantitative separation}
\label{app:proofs-extension}

The proofs proceed from the radial baseline to the Yat separation: IMQ-RKHS decay,
channel decomposition, asymptotic trace, strict RKHS containment, and the resulting
far-field separation.

\begin{proof}[Proof of Lemma~\ref{lem:imq-rkhs-c0}]
Every finite-span element $f_m(\mathbf{x})=\sum_{j=1}^m a_j h_\varepsilon(\mathbf{x},\mathbf{w}_j)$
vanishes at infinity since each atom decays as $O(\|\mathbf{x}\|^{-2})$. The finite span is
dense in $\mathcal{H}_{h_\varepsilon}$; if $f_m\to f$ in RKHS norm, the reproducing property gives
\[
\sup_{\mathbf{x}\in\mathbb{R}^d}|f_m(\mathbf{x})-f(\mathbf{x})|
\le
\sup_\mathbf{x}\sqrt{h_\varepsilon(\mathbf{x},\mathbf{x})}\,\|f_m-f\|_{\mathcal{H}_{h_\varepsilon}}
=\varepsilon^{-1/2}\|f_m-f\|_{\mathcal{H}_{h_\varepsilon}},
\]
so $f_m\to f$ uniformly. Since $C_0(\mathbb{R}^d)$ is closed under uniform limits,
$f\in C_0(\mathbb{R}^d)$.
\end{proof}

\begin{theorem}[Yat RKHS channel decomposition]
\label{thm:channel-decomp}
Write $k_{b,\varepsilon}=k_0+k_1+k_2$ where
$k_0:=b^2 h_\varepsilon$,
$k_1:=2b(\mathbf{x}^\top\mathbf{w})h_\varepsilon$,
$k_2:=(\mathbf{x}^\top\mathbf{w})^2 h_\varepsilon$.
By the RKHS sum rule~\citep{paulsen2016introduction,berlinet2004rkhs},
\[
\mathcal{H}_{b,\varepsilon}=\mathcal{H}_{k_0}+\mathcal{H}_{k_1}+\mathcal{H}_{k_2},
\]
with squared infimal-convolution norm
\[
\|f\|_{\mathcal{H}_{b,\varepsilon}}^2
=\inf_{\substack{f_0+f_1+f_2=f\\f_i\in\mathcal{H}_{k_i}}}
\bigl(\|f_0\|_{\mathcal{H}_{k_0}}^2+\|f_1\|_{\mathcal{H}_{k_1}}^2+\|f_2\|_{\mathcal{H}_{k_2}}^2\bigr).
\]
\end{theorem}

\begin{proof}[Proof of Theorem~\ref{thm:channel-decomp}]
The RKHS sum rule states that $\mathcal{H}_{k_1+k_2}=\mathcal{H}_{k_1}+\mathcal{H}_{k_2}$
with the infimal-convolution norm~\citep{paulsen2016introduction}; extended to three summands by
induction. Each $k_i$ is PSD: $k_0=b^2h_\varepsilon$ since $h_\varepsilon$ is PSD and $b^2\ge0$;
$k_1=2b(\mathbf{x}^\top\mathbf{w})h_\varepsilon$ because the dot-product kernel
$\mathbf{x}^\top\mathbf{w}$ is PSD ($\sum_{ij}c_ic_j\mathbf{x}_i^\top\mathbf{x}_j=\|\sum_ic_i\mathbf{x}_i\|^2\ge0$),
$h_\varepsilon$ is PSD, their Schur product is PSD by the Schur product theorem, and $2b\ge0$
(the channel vanishes at $b=0$); $k_2=(\mathbf{x}^\top\mathbf{w})^2h_\varepsilon$ by the Schur
product theorem applied to $(\mathbf{x}^\top\mathbf{w})^2$ and $h_\varepsilon$, both PSD.
The RKHS sum rule therefore applies.
\end{proof}

The three channels are: the \emph{radial IMQ channel} ($k_0=b^2h_\varepsilon$; as sets
$\mathcal{H}_{k_0}=\mathcal{H}_{h_\varepsilon}$ with rescaled norm
$\|f\|_{\mathcal{H}_{k_0}}=b^{-1}\|f\|_{\mathcal{H}_{h_\varepsilon}}$ (constant-rescaling identity for RKHS norms,~\citealt{paulsen2016introduction}), present when $b>0$),
the \emph{linear-alignment IMQ channel} ($k_1=2b(\mathbf{x}^\top\mathbf{w})h_\varepsilon$,
vanishes at $b=0$), and the \emph{quadratic-alignment IMQ channel}
($k_2=(\mathbf{x}^\top\mathbf{w})^2h_\varepsilon$, carrying the far-field trace, present for all $b\ge0$).
Under the zero-kernel convention fixed above, both $k_0$ and $k_1$ vanish at $b=0$:
the RKHS loses the IMQ subspace and the constant coordinate, which is the
degeneracy responsible for failure of fixed-kernel universality on domains containing
the origin. Equivalently, $k_{0,\varepsilon}(\mathbf{0},\mathbf{0})=0$ forces every
$f\in\mathcal{H}_{0,\varepsilon}$ to satisfy $f(\mathbf{0})=0$. The channel
decomposition and the vanishing diagonal thus identify the same endpoint degeneracy.

\begin{proof}[Proof of Proposition~\ref{prop:trace-sep}]
\emph{Yat limit.} For~\eqref{eq:trace-yat}, write
\[
g_\varepsilon(r\mathbf{u};\mathbf{w},b)
\;=\; \frac{(r\,\mathbf{u}^\top\mathbf{w}+b)^2}{r^2-2r\,\mathbf{u}^\top\mathbf{w}+\|\mathbf{w}\|^2+\varepsilon}
\;=\; \frac{(\mathbf{u}^\top\mathbf{w}+b/r)^2}{1-2(\mathbf{u}^\top\mathbf{w})/r+(\|\mathbf{w}\|^2+\varepsilon)/r^2}
\;\xrightarrow[r\to\infty]{}\; (\mathbf{u}^\top\mathbf{w})^2,
\]
uniformly in $\mathbf{u}\in\Sph^{d-1}$.

\emph{IMQ decay.} For each fixed center $\mathbf{w}_i$,
$|k_\IMQ^\varepsilon(r\mathbf{u},\mathbf{w}_i)|\le(r^2-2r\|\mathbf{w}_i\|+\|\mathbf{w}_i\|^2+\varepsilon)^{-1}=O(r^{-2})$ uniformly on the sphere. Hence every finite combination $F=\sum_i c_i k_\IMQ^\varepsilon(\cdot,\mathbf{w}_i)$ satisfies $T_\infty F\equiv 0$.

\emph{Range of the trace.} The image of $T_\infty$ contains every rank-one quadratic form $(\mathbf{u}^\top\mathbf{w})^2=\mathbf{u}^\top(\mathbf{w}\mathbf{w}^\top)\mathbf{u}$. The symmetric rank-one tensors $\mathbf{w}\mathbf{w}^\top$ span $\mathrm{Sym}(d)$: the diagonal generators are $\mathbf{e}_i\mathbf{e}_i^\top$, and the off-diagonal generators follow from $\tfrac12((\mathbf{e}_i+\mathbf{e}_j)(\mathbf{e}_i+\mathbf{e}_j)^\top-\mathbf{e}_i\mathbf{e}_i^\top-\mathbf{e}_j\mathbf{e}_j^\top)$. Therefore $\mathrm{Im}(T_\infty)=\mathcal{Q}_2(\Sph^{d-1})$.
\end{proof}

\begin{proof}[Proof of Corollary~\ref{cor:strict-order}]
The inclusion $\mathcal{H}_{h_\varepsilon}\subseteq\mathcal{H}_{b,\varepsilon}$ is
Theorem~\ref{thm:kernel-order}. For strictness, choose $\mathbf{w}\neq\mathbf{0}$ and any
$\mathbf{u}$ with $\mathbf{u}^\top\mathbf{w}\neq0$: the far-field trace gives
$\lim_{r\to\infty}k_{b,\varepsilon}(\mathbf{w},r\mathbf{u})=(\mathbf{u}^\top\mathbf{w})^2\neq0$,
so $k_{b,\varepsilon}(\mathbf{w},\cdot)\notin C_0(\mathbb{R}^d)$ while, by
Lemma~\ref{lem:imq-rkhs-c0}, $\mathcal{H}_{h_\varepsilon}\subset C_0(\mathbb{R}^d)$.
Hence $k_{b,\varepsilon}(\mathbf{w},\cdot)\in\mathcal{H}_{b,\varepsilon}\setminus\mathcal{H}_{h_\varepsilon}$.
Non-reversibility: if $k_{b,\varepsilon}\preceq C\,h_\varepsilon$, Aronszajn inclusion would give
$\mathcal{H}_{b,\varepsilon}\subseteq\mathcal{H}_{h_\varepsilon}\subset C_0(\mathbb{R}^d)$, contradicting
the existence of Yat sections not vanishing at infinity.
\end{proof}

\begin{corollary}[Concrete large-radius regression gap]
\label{cor:directional-gap}
Fix $\mathbf{w}\neq\mathbf{0}$, $b\in\R$, $\varepsilon>0$, and choose
$\mathbf{u}\in\Sph^{d-1}$ with $a:=(\mathbf{u}^\top\mathbf{w})^2>0$.
Let $f(\mathbf{x})=g_\varepsilon(\mathbf{x};\mathbf{w},b)$. Then for
every finite IMQ combination $F\in F_{\IMQ,\varepsilon}$,
\[
|f(r\mathbf{u})-F(r\mathbf{u})|\;\xrightarrow[r\to\infty]{}\; a.
\]
Hence for every $\eta\in(0,a)$ there exists $R_\eta$ such that
$\sup_{r\ge R_\eta}|f(r\mathbf{u})-F(r\mathbf{u})|\ge a-\eta$.
In particular, a single Yat atom carries an $O(1)$ directional signal on
a sufficiently large exterior ray that no finite IMQ expansion can
uniformly match there.
\end{corollary}

\begin{proof}[Proof of Corollary~\ref{cor:directional-gap}]
Proposition~\ref{prop:trace-sep} gives $f(r\mathbf{u})\to a$ and
$F(r\mathbf{u})\to 0$, hence $f(r\mathbf{u})-F(r\mathbf{u})\to a$.
Given $\eta\in(0,a)$, choose $R_\eta$ so that for all $r\ge R_\eta$:
$|f(r\mathbf{u})-a|<\eta/2$ and $|F(r\mathbf{u})|<\eta/2$. Then
$|f(r\mathbf{u})-F(r\mathbf{u})|\ge a-\eta$, and the supremum over
$r\ge R_\eta$ preserves the bound.
\end{proof}

\begin{proposition}[Far-field gradient asymptotics]
\label{prop:farfield-gradient}
Fix $\mathbf{w}\in\R^d\setminus\{\mathbf{0}\}$, $b\ge 0$, $\varepsilon>0$, and a unit
direction $\mathbf{u}\in\Sph^{d-1}$ with $\mathbf{u}^\top\mathbf{w}\neq 0$. Along the ray
$\mathbf{x}=r\mathbf{u}$, the center-gradients of the IMQ section
$h_\varepsilon(\mathbf{x},\mathbf{w})$ and the Yat section
$k_{b,\varepsilon}(\mathbf{w},\mathbf{x})$ satisfy
\[
\bigl\|\nabla_\mathbf{w}\,h_\varepsilon(r\mathbf{u},\mathbf{w})\bigr\|=\Theta(r^{-3}),
\qquad
\lim_{r\to\infty}\nabla_\mathbf{w}\,k_{b,\varepsilon}(\mathbf{w},r\mathbf{u})
=2(\mathbf{u}^\top\mathbf{w})\,\mathbf{u}.
\]
\end{proposition}

\begin{remark}
The proposition records an asymptotic difference in the center-gradient of a single Yat
atom versus a single IMQ atom at large $\|\mathbf{x}\|$, evaluated in isolation. It is a
statement about the kernel itself, not about gradients of trained networks: in a finite
expansion the per-atom gradients combine through the readout coefficients
$\boldsymbol{\alpha}$, and far-field signal can be informative or uninformative depending
on where the data actually live.
\end{remark}

\begin{proof}[Proof of Proposition~\ref{prop:farfield-gradient}]
For the IMQ atom, $\nabla_\mathbf{w} h_\varepsilon=2(\mathbf{x}-\mathbf{w})/(\|\mathbf{x}-\mathbf{w}\|^2+\varepsilon)^2$.
With $\mathbf{x}=r\mathbf{u}$, the numerator $\|r\mathbf{u}-\mathbf{w}\|=r+O(1)$ and the
denominator $(r^2-2r\,\mathbf{u}^\top\mathbf{w}+\|\mathbf{w}\|^2+\varepsilon)^2=r^4+O(r^3)$, so
$\|\nabla_\mathbf{w}h_\varepsilon\|=2r^{-3}(1+O(r^{-1}))=\Theta(r^{-3})$.

For the Yat atom, set $N(\mathbf{x}):=(\mathbf{w}^\top\mathbf{x}+b)^2$ and
$D(\mathbf{x}):=\|\mathbf{x}-\mathbf{w}\|^2+\varepsilon$. The quotient rule gives
\[
\nabla_\mathbf{w}\,k_{b,\varepsilon}(\mathbf{w},\mathbf{x})
=\frac{2(\mathbf{w}^\top\mathbf{x}+b)\,\mathbf{x}}{D(\mathbf{x})}
+\frac{2(\mathbf{w}^\top\mathbf{x}+b)^2(\mathbf{x}-\mathbf{w})}{D(\mathbf{x})^2}.
\]
Substituting $\mathbf{x}=r\mathbf{u}$, the first term equals
$2(r\mathbf{u}^\top\mathbf{w}+b)\,r\mathbf{u}/D(r\mathbf{u})$. The leading-order behavior
of numerator and denominator is $2r^2(\mathbf{u}^\top\mathbf{w})\mathbf{u}+O(r)$ over
$r^2+O(r)$, giving the limit $2(\mathbf{u}^\top\mathbf{w})\mathbf{u}$. The second term has
numerator $2(r\mathbf{u}^\top\mathbf{w}+b)^2(r\mathbf{u}-\mathbf{w})=O(r^3)$ over
denominator $D(r\mathbf{u})^2=O(r^4)$, giving $O(r^{-1})\to 0$. Summing yields the stated
limit.
\end{proof}

\subsection{Gaussian and fixed-norm activation covariance}
\label{app:covariance-proof}
The input distribution determines the leading radial covariance. Under isotropic
Gaussian input, IMQ units share the fluctuation of $\|\mathbf x\|^2$ and their leading
covariance is rank one, while the Yat numerator retains order-one relative variation.
On the fixed-norm sphere, the shared radial fluctuation vanishes and center-specific
projections give both kernels full-rank limiting covariance at orthogonal centers.
Related high-dimensional linearization and
low-degree phenomena are well documented: \citet{elkaroui2010kernel} shows that broad
classes of radial kernel random matrices become essentially linear on isotropic
high-dimensional inputs, while \citet{ghorbani2019linearized} derives low-degree
approximation limits for rotationally invariant random-feature and neural-tangent
models. The proposition below gives a different, activation-level statement for the
present IMQ/Yat pair: under isotropic Gaussian inputs, the Yat numerator
$(\mathbf{x}^\top\mathbf{w}+b)^2$ depends on a single 1-D projection and resists the
collapse.

\begin{proposition}[High-dimensional variance preservation]
\label{prop:hd-variance}
Let $\mathbf{x}\sim\mathcal{N}(\mathbf{0},\mathbf{I}_d)$, let $\mathbf{w}\in\R^d$ have unit
norm $\|\mathbf{w}\|=1$, and fix $b\ge 0$, $\varepsilon>0$. For a positive random variable
$Y$, write $\mathrm{CV}[Y]:=\sqrt{\mathrm{Var}[Y]}/\mathbb{E}[Y]$. As $d\to\infty$,
\[
\sqrt d\,\mathrm{CV}[h_\varepsilon(\mathbf{x},\mathbf{w})]\longrightarrow\sqrt2,
\qquad
\mathrm{CV}[k_{b,\varepsilon}(\mathbf{w},\mathbf{x})]
\longrightarrow\frac{\sqrt{2+4b^2}}{1+b^2}.
\]
\end{proposition}

\begin{proof}
By rotational invariance of the standard Gaussian, take $\mathbf{w}=\mathbf{e}_1$. Then
$\mathbf{w}^\top\mathbf{x}=x_1\sim\mathcal{N}(0,1)$, and the regularized distance
$D:=\|\mathbf{x}-\mathbf{w}\|^2+\varepsilon=(x_1-1)^2+\sum_{i=2}^d x_i^2+\varepsilon$ satisfies
$D-\varepsilon\sim\chi_d^2(\lambda=1)$ (non-central chi-squared, $d$ degrees of freedom,
non-centrality $\|\mathbf{w}\|^2=1$). Standard moment formulas give
$\mu_D:=\mathbb{E}[D]=d+1+\varepsilon$ and $\sigma_D^2:=\mathrm{Var}[D]=2d+4$.

\emph{Inverse-distance concentration.}
The non-central chi-square law gives
\[
\frac{D}{d}\longrightarrow 1
\quad\text{and}\quad
\frac{D-\mu_D}{\sqrt d}\Longrightarrow \mathcal{N}(0,2).
\]
The required moment convergence follows from an explicit event split. On
$\mathcal{E}_d:=\{D\ge d/2\}$, $d/D\le2$. On $\mathcal{E}_d^c$, the deterministic
floor $D\ge\varepsilon$ and a non-central chi-square lower-tail bound give, for every
fixed $q<\infty$,
$\mathbb{E}[(d/D)^q\mathbf{1}_{\mathcal{E}_d^c}]\le(d/\varepsilon)^q e^{-cd}=o(1)$.
Since $D/d\to1$ in probability, this proves $d/D\to1$ in $L^q$ for every fixed
$q$. Moreover, the standardized non-central chi-square variables
$(D-\mu_D)/\sqrt d$ have uniformly bounded moments of every fixed order. These two
facts make the squared variables in the following delta-method limit uniformly
integrable:
\[
d^{3/2}\left(\frac1D-\frac1{\mu_D}\right)
=-\frac{D-\mu_D}{\sqrt d}\frac{d^2}{D\mu_D}
\Longrightarrow \mathcal{N}(0,2),
\]
with convergence of second moments. Consequently,
\[
\mathbb{E}[h_\varepsilon]=d^{-1}(1+o(1)),
\qquad
\mathrm{Var}[h_\varepsilon]=2d^{-3}(1+o(1)),
\]
and hence $\sqrt d\,\mathrm{CV}[h_\varepsilon]\to\sqrt2$.

\emph{Yat unit.}
Let $N:=(x_1+b)^2$. Then $k_{b,\varepsilon}=N/D$,
$\mathbb{E}N=1+b^2$, and $\mathrm{Var}(N)=2+4b^2>0$.
The $L^4$ convergence $d/D\to1$ and H\"older's inequality give
$\mathbb{E}[N^2(d/D-1)^2]\le(\mathbb{E}N^4)^{1/2}
(\mathbb{E}|d/D-1|^4)^{1/2}\to0$. Hence
\[
\frac{dN}{D}\longrightarrow N\qquad\text{in }L^2.
\]
No independence between $N$ and $D$ is required. Taking first and second moments,
\[
d\,\mathbb{E}[k_{b,\varepsilon}]\longrightarrow 1+b^2,
\qquad
d^2\,\mathrm{Var}[k_{b,\varepsilon}]\longrightarrow 2+4b^2.
\]
Therefore
$\mathrm{CV}[k_{b,\varepsilon}]\to\sqrt{2+4b^2}/(1+b^2)$.
\end{proof}

\begin{proof}[Proof of Theorem~\ref{thm:activation-covariance}(i)]
Write $Z_{j,d}:=\mathbf{w}_{j,d}^\top\mathbf{x}_d$,
$S_d:=\|\mathbf{x}_d\|^2$, and
$D_{j,d}:=S_d-2Z_{j,d}+1+\varepsilon$. The vector
$\mathbf{Z}_d=(Z_{1,d},\ldots,Z_{m,d})$ is Gaussian with covariance
$\mathbf{C}_d$ and therefore converges in distribution, with convergence of every
fixed joint moment, to $\mathbf{Z}\sim\mathcal{N}(\mathbf{0},\mathbf{C})$.

\emph{Radial covariance.}
Let $\mu_d:=d+1+\varepsilon$ and $G_d:=(S_d-d)/\sqrt d$. The same lower-tail
event split used above, applied to the fixed collection $D_{1,d},\ldots,D_{m,d}$,
gives $d/D_{j,d}\to1$ in every fixed $L^q$. Since
\[
\frac{D_{j,d}-\mu_d}{\sqrt d}=G_d-\frac{2Z_{j,d}}{\sqrt d},
\]
the second term vanishes in every fixed $L^q$, while $G_d\Rightarrow
G\sim\mathcal{N}(0,2)$ with convergence of moments. Since $m$ is fixed,
$d/D_{j,d}\to1$ in probability holds jointly over $j=1,\ldots,m$. Slutsky's theorem
therefore gives the joint distributional convergence
\[
d^{3/2}\left(\frac1{D_{j,d}}-\frac1{\mu_d}\right)
=-\frac{D_{j,d}-\mu_d}{\sqrt d}\frac{d^2}{D_{j,d}\mu_d}
\Longrightarrow -G
\qquad (j=1,\ldots,m).
\]
This convergence also passes to cross-moments. Indeed, the lower-tail event split gives
uniform bounds on $d/D_{j,d}$ in every fixed $L^q$, while
$(D_{j,d}-\mu_d)/\sqrt d$ has uniformly bounded moments of every fixed order.
H\"older's inequality therefore bounds the displayed variables uniformly in
$L^{2+\eta}$ for some $\eta>0$. Their products are uniformly integrable, so joint
distributional convergence implies convergence of first and cross-moments. Since
$\mathbb{E}G=0$, centering therefore yields
$d^3\operatorname{Cov}(D_{i,d}^{-1},D_{j,d}^{-1})\to
\mathbb{E}G^2=2$ for every $i,j$.

\emph{Yat covariance.}
Set $Q_{j,d}:=(Z_{j,d}+b)^2$. Joint convergence of $\mathbf Z_d$, the continuous
mapping theorem, and the joint convergence $d/D_{j,d}\to1$ in probability give
\[
d\,k_{b,\varepsilon}(\mathbf{w}_{j,d},\mathbf{x}_d)
=Q_{j,d}\frac d{D_{j,d}}\Longrightarrow (Z_j+b)^2
\qquad\text{jointly in distribution.}
\]
The fixed Gaussian moments of $Q_{j,d}$, the uniform $L^q$ bounds for $d/D_{j,d}$,
and H\"older's inequality bound the left-hand variables uniformly in
$L^{2+\eta}$ for some $\eta>0$. Their products are therefore uniformly integrable,
so first and cross-moments converge to those of $((Z_j+b)^2)_{j=1}^m$.
For jointly centered Gaussian $Z_i,Z_j$ with
$\mathbb{E}[Z_iZ_j]=C_{ij}$, Isserlis' formula gives
\[
\operatorname{Cov}\bigl((Z_i+b)^2,(Z_j+b)^2\bigr)
=2C_{ij}^2+4b^2C_{ij}.
\]
Taking the resulting covariance limits proves the second limit
in~\eqref{eq:activation-covariance-limits}.
Because $m$ is fixed, entrywise convergence is equivalent to convergence in operator
norm. When $\mathbf{C}=\mathbf{I}_m$, the two displayed eigenspectra follow directly.
Finally, the strong law applied entrywise to the first and second activation moments
gives almost-sure convergence of each empirical covariance at fixed $d$.
\end{proof}

\phantomsection
\label{app:fixed-norm-proof}
\begin{proof}[Proof of Theorem~\ref{thm:activation-covariance}(ii)]
Write $Z_{j,d}=\mathbf w_{j,d}^{\top}\mathbf x_d$ and $a=1+\varepsilon$.
The finite projection vector $\mathbf Z_d=(Z_{j,d})_{j=1}^m$ converges in distribution,
and its joint moments of every fixed order converge to those of
$\mathbf G\sim\mathcal N(\mathbf0,\mathbf C)$, by the spherical moment formula.
Since $\|\mathbf x_d\|^2=d$ exactly,
\[
D_{j,d}:=\|\mathbf x_d-\mathbf w_{j,d}\|^2+\varepsilon
=d+a-2Z_{j,d}.
\]
The bound $|Z_{j,d}|\le\sqrt d$ gives
$D_{j,d}\ge(\sqrt d-1)^2+\varepsilon$. Expanding the multiplier and using the
spherical moments gives
\[
\left\|d^2\left(H_{j,d}-\frac1d\right)-(2Z_{j,d}-a)\right\|_{L^2}
=\left\|(2Z_{j,d}-a)\left(\frac d{D_{j,d}}-1\right)\right\|_{L^2}
\longrightarrow0.
\]
Covariance is unchanged by subtracting $1/d$. Joint moment convergence of
$\mathbf Z_d$ now gives the first limit. Likewise,
\[
\left\|dY_{j,d}-(Z_{j,d}+b)^2\right\|_{L^2}
=\left\|(Z_{j,d}+b)^2\left(\frac d{D_{j,d}}-1\right)\right\|_{L^2}
\longrightarrow0.
\]
All variables in each norm live on the $d$-dimensional spherical probability space;
the Gaussian vector enters only after passing the joint moments of $\mathbf Z_d$ to
their limits.
For jointly centered Gaussian $G_i,G_j$ with covariance $C_{ij}$,
\(
\operatorname{Cov}((G_i+b)^2,(G_j+b)^2)=2C_{ij}^2+4b^2C_{ij}.
\)
This proves the second limit. Since $m$ is fixed, entrywise convergence is equivalent
to operator-norm convergence.
\end{proof}

\phantomsection
\label{app:rank-geometry-proof}
\begin{proof}[Proof of Corollary~\ref{cor:covariance-rank-geometry}]
Tensor inner products give
$\langle\mathbf v_i\otimes\mathbf v_i,\mathbf v_j\otimes\mathbf v_j\rangle
=C_{ij}^2$. Hence the lifted Gram matrix is
$2(\mathbf C\circ\mathbf C)+4b^2\mathbf C$. The rank of a Gram matrix equals
the dimension of the span of its features. The dimension bounds follow because
$\mathbf v\otimes\mathbf v$ lies in $\mathrm{Sym}(r)$; the IMQ claim follows
from its limit $4\mathbf C$.
\end{proof}

For the main-text example, let the three limiting centers be
$\mathbf e_1,\mathbf e_2,(\mathbf e_1+\mathbf e_2)/\sqrt2$.
Then $\operatorname{rank}\mathbf C=2$ while
$\det(\mathbf C\circ\mathbf C)=1/2>0$. Thus at $b=0$ the spherical IMQ
limit has rank two and the Yat limit has rank three. For $b>0$, the added term
$4b^2\mathbf C$ is PSD, so the Yat limit remains positive definite and has rank three.

\begin{remark}
The two cases of Theorem~\ref{thm:activation-covariance} separate two mechanisms. Gaussian
inputs contain an order-$\sqrt d$ norm fluctuation shared by every center, which
dominates the centered IMQ covariance. Spherical normalization removes that common
term, exposing the center-specific projection covariance $4\mathbf C$. The Yat
numerator already depends at leading order on those projections, so its covariance
limit is the same in both models. These two distributions provide analytically exact
reference regimes for studying normalized neural representations.
\end{remark}

\subsection{Further Far-Field Separations}
\label{app:far-field-separations}
\label{app:additional-theory}%
\label{app:additional-theory-bc}

We write
\[
k_{b,\varepsilon}(\mathbf{w},\mathbf{x})
=
\frac{(\mathbf{w}^\top\mathbf{x}+b)^2}{\|\mathbf{x}-\mathbf{w}\|^2+\varepsilon},
\qquad
b\ge 0,\quad \varepsilon>0.
\]
For fixed shared \((b,\varepsilon)\), let \(\mathcal H_{b,\varepsilon}\) denote the RKHS of
\(k_{b,\varepsilon}\). For \(R>0\), write
\[
B_R := \{\mathbf{x}\in\R^d:\|\mathbf{x}\|\le R\},
\qquad
\Sph^{d-1}:=\{\mathbf{u}\in\R^d:\|\mathbf{u}\|=1\}.
\]

This appendix develops two radial-versus-Yat comparisons: asymptotic $L^\infty$ and
$L^2$ bounds on exterior shells, and a polynomial-separation gap on domains containing
a line segment. Both isolate structural differences under explicit coefficient and
center constraints.

\subsubsection{Asymptotic exterior-shell separation}
\label{subsec:exterior-shell-asymptotic}

We complement the bounded-domain gap of Section~\ref{sec:exterior-shell} with
asymptotic $L^\infty$ and $L^2$ separations on exterior shells $A_R:=\{r\mathbf{u}:R\le r\le 2R,\ \mathbf{u}\in\Sph^{d-1}\}$ for $R\gg W$. For $W,\Lambda>0$, write
\(
\mathcal R_{\mathrm{IMQ}}(\Lambda,W)
\)
and
\(
\mathcal R_{\mathrm{RBF}}(\Lambda,W,\gamma),\qquad \gamma>0,
\)
for the bounded-variation classes
$\sum_{j=1}^m a_j\,h_\varepsilon(\cdot,\mathbf{v}_j)$ and
$\sum_{j=1}^m a_j\,e^{-\gamma\|\cdot-\mathbf{v}_j\|^2}$ with $\sum|a_j|\le\Lambda$ and
$\|\mathbf{v}_j\|\le W$.

\begin{lemma}[Uniform radial decay on exterior shells]
\label{lem:radial-decay-shell}
Let $R>W$. For $F\in\mathcal R_{\mathrm{IMQ}}(\Lambda,W)$, $\sup_{A_R}|F|\le\Lambda/((R-W)^2+\varepsilon)$;
for $F\in\mathcal R_{\mathrm{RBF}}(\Lambda,W,\gamma)$, $\sup_{A_R}|F|\le\Lambda e^{-\gamma(R-W)^2}$.
\end{lemma}

\begin{proof}
$\|\mathbf{x}-\mathbf{v}_j\|\ge R-W$ by reverse triangle inequality, so
$(\|\mathbf{x}-\mathbf{v}_j\|^2+\varepsilon)^{-1}\le((R-W)^2+\varepsilon)^{-1}$ and
$e^{-\gamma\|\mathbf{x}-\mathbf{v}_j\|^2}\le e^{-\gamma(R-W)^2}$. Sum against $|a_j|$.
\end{proof}

\begin{lemma}[Quantitative Yat far-field approximation]
\label{lem:yat-quantitative-farfield}
Fix \(\mathbf{w}_\star\in\R^d\), \(b\ge 0\), and \(\varepsilon>0\). Let
\(g_\star(\mathbf{x})=k_{b,\varepsilon}(\mathbf{w}_\star,\mathbf{x})\) and \(q(\mathbf{u}):=(\mathbf{u}^\top \mathbf{w}_\star)^2\).
For every \(R>\|\mathbf{w}_\star\|\) and every \(r\in[R,2R]\),
\[
\left|
g_\star(r\mathbf{u})-q(\mathbf{u})
\right|
\le
\eta_R,
\qquad
\eta_R
:=
\frac{
4R\|\mathbf{w}_\star\|(b+\|\mathbf{w}_\star\|^2)+b^2+\|\mathbf{w}_\star\|^2(\|\mathbf{w}_\star\|^2+\varepsilon)
}{
(R-\|\mathbf{w}_\star\|)^2+\varepsilon
}.
\]
In particular, \(\eta_R=O(R^{-1})\) as \(R\to\infty\).
\end{lemma}

\begin{proof}
Let \(W_\star:=\|\mathbf{w}_\star\|\), \(a:=\mathbf{u}^\top \mathbf{w}_\star\), so \(|a|\le W_\star\). Then
\(g_\star(r\mathbf{u})=(ra+b)^2/(r^2-2ra+W_\star^2+\varepsilon)\) and \(q(\mathbf{u})=a^2\). The numerator of
\(g_\star(r\mathbf{u})-a^2\) equals \((ra+b)^2-a^2(r^2-2ra+W_\star^2+\varepsilon)
=2ra(b+a^2)+b^2-a^2(W_\star^2+\varepsilon)\),
whose absolute value is at most \(2rW_\star(b+W_\star^2)+b^2+W_\star^2(W_\star^2+\varepsilon)
\le 4RW_\star(b+W_\star^2)+b^2+W_\star^2(W_\star^2+\varepsilon)\).
The denominator \(\|r\mathbf{u}-\mathbf{w}_\star\|^2+\varepsilon\ge(R-W_\star)^2+\varepsilon\).
Combining gives \(\eta_R\).
\end{proof}

\begin{theorem}[Far-field separation: $L^\infty$ and $L^2$]
\label{thm:far-field-separation}
Fix $\mathbf{w}_\star\neq\mathbf{0}$, $b\ge 0$, $\varepsilon>0$,
$g_\star=k_{b,\varepsilon}(\mathbf{w}_\star,\cdot)$, $W\ge\|\mathbf{w}_\star\|$,
$\Lambda>0$, $R>W$, and let $\eta_R$ be as in Lemma~\ref{lem:yat-quantitative-farfield}.

\medskip\noindent\textbf{(i) $L^\infty$-bound on the exterior shell.}
\[
\inf_{F\in\mathcal R_{\mathrm{IMQ}}(\Lambda,W)}\|g_\star-F\|_{L^\infty(A_R)}
\ge \bigl(\|\mathbf{w}_\star\|^2-\eta_R-\tfrac{\Lambda}{(R-W)^2+\varepsilon}\bigr)_+,
\]
with the analogous bound over $\mathcal R_{\mathrm{RBF}}(\Lambda,W,\gamma)$ replacing the IMQ decay term by $\Lambda e^{-\gamma(R-W)^2}$.

\medskip\noindent\textbf{(ii) $L^2$ risk lower bound under directional sampling.}
Let $\nu$ be a probability distribution on $\Sph^{d-1}$ with $r\in[R,2R]$ independent of $\mathbf{u}\sim\nu$, $q(\mathbf{u}):=(\mathbf{u}^\top\mathbf{w}_\star)^2$, and $c_\nu(\mathbf{w}_\star):=\mathbb E[q(\mathbf{u})^2]>0$. Then every $F\in\mathcal R_{\mathrm{IMQ}}(\Lambda,W)$ satisfies
\[
\mathbb E\bigl[(g_\star(\mathbf{x})-F(\mathbf{x}))^2\bigr]
\ge \Bigl[\bigl(\sqrt{c_\nu(\mathbf{w}_\star)}-\eta_R-\tfrac{\Lambda}{(R-W)^2+\varepsilon}\bigr)_+\Bigr]^2,
\]
with the analogous RBF bound. A single Yat atom realizes $g_\star$ at zero error in either norm.
\end{theorem}

\begin{proof}
\textbf{(i)} Take $\mathbf{u}_\star:=\mathbf{w}_\star/\|\mathbf{w}_\star\|$. For $r\in[R,2R]$, $g_\star(r\mathbf{u}_\star)\ge\|\mathbf{w}_\star\|^2-\eta_R$ by Lemma~\ref{lem:yat-quantitative-farfield} and $|F(r\mathbf{u}_\star)|\le\Lambda/((R-W)^2+\varepsilon)$ by Lemma~\ref{lem:radial-decay-shell}. The triangle inequality and supremum over $A_R$ give the IMQ bound; the RBF case is identical. \textbf{(ii)} $\|g_\star-q\|_{L^2}\le\eta_R$ (Lemma~\ref{lem:yat-quantitative-farfield} transferred $L^\infty\to L^2$), so reverse triangle gives $\|g_\star\|_{L^2}\ge\sqrt{c_\nu(\mathbf{w}_\star)}-\eta_R$. $\|F\|_{L^2}\le\Lambda/((R-W)^2+\varepsilon)$ from Lemma~\ref{lem:radial-decay-shell}. A second reverse-triangle step, positive part, and squaring give the IMQ bound; RBF identical.
\end{proof}

\begin{remark}
Both bounds require $\|\mathbf{v}_j\|\le W$ and $\sum|a_j|\le\Lambda$; lifting either restriction allows radial methods to place centers in the shell or use coefficients growing with the radius, and the bounds become vacuous.
\end{remark}

\section{Constructive Representation and Approximation}
\label{app:representation-proofs}

\subsection{Exact IMQ recovery and three-atom optimality}
\label{app:imq-proof}

The finite-difference proof is elementary. Optimality requires a separate denominator
argument: irreducibility and coprimality for $d\ge2$, and pole matching for $d=1$.

\begin{theorem*}[Theorem~\ref{thm:imq-embed}(i) restated]
Fix $\varepsilon>0$. For every $\mathbf{w}\in\R^d$ and every $h>0$, the
bias-finite-difference identity~\eqref{eq:imq-identity} holds
identically in $\mathbf{x}\in\R^d$. In particular, $F_{\IMQ,\varepsilon}\subseteq F_{\yat,\varepsilon}^{+}$.
The containment is strict.
\end{theorem*}

\begin{proof}
The denominator $D_\varepsilon(\mathbf{x},\mathbf{w})=\|\mathbf{x}-\mathbf{w}\|^2+\varepsilon$ is independent of the bias parameter. For fixed $(\mathbf{x},\mathbf{w})$ set $t:=\mathbf{x}^\top\mathbf{w}+2h$. Then
\[
g_\varepsilon(\mathbf{x};\mathbf{w},3h)=\frac{(t+h)^2}{D_\varepsilon},\quad
g_\varepsilon(\mathbf{x};\mathbf{w},2h)=\frac{t^2}{D_\varepsilon},\quad
g_\varepsilon(\mathbf{x};\mathbf{w},h)=\frac{(t-h)^2}{D_\varepsilon}.
\]
The numerator second-order difference is
\[
(t+h)^2-2t^2+(t-h)^2 \;=\; (t^2+2th+h^2)-2t^2+(t^2-2th+h^2)\;=\;2h^2,
\]
independent of $t$. Dividing the numerator identity by $D_\varepsilon$ gives
$g_\varepsilon(\mathbf{x};\mathbf{w},3h)-2g_\varepsilon(\mathbf{x};\mathbf{w},2h)+g_\varepsilon(\mathbf{x};\mathbf{w},h)=2h^2/D_\varepsilon=2h^2k_\IMQ^\varepsilon(\mathbf{x},\mathbf{w})$, which is~\eqref{eq:imq-identity}. Since $h>0$ the three biases $\{h,2h,3h\}$ are all strictly positive, so the identity uses only atoms from $F_{\yat,\varepsilon}^{+}$ and hence $F_{\IMQ,\varepsilon}\subseteq F_{\yat,\varepsilon}^{+}$.

\emph{Strictness.} Fix $\mathbf{w}\neq 0$ and a unit vector $\mathbf{u}\in\Sph^{d-1}$ with $\mathbf{u}^\top\mathbf{w}\neq 0$. Along the ray $\mathbf{x}=r\mathbf{u}$,
\[
\kYat(r\mathbf{u},\mathbf{w})\;=\;\frac{r^2(\mathbf{u}^\top\mathbf{w})^2}{r^2-2r\,\mathbf{u}^\top\mathbf{w}+\|\mathbf{w}\|^2+\varepsilon}\;\xrightarrow[r\to\infty]{}\;(\mathbf{u}^\top\mathbf{w})^2 \;\neq\;0,
\]
while every finite linear combination $F=\sum_{i=1}^N c_i k_\IMQ^\varepsilon(\cdot,\mathbf{w}_i)\in F_{\IMQ,\varepsilon}$ obeys
\[
|F(r\mathbf{u})|\;\le\;\sum_{i=1}^N\frac{|c_i|}{r^2-2r\,\mathbf{u}^\top\mathbf{w}_i+\|\mathbf{w}_i\|^2+\varepsilon}\;=\;O(r^{-2})\;\xrightarrow[r\to\infty]{}\;0.
\]
The two far-field limits disagree, so $\kYat(\cdot,\mathbf{w})\notin F_{\IMQ,\varepsilon}$ for every $\mathbf{w}\neq 0$. The quadratic extrapolation identity
$\kYat(\cdot,\mathbf{w})=3g_\varepsilon(\cdot;\mathbf{w},h)-3g_\varepsilon(\cdot;\mathbf{w},2h)+g_\varepsilon(\cdot;\mathbf{w},3h)$
places the same section in $F_{\yat,\varepsilon}^{+}$ for every $h>0$, proving strict
containment $F_{\IMQ,\varepsilon}\subsetneq F_{\yat,\varepsilon}^{+}$.
\end{proof}

\begin{lemma}[Irreducibility and coprimality of the regularized distance]
\label{lem:Dw-irred}
Let $d\ge 2$, $\varepsilon>0$, and for $\mathbf{w}\in\R^d$ set
$D_\mathbf{w}(\mathbf{x}):=\|\mathbf{x}-\mathbf{w}\|^2+\varepsilon$. Then
\textup{(i)} $D_\mathbf{w}$ is irreducible in $\mathbb{C}[\mathbf{x}_1,\ldots,\mathbf{x}_d]$;
\textup{(ii)} for $\mathbf{v}\neq \mathbf{w}$, $D_\mathbf{v}$ and $D_\mathbf{w}$ are coprime
in $\mathbb{C}[\mathbf{x}]$.
\end{lemma}

\begin{proof}
\textup{(i)} Translate by $\mathbf{w}$: with $\mathbf{y}:=\mathbf{x}-\mathbf{w}$,
$D_\mathbf{w}=\|\mathbf{y}\|^2+\varepsilon$. Suppose this factors as
$f(\mathbf{y})g(\mathbf{y})$ in $\mathbb{C}[\mathbf{y}]$ with $\deg f,\deg g\ge 1$.
Both factors are linear: $f=\boldsymbol\alpha^\top\mathbf{y}+\beta$,
$g=\boldsymbol\gamma^\top\mathbf{y}+\delta$. Matching coefficients gives
\[
\boldsymbol\alpha\boldsymbol\gamma^\top+\boldsymbol\gamma\boldsymbol\alpha^\top=2\mathbf{I}_d,
\qquad \delta\boldsymbol\alpha+\beta\boldsymbol\gamma=\mathbf{0},
\qquad \beta\delta=\varepsilon.
\]
Since $\varepsilon\neq 0$, neither $\beta$ nor $\delta$ vanishes, so the linear
constraint gives $\boldsymbol\gamma=-(\delta/\beta)\boldsymbol\alpha$ --- the two
direction vectors are parallel. But then
$\boldsymbol\alpha\boldsymbol\gamma^\top+\boldsymbol\gamma\boldsymbol\alpha^\top
=-(2\delta/\beta)\boldsymbol\alpha\boldsymbol\alpha^\top$ has rank at most~$1$, while
$2\mathbf{I}_d$ has rank $d\ge 2$. Contradiction; $D_\mathbf{w}$ is irreducible.
\textup{(ii)} If $D_\mathbf{v}$ and $D_\mathbf{w}$ shared an irreducible factor for
$\mathbf{v}\neq \mathbf{w}$, since both are irreducible they would equal each other up to
a scalar; comparing the degree-$1$ part $-2\mathbf{x}^\top\mathbf{v}$ vs.\
$-2\mathbf{x}^\top\mathbf{w}$ then forces $\mathbf{v}=\mathbf{w}$, contradiction.
\end{proof}

\subsubsection*{Proof of Theorem~\ref{thm:imq-embed}(ii)}
\label{thm:tightness-app}
Let $d\ge 1$, $\varepsilon>0$. For every $\mathbf{w}_0\in\R^d\setminus\{0\}$, no linear combination of at most two biased Yat atoms (at arbitrary centers and biases) can equal a nonzero scalar multiple of $k_\IMQ^{\varepsilon}(\cdot,\mathbf{w}_0)$.

\noindent\textit{Outline.}
The argument splits by dimension. For $d\ge 2$ we use Lemma~\ref{lem:Dw-irred}
(irreducibility and coprimality of $D_\mathbf{w}$ in $\mathbb{C}[\mathbf{x}]$) and a
case-by-case polynomial analysis. For $d=1$, $D_w$ factors over $\mathbb{C}$ as
$(z-w-i\sqrt{\varepsilon})(z-w+i\sqrt{\varepsilon})$, so Lemma~\ref{lem:Dw-irred} fails;
we instead match poles directly in the complex plane.

\begin{proof}[Proof for $d\ge 2$]
We first reduce to non-trivial atoms. An atom with $c_i=0$, or with
$L_i\equiv 0$ (i.e.\ $\mathbf{w}_i=\mathbf{0}$ and $b_i=0$), contributes
zero; dropping it gives a one-atom or zero-atom identity. A zero-atom
combination is identically $0\neq\lambda k_\IMQ^\varepsilon$ for $\lambda\neq 0$.
A single non-trivial atom: clearing denominators in
$c\,g_\varepsilon(\cdot;\mathbf{w},b)=\lambda\,k_\IMQ^\varepsilon(\cdot,\mathbf{w}_0)$ gives
$c\,L^2\,D_0=\lambda\,D_w$ in $\mathbb{C}[\mathbf{x}]$, where $L(\mathbf{x})=\mathbf{x}^\top\mathbf{w}+b$
and $D_w=\|\mathbf{x}-\mathbf{w}\|^2+\varepsilon$. If $\mathbf{w}=\mathbf{w}_0$: cancelling
$D_0$ gives $c\,L^2=\lambda$, impossible since $L$ is nonconstant for
$\mathbf{w}_0\neq\mathbf{0}$. If $\mathbf{w}\neq\mathbf{w}_0$: by Lemma~\ref{lem:Dw-irred},
$D_w$ is irreducible (hence prime in the UFD $\mathbb{C}[\mathbf{x}]$) and coprime to $D_0$;
from $D_w\mid c\,L^2$ we get $D_w\mid L$, but $\deg D_w=2>\deg L\le 1$ and
$L\not\equiv 0$: contradiction.
Hence $\lambda=0$ in the one-atom case. Assume both atoms are non-trivial,
$c_i\neq 0$, $L_i\not\equiv 0$, and for contradiction $\lambda\neq 0$.
Let $D_i(\mathbf{x})=\|\mathbf{x}-\mathbf{w}_i\|^2+\varepsilon$ and $L_i(\mathbf{x})=\mathbf{x}^\top\mathbf{w}_i+b_i$.
Clearing denominators gives
\[
c_1 L_1^2 D_0 D_2 + c_2 L_2^2 D_0 D_1 = \lambda D_1 D_2
\quad\text{in }\R[\mathbf{x}]. \tag{$*$}
\]

\emph{Case 1: $\mathbf{w}_1=\mathbf{w}_2=:\mathbf{w}$.}
Then $D_1=D_2=:D$ and $(*)$ becomes
$N\,D_0=\lambda\,D$ where $N:=c_1 L_1^2+c_2 L_2^2$.
\emph{Subcase $\mathbf{w}=\mathbf{w}_0$:} $D=D_0$, so $N=\lambda$.
The degree-$2$ homogeneous part of $N$ is $(c_1+c_2)(\mathbf{x}^\top\mathbf{w}_0)^2$.
For $d\ge 2$ and $\mathbf{w}_0\neq\mathbf{0}$ this vanishes only if $c_1+c_2=0$;
the degree-$1$ part then forces $c_1 b_1+c_2 b_2=0$, so $b_1=b_2$, and
$N=b_1^2(c_1+c_2)=0$, contradicting $\lambda\neq 0$.
\emph{Subcase $\mathbf{w}\neq\mathbf{w}_0$:} by Lemma~\ref{lem:Dw-irred}, $D$ and $D_0$
are distinct irreducibles in $\mathbb{C}[\mathbf{x}]$, hence coprime. From $N D_0=\lambda D$:
$D\mid N D_0$ and $\gcd(D,D_0)=1$, so $D\mid N$. Since $\deg N\le 2=\deg D$,
we have $N=\mu D$ for some constant $\mu$; then $\mu D_0=\lambda$, a degree-$2$
polynomial equal to a constant, forcing $\mu=\lambda=0$: contradiction.

\emph{Case 2: $\mathbf{w}_1=\mathbf{w}_0\neq\mathbf{w}_2$
(the case $\mathbf{w}_2=\mathbf{w}_0\neq\mathbf{w}_1$ is symmetric).}
Then $D_1=D_0$; cancelling $D_0$ from $(*)$ gives
$c_1 L_1^2 D_2+c_2 L_2^2 D_0=\lambda D_2$, hence
\[
(c_1 L_1^2-\lambda)\,D_2 = -c_2 L_2^2\,D_0.
\]
Since $\mathbf{w}_0\neq\mathbf{w}_2$, by Lemma~\ref{lem:Dw-irred} $D_0$ and $D_2$ are
coprime irreducibles in $\mathbb{C}[\mathbf{x}]$, so $D_0\mid(c_1 L_1^2-\lambda)$.
Both have degree $\le 2=\deg D_0$, so $c_1 L_1^2-\lambda=\mu D_0$ for
some $\mu\in\mathbb{C}$. Comparing degree-$2$ homogeneous parts:
$c_1(\mathbf{x}^\top\mathbf{w}_0)^2=\mu\|\mathbf{x}\|^2$.
For $d\ge 2$ and $\mathbf{w}_0\neq\mathbf{0}$, the polynomials
$(\mathbf{x}^\top\mathbf{w}_0)^2$ and $\|\mathbf{x}\|^2$ are linearly independent
in $\mathbb{C}[\mathbf{x}]$, witnessed by two evaluations:
\textbf{(i)} at $\mathbf{x}=\mathbf{w}_0$, the values are $\|\mathbf{w}_0\|^4$ and $\|\mathbf{w}_0\|^2$ respectively, with ratio $\|\mathbf{w}_0\|^2$;
\textbf{(ii)} at any non-zero $\mathbf{x}\perp\mathbf{w}_0$ (which exists for $d\ge 2$ and is the only step that uses the dimensional hypothesis), the values are $0$ and $\|\mathbf{x}\|^2>0$, with a different ratio, so the two polynomials are not proportional.
At $d=1$ both polynomials equal $w_{0,1}^2 x^2$ up to a non-zero scalar, so Case~2 genuinely requires $d\ge 2$; Case~1 applies for all $d\ge 1$. Hence $\mu=c_1=0$, giving $\lambda=0$: contradiction.

\emph{Case 3: $\mathbf{w}_0,\mathbf{w}_1,\mathbf{w}_2$ pairwise distinct.}
By Lemma~\ref{lem:Dw-irred}, each $D_i$ is irreducible in $\mathbb{C}[\mathbf{x}]$ and
distinct $D_i$ are pairwise coprime. Since $\mathbb{C}[\mathbf{x}]$ is a unique factorization domain (UFD), irreducibles are prime and the principal ideal $(D_1)$ is therefore prime, making $R:=\mathbb{C}[\mathbf{x}]/(D_1)$ an integral domain. Reducing $(*)$ modulo $(D_1)$:
$c_1\,\overline{L_1}^{\,2}\,\overline{D_0}\,\overline{D_2}=0$.
Since $\mathbf{w}_0,\mathbf{w}_2\neq\mathbf{w}_1$, the classes
$\overline{D_0},\overline{D_2}$ are nonzero in $R$; since $R$ is an
integral domain and $c_1\neq 0$, we get $\overline{L_1}=0$, i.e.\
$L_1\in(D_1)$. But $L_1\not\equiv 0$ and $\deg L_1\le1<2=\deg D_1$:
contradiction. By symmetry (quotient by $D_2$), $c_2=0$: contradiction.
\end{proof}

\begin{proof}[Proof for $d=1$ via complex pole matching]
At $d=1$ the variables collapse to scalars: $w_i,b_i,w_0\in\R$, and
$D_w(z)=(z-w)^2+\varepsilon$ factors over $\mathbb{C}$ as
$(z-w-i\sqrt{\varepsilon})(z-w+i\sqrt{\varepsilon})$, so
Lemma~\ref{lem:Dw-irred} no longer applies. Instead, assume for contradiction
\[
c_1\frac{(w_1 z+b_1)^2}{(z-w_1)^2+\varepsilon}
+c_2\frac{(w_2 z+b_2)^2}{(z-w_2)^2+\varepsilon}
=\frac{\lambda}{(z-w_0)^2+\varepsilon}
\]
holds for all $z\in\R$ with $\lambda\neq 0$, $\varepsilon>0$. The standing hypothesis
$\mathbf{w}_0\neq\mathbf{0}$ becomes $w_0\neq0$ in dimension one. This excludes the
origin-centered exception: for $b\neq0$, the single atom
$g_\varepsilon(\cdot;0,b)/b^2=k_\IMQ^\varepsilon(\cdot,0)$ already realizes that IMQ
section.

First discard terms with zero coefficient or identically zero numerator. If none
remain, the left-hand side is zero, contradicting the nonzero right-hand side. If only one term
remains, its center must equal $w_0$: otherwise it is analytic at the two poles
of the right-hand side. With that center, cancellation of the common denominator
would give $c(w_0z+b)^2=\lambda$, impossible since $c\neq0$ and $w_0\neq0$.
We may therefore assume $c_1,c_2\neq0$ and both numerators are nonzero polynomials.

Clearing denominators extends the equality to an identity of rational functions
over $\mathbb C$. The right-hand side has exactly two simple poles, at $z=w_0\pm i\sqrt{\varepsilon}$. We partition by the pattern of equality among $\{w_0,w_1,w_2\}$ into four mutually exclusive and jointly exhaustive cases:
$\{w_1,w_2,w_0\}$ contains (1) three distinct values; (2) exactly $w_1=w_2$ and both differ from $w_0$; (3) exactly one of $w_1,w_2$ equals $w_0$ (and the two atom centers differ); (4) $w_1=w_2=w_0$.

\emph{Case 1: $w_0,w_1,w_2$ pairwise distinct.}
The first atom contributes potential simple poles at $z=w_1\pm i\sqrt{\varepsilon}$.
These poles are not present in either the second atom (whose poles are at
$w_2\pm i\sqrt{\varepsilon}\neq w_1\pm i\sqrt{\varepsilon}$) nor the right-hand side
(whose poles are at $w_0\pm i\sqrt{\varepsilon}\neq w_1\pm i\sqrt{\varepsilon}$). Hence
they must be removable, requiring the linear factor $w_1 z+b_1$ to vanish at
$z=w_1+i\sqrt{\varepsilon}$:
\[
w_1(w_1+i\sqrt{\varepsilon})+b_1=(w_1^2+b_1)+i w_1\sqrt{\varepsilon}=0.
\]
Since $w_1,b_1,\varepsilon$ are real with $\varepsilon>0$, separating real and imaginary parts forces $w_1=0$ and $b_1=0$, i.e.\ the first atom is identically zero. The same argument applied to atom~2 forces $w_2=0=b_2$. The identity collapses to $0=\lambda/D_{w_0}$, contradicting $\lambda\neq 0$.

\emph{Case 2: $w_1=w_2=:w\neq w_0$.}
The two atoms share denominator $D_w(z)$. The combined left-hand-side numerator
$P(z):=c_1(wz+b_1)^2+c_2(wz+b_2)^2$ has degree $\le 2$. Equality of rational functions
implies equality of pole sets with multiplicity; since $w\neq w_0$, the LHS poles at
$w\pm i\sqrt{\varepsilon}$ must cancel, forcing $D_w\mid P$. Since
$\deg P\le 2=\deg D_w$, $P=C\cdot D_w$ for some constant $C\in\mathbb{C}$. Then
LHS$=C$ is constant on $\mathbb{C}$, while RHS$=\lambda/D_{w_0}$ is non-constant
(since $\lambda\neq 0$). Contradiction.

\emph{Case 3: $w_1=w_0\neq w_2$ (and the symmetric subcase $w_2=w_0\neq w_1$).}
The atom whose center differs from $w_0$ has poles at the off-$w_0$ locations; by the Case~1 argument applied to that atom alone, the atom is identically zero. The identity reduces to one atom $c_1(w_0 z+b_1)^2/D_{w_0}(z)=\lambda/D_{w_0}(z)$, i.e.\ $c_1(w_0 z+b_1)^2=\lambda$. The left side is a polynomial in $z$ of degree exactly $2$ (using $w_0\neq 0$ and $c_1\neq 0$ for a non-trivial atom), while the right side is a non-zero constant. Contradiction.

\emph{Case 4: $w_1=w_2=w_0$.}
LHS and RHS share denominator $D_{w_0}(z)$; equating numerators gives the polynomial
identity
\[
c_1(w_0 z+b_1)^2+c_2(w_0 z+b_2)^2=\lambda
\quad\text{in }\R[z].
\]
Matching the $z^2$ coefficient: $w_0^2(c_1+c_2)=0$, so $c_2=-c_1$ (since
$w_0\neq 0$). Matching the $z^1$ coefficient: $2w_0(c_1 b_1+c_2 b_2)=2w_0 c_1(b_1-b_2)=0$,
so $b_1=b_2$ (since $c_1\neq 0$). Substituting back into the constant term:
$c_1 b_1^2+c_2 b_2^2=c_1 b_1^2-c_1 b_1^2=0\neq\lambda$. Contradiction.

All four cases yield a contradiction, so the factor-$3$ lower bound holds at $d=1$ as
well.
\end{proof}

\subsection{Variable-bias span and approximation consequences}
\label{app:proofs-universality-extra}

The exact finite difference yields three span-level consequences: recovery of the
unbiased atom from positive biases, strict containment of finite spans, and direct
transfer of universality and approximation rates.

\begin{lemma}[The unbiased section lies in the positive-bias span]
\label{lem:b0-in-Fplus}
Fix $\mathbf{w}\in\R^d$ and $\varepsilon>0$. The map $b\mapsto g_\varepsilon(\cdot;\mathbf{w},b)$ is quadratic in $b$, so for every $h>0$
\begin{equation}
\label{eq:b0-extract}
g_\varepsilon(\cdot;\mathbf{w},0)
\;=\; 3\,g_\varepsilon(\cdot;\mathbf{w},h) \;-\; 3\,g_\varepsilon(\cdot;\mathbf{w},2h) \;+\; g_\varepsilon(\cdot;\mathbf{w},3h),
\end{equation}
where every atom on the right has bias in $\{h,2h,3h\}\subset(0,\infty)$. In particular $g_\varepsilon(\cdot;\mathbf{w},0)\in F_{\yat,\varepsilon}^{+}$.
\end{lemma}

\begin{proof}[Proof of Lemma~\ref{lem:b0-in-Fplus}]
$D_\varepsilon$ is independent of $b$, so $g_\varepsilon(\cdot;\mathbf{w},b)D_\varepsilon=(\mathbf{x}^\top\mathbf{w}+b)^2$ is a univariate quadratic in $b$. Quadratic interpolation at $\{h,2h,3h\}$ and evaluation at $b=0$ give coefficients $(3,-3,1)$, since $f(0)=3f(h)-3f(2h)+f(3h)$ for every quadratic $f$. Dividing by $D_\varepsilon$ gives~\eqref{eq:b0-extract}.
\end{proof}

\begin{corollary}[Strict span containment]
\label{cor:span-strict-cont}
$F_{\IMQ,\varepsilon}\subsetneq F_{\yat,\varepsilon}^{+}$ for every $\varepsilon>0$. Concretely, for every $\mathbf{w}\neq\mathbf{0}$ the unbiased section $g_\varepsilon(\cdot;\mathbf{w},0)$ lies in $F_{\yat,\varepsilon}^{+}$ (Lemma~\ref{lem:b0-in-Fplus}) but not in $F_{\IMQ,\varepsilon}$. This is the span-level companion to the RKHS-level strict containment of Corollary~\ref{cor:strict-order}: the same separation holds at the level of finite linear spans of atoms, witnessed by the directional far-field trace rather than by Loewner domination.
\end{corollary}

\begin{proof}[Proof of Corollary~\ref{cor:span-strict-cont}]
Inclusion is Theorem~\ref{thm:imq-embed}. For the strict separation, fix a unit vector $\mathbf{u}\in\Sph^{d-1}$ with $\mathbf{u}^\top\mathbf{w}\neq 0$. Along the ray $\mathbf{x}=r\mathbf{u}$,
\(
g_\varepsilon(r\mathbf{u};\mathbf{w},0)=r^2(\mathbf{u}^\top\mathbf{w})^2/(r^2-2r\,\mathbf{u}^\top\mathbf{w}+\|\mathbf{w}\|^2+\varepsilon)\to(\mathbf{u}^\top\mathbf{w})^2\neq 0,
\)
while every finite combination $F\in F_{\IMQ,\varepsilon}$ obeys $|F(r\mathbf{u})|=O(r^{-2})\to 0$. The far-field limits disagree, so $g_\varepsilon(\cdot;\mathbf{w},0)\notin F_{\IMQ,\varepsilon}$. By Lemma~\ref{lem:b0-in-Fplus} it does lie in $F_{\yat,\varepsilon}^{+}$, witnessing strictness.
\end{proof}

\begin{corollary}[Family universality, algebraic route]
\label{cor:family-uat}
For every compact $\mathcal{X}\subset\R^d$, both
$F_{\yat,\varepsilon}^{+}$ at any single $\varepsilon>0$ and the broader
families
$F_{\yat}^{+}=\mathrm{span}\{k_{b,\varepsilon}(\cdot,\mathbf{w}):\mathbf{w}\in\R^d, b>0,\varepsilon>0\}$
and
$F_{\yat}^{\ge 0}=\mathrm{span}\{k_{b,\varepsilon}(\cdot,\mathbf{w}):\mathbf{w}\in\R^d, b\ge 0,\varepsilon>0\}$
are dense in $C(\mathcal{X})$ in the uniform norm. This gives a second algebraic route to the universality already established analytically by Proposition~\ref{prop:singular} via Loewner-IMQ domination.
\end{corollary}

\begin{proof}[Proof of Corollary~\ref{cor:family-uat}]
$F_{\IMQ,\varepsilon}$ is dense in $C(\mathcal{X})$ by~\citet[Thm.~17]{micchelli2006}.
By Theorem~\ref{thm:imq-embed}, $F_{\IMQ,\varepsilon}\subseteq F_{\yat,\varepsilon}^{+}\subseteq F_{\yat}^{+}\subseteq F_{\yat}^{\ge 0}$,
so each superset is dense.
\end{proof}

\begin{corollary}[Approximation rate transfer at factor-$3$ width overhead]
\label{cor:rate-transfer}
Fix $\varepsilon>0$ and $h>0$, and let $\mathcal{X}\subset\R^d$ be compact. For every
finite IMQ approximant $F=\sum_{j=1}^m a_j\,k_\IMQ^\varepsilon(\cdot,\mathbf{w}_j)\in F_{\IMQ,\varepsilon}$
of $m$ atoms, there exists a Yat approximant $G\in F_{\yat,\varepsilon}^{+}$ of at most
$3m$ atoms (positive biases drawn from $\{h,2h,3h\}$) with $G(\mathbf{x})=F(\mathbf{x})$
for every $\mathbf{x}\in\R^d$. In particular, any quantitative IMQ approximation rate
$\inf_{F\in F_{\IMQ,\varepsilon},\,|F|\le m}\|f-F\|_{L^\infty(\mathcal{X})}\le\phi(m)$
transfers to a Yat rate $\inf_{G\in F_{\yat,\varepsilon}^{+},\,|G|\le 3m}\|f-G\|_{L^\infty(\mathcal{X})}\le\phi(m)$.
Equivalently, for every Yat budget $M\ge3$, the transferred rate is
$\phi(\lfloor M/3\rfloor)$.
\end{corollary}

\begin{proof}[Proof of Corollary~\ref{cor:rate-transfer}]
By Theorem~\ref{thm:imq-embed}, every IMQ atom equals $(2h^2)^{-1}$ times a positive-bias
second-order finite difference of three Yat atoms, identically in $\mathbf{x}$.
Substituting term by term in $F$ produces a $3m$-atom Yat expansion $G$ with $G\equiv F$
pointwise; the quantitative rate transfer is then immediate.
\end{proof}

\subsection{Quadratic Approximation and an IMQ Coefficient-Budget Obstruction}
\label{app:yat-native-atom-count}

The exterior-shell separation (Appendix~\ref{subsec:exterior-shell-asymptotic}) and
Theorem~\ref{thm:main-atom-separation} capture complementary limitations of finite radial
expansions. We now prove the bounded-domain coefficient-budget theorem and extend it to
mixed quadratic--IMQ targets. The obstruction places no restriction on IMQ width or
center locations.

\paragraph{Setup.} For $\Lambda>0$, define the bounded-variation IMQ class
\[
\mathcal{V}_\IMQ(\Lambda) \;:=\; \Bigl\{\textstyle\sum_j a_j\,h_\varepsilon(\cdot,\mathbf{v}_j)\,:\,\textstyle\sum |a_j|\le\Lambda,\ \mathbf{v}_j\in\R^d\Bigr\}.
\]
The atom count of an element is the smallest such expansion length.

\begin{lemma}[Single-atom Yat approximation of a rank-one quadratic]
\label{lem:single-atom-rank-one}
Let $\varepsilon>0$, $\mathbf{w}_\star\in\Sph^{d-1}$, $\delta>0$, $R>0$, and set
\[
\alpha_0(d,R,\varepsilon,\delta)\;:=\;\max\Bigl\{
8R\sqrt d,\ 2\sqrt{dR^2+\varepsilon},\
\tfrac{8R^3 d^{3/2}}{\delta},\
2\sqrt{\tfrac{dR^2(dR^2+\varepsilon)}{\delta}}
\Bigr\}.
\]
For every $\alpha\ge\alpha_0$,
\[
\sup_{\mathbf{x}\in[-R,R]^d}\bigl|\,\kYat(\alpha\mathbf{w}_\star,\mathbf{x})-(\mathbf{w}_\star^\top\mathbf{x})^2\,\bigr| \;\le\;\delta.
\]
\end{lemma}

\begin{proof}
With $\mathbf{x}\in[-R,R]^d$, $\|\mathbf{x}\|\le R\sqrt d$. Compute
\[
\kYat(\alpha\mathbf{w}_\star,\mathbf{x})
\;=\;\frac{(\mathbf{w}_\star^\top\mathbf{x})^2}{1+\eta(\mathbf{x},\alpha)},
\qquad
\eta(\mathbf{x},\alpha)\;:=\;-\frac{2\,\mathbf{w}_\star^\top\mathbf{x}}{\alpha}+\frac{\|\mathbf{x}\|^2+\varepsilon}{\alpha^2}.
\]
On $[-R,R]^d$, $|\mathbf{w}_\star^\top\mathbf{x}|\le R\sqrt d$, so
\[
|\eta|\le \frac{2R\sqrt d}{\alpha}+\frac{dR^2+\varepsilon}{\alpha^2}.
\]
The first two terms in the definition of $\alpha_0$ make these two summands at most
$1/4$ each, hence $|\eta|\le1/2$ for every $\delta>0$. Therefore
$|1/(1+\eta)-1|\le2|\eta|$. The last two terms in $\alpha_0$ give
\[
dR^2\frac{4R\sqrt d}{\alpha}\le\frac\delta2,
\qquad
dR^2\frac{2(dR^2+\varepsilon)}{\alpha^2}\le\frac\delta2.
\]
Combining these estimates gives
\[
\bigl|\kYat(\alpha\mathbf{w}_\star,\mathbf{x})-(\mathbf{w}_\star^\top\mathbf{x})^2\bigr|
\le dR^2\left(\frac{4R\sqrt d}{\alpha}
+\frac{2(dR^2+\varepsilon)}{\alpha^2}\right)\le\delta.
\]
\end{proof}

\begin{proposition}[Spectral Yat approximation of any symmetric quadratic form]
\label{prop:spectral-yat-approx}
Fix $R>0$ and $\varepsilon>0$. Let $\mathbf{A}\in\mathrm{Sym}(d)$ have rank $r$ with spectral decomposition $\mathbf{A}=\sum_{k=1}^r\lambda_k\mathbf{e}_k\mathbf{e}_k^\top$. For every $\delta>0$, set $\delta_k:=\delta/(r|\lambda_k|)$ and choose $\alpha_k\ge\alpha_0(d,R,\varepsilon,\delta_k)$ as in Lemma~\ref{lem:single-atom-rank-one}. Then
\[
G(\mathbf{x})\;:=\;\sum_{k=1}^r\lambda_k\,\kYat(\alpha_k\mathbf{e}_k,\mathbf{x})
\]
is an unbiased Yat expansion with at most $r$ atoms satisfying $\sup_{\mathbf{x}\in[-R,R]^d}|G(\mathbf{x})-\mathbf{x}^\top\mathbf{A}\mathbf{x}|\le\delta$.
For $\mathbf A=0$, take the empty expansion $G=0$.
\end{proposition}

\begin{proof}
$\mathbf{x}^\top\mathbf{A}\mathbf{x}=\sum_{k=1}^r\lambda_k(\mathbf{e}_k^\top\mathbf{x})^2$. Lemma~\ref{lem:single-atom-rank-one} applied to each unit eigenvector $\mathbf{e}_k$ with tolerance $\delta_k$ gives $|\kYat(\alpha_k\mathbf{e}_k,\mathbf{x})-(\mathbf{e}_k^\top\mathbf{x})^2|\le\delta/(r|\lambda_k|)$. Multiplying by $|\lambda_k|$ and summing gives the displayed bound.
\end{proof}

\begin{lemma}[Spherical average of an IMQ section]
\label{lem:imq-spherical-average}
Let $k\ge4$, let $a>0$ and $\tau>0$, and let $\mathbf U$ be uniform on the
radius-$a$ sphere in $\R^k$. Then, for every $\mathbf v\in\R^k$,
\[
\mathbb E\frac{1}{\|\mathbf U-\mathbf v\|^2+\tau}
\le \frac{1}{a^2+\tau}.
\]
\end{lemma}

\begin{proof}
For $f_\tau(\mathbf y)=(\|\mathbf y\|^2+\tau)^{-1}$, direct differentiation gives
\[
\Delta f_\tau(\mathbf y)=
\frac{(8-2k)\|\mathbf y\|^2-2k\tau}
{(\|\mathbf y\|^2+\tau)^3}\le0.
\]
The function $J(\mathbf v)=\mathbb E f_\tau(\mathbf U-\mathbf v)$ is smooth,
radial, and superharmonic by differentiation under the integral. Writing
$J(\mathbf v)=j(r)$, smooth radial symmetry and
$(r^{k-1}j'(r))'=r^{k-1}\Delta J(r)\le0$ imply $j'(r)\le0$. Therefore
$J(\mathbf v)\le J(\mathbf0)=(a^2+\tau)^{-1}$.
\end{proof}

\begin{proposition}[IMQ coefficient-budget obstruction for a PSD quadratic target]
\label{prop:imq-slice-lower}
Let $d\ge5$, let $\mathbf{A}\in\mathrm{Sym}(d)$ be PSD with
$\lambda_{\max}:=\lambda_{\max}(\mathbf{A})>0$, and fix $R,\varepsilon>0$ and
$\delta\ge0$. If $F\in\mathcal{V}_\IMQ(\Lambda)$ satisfies
$\sup_{[-R,R]^d}|F-\mathbf{x}^\top\mathbf{A}\mathbf{x}|\le\delta$, then
\[
\Lambda\ge
\sup_{0\le t\le1}
\bigl(\lambda_{\max}R^2t-\delta\bigr)_+
\bigl(R^2(1-t)+\varepsilon\bigr).
\]
\end{proposition}

\begin{proof}
Let $\mathbf e_1$ be a unit eigenvector for $\lambda_{\max}$. Fix $0\le t<1$
and average over the sphere
\[
\mathbf x=R\sqrt t\,\mathbf e_1+\mathbf u,
\qquad \mathbf u\perp\mathbf e_1,
\qquad \|\mathbf u\|=R\sqrt{1-t}.
\]
This sphere lies in the radius-$R$ ball and hence in the cube. Positivity of
$\mathbf A$ gives
$\mathbf x^\top\mathbf A\mathbf x\ge\lambda_{\max}R^2t$, so its spherical
average of $F$ is at least $q_t:=\lambda_{\max}R^2t-\delta$.
Decompose an arbitrary center as $\mathbf v=s\mathbf e_1+\mathbf v_\perp$.
Lemma~\ref{lem:imq-spherical-average}, with ambient sphere dimension $d-1\ge4$,
$a=R\sqrt{1-t}$, and $\tau=\varepsilon+(s-R\sqrt t)^2$, yields
\[
\int h_\varepsilon(\mathbf x,\mathbf v)\,d\mu_t(\mathbf x)
\le\frac{1}{R^2(1-t)+\varepsilon}.
\]
Thus, whenever $q_t>0$,
\[
q_t\le\int F\,d\mu_t
\le\sum_j|a_j|\int h_\varepsilon(\mathbf x,\mathbf v_j)\,d\mu_t
\le\frac{\Lambda}{R^2(1-t)+\varepsilon}.
\]
The same endpoint inequality at $t=1$ follows directly from
$h_\varepsilon\le\varepsilon^{-1}$. Taking the supremum over $t$ proves the claim.
\end{proof}

\begin{proof}[Proof of Theorem~\ref{thm:main-atom-separation}]
Part~\textup{(i)} is Proposition~\ref{prop:spectral-yat-approx}. Part~\textup{(ii)} is
Proposition~\ref{prop:imq-slice-lower}; setting $t=1/4$ gives the displayed special
case. Neither argument constrains the IMQ center locations.
\end{proof}

\paragraph{Explicit budget and atom normalization.}
The supremum in~\eqref{eq:main-imq-budget} is explicit. With
$L=\lambda_{\max}R^2$ and $S=R^2$,
\[
\sup_{0\le t\le1}(Lt-\delta)_+(S(1-t)+\varepsilon)
=
\begin{cases}
\dfrac{[L(S+\varepsilon)-\delta S]^2}{4LS},
& \varepsilon/S+\delta/L\le1,\\[6pt]
(L-\delta)\varepsilon,
& \varepsilon/S+\delta/L>1.
\end{cases}
\]
The obstruction depends only on total coefficient variation and does not restrict the
number or locations of IMQ centers. For example, when
$R=\lambda_{\max}=1$, $\delta=1/8$, and $\varepsilon=1/64$, every accurate IMQ
expansion must have $\Lambda\ge3249/16384\approx0.1983$; the scalar supremum is
attained at $t=73/128$. This does not assert that an approximant attains that budget.
The normalization here is $h_\varepsilon=(\|\mathbf x-\mathbf v\|^2+\varepsilon)^{-1}$;
using atoms $\varepsilon h_\varepsilon$ divides the coefficients of the same
expansion by $\varepsilon$. Thus the concrete budget
$\Lambda=1/32$ makes the approximation class empty, rather than forcing a large
but finite width.

\begin{theorem}[Yat-native approximation and IMQ budget obstruction]
\label{thm:yat-native-separation}
Let $f^\star=p^\star+g^\star$ on $[-R,R]^d$, where
$p^\star(\mathbf{x})=\mathbf{x}^\top\mathbf{A}\mathbf{x}$ for a PSD
$\mathbf{A}\in\mathrm{Sym}(d)$ of rank $r$ with $\lambda_{\max}(\mathbf{A})>0$, and
$g^\star\in\mathcal{H}_\IMQ^\varepsilon$ satisfies
$\|g^\star\|_{L^\infty([-R,R]^d)}\le M$. Assume there is a nonincreasing function
$\phi:\mathbb{N}\to[0,\infty)$ with $\phi(m)\downarrow0$ such that for every
$m\in\mathbb{N}$ there exists $\bar F_m\in F_{\IMQ,\varepsilon}$ using at most
$m$ atoms and satisfying
\[
\|g^\star-\bar F_m\|_{L^\infty([-R,R]^d)}\le\phi(m).
\]
Fix $\delta>0$ with $\delta+M<\lambda_{\max}R^2/4$, and define the finite integer
$m_\delta:=\min\{m\in\mathbb{N}:\phi(m)\le\delta/2\}$.

\medskip\noindent\textbf{(i) Yat upper bound.} There exists
$G\in F_{\yat}^{\ge0}$ with at most $r+3m_\delta$ atoms and
$\sup_{\mathbf{x}\in[-R,R]^d}|f^\star(\mathbf{x})-G(\mathbf{x})|\le\delta$.

\medskip\noindent\textbf{(ii) IMQ budget obstruction.} Suppose $d\ge5$. Every
$F\in\mathcal{V}_\IMQ(\Lambda)$ with
$\sup_{[-R,R]^d}|f^\star-F|\le\delta$ necessarily satisfies
\[
\Lambda\ge
\sup_{0\le t\le1}
\bigl(\lambda_{\max}R^2t-(\delta+M)\bigr)_+
\bigl(R^2(1-t)+\varepsilon\bigr).
\]
In particular, at $t=1/4$,
\[
\Lambda\ge
\left(\frac{\lambda_{\max}R^2}{4}-(\delta+M)\right)
\left(\frac{3R^2}{4}+\varepsilon\right).
\]
\end{theorem}

\begin{proof}
\textbf{(i)} Apply Proposition~\ref{prop:spectral-yat-approx} to $\mathbf{A}$ with
tolerance $\delta/2$, producing an unbiased Yat expansion $G_1\in F_\yat^{\ge0}$
with $r$ atoms and $\sup|G_1-p^\star|\le\delta/2$. By definition of $m_\delta$,
choose $\bar F_{m_\delta}\in F_{\IMQ,\varepsilon}$ with at most $m_\delta$ IMQ
atoms and $\|g^\star-\bar F_{m_\delta}\|_{L^\infty}\le\delta/2$.
Theorem~\ref{thm:imq-embed} converts $\bar F_{m_\delta}$ into a positive-bias Yat
expansion $G_2\in F_{\yat,\varepsilon}^{+}\subset F_\yat^{\ge0}$ with at most
$3m_\delta$ atoms and $G_2=\bar F_{m_\delta}$ pointwise. Set $G:=G_1+G_2$;
the triangle inequality gives the displayed atom count and error bound.

\textbf{(ii)} Given $\sup|f^\star-F|\le\delta$, the bound on $g^\star$ gives
$\sup|p^\star-F|\le\delta+M$. Proposition~\ref{prop:imq-slice-lower} applied with
tolerance $\delta+M$ gives the displayed necessary budget.
\end{proof}

\begin{remark}[RKHS-norm cost of the Yat upper bound]\label{rem:yat-norm-cost}
Theorem~\ref{thm:yat-native-separation}(i) is an \emph{atom-count} statement, not an
RKHS-norm statement. The single-atom rank-one approximation
(Lemma~\ref{lem:single-atom-rank-one}) uses a center at scale $\alpha_0$. Choosing
$\alpha=\alpha_0$, the corresponding section has the exact squared norm
\[
\|k_{0,\varepsilon}(\alpha_0\mathbf{w}_\star,\cdot)\|_{\mathcal{H}_{0,\varepsilon}}^2
=k_{0,\varepsilon}(\alpha_0\mathbf{w}_\star,\alpha_0\mathbf{w}_\star)
=\frac{\alpha_0^4}{\varepsilon}.
\]
In the high-accuracy regime where the $\delta^{-1}$ term dominates, this is
$\Theta(R^{12}d^6/(\varepsilon\delta^4))$. For the rank-$r$ expansion of
Proposition~\ref{prop:spectral-yat-approx}, the total RKHS norm also depends on the
eigenvalues $\lambda_k$ and on cross-inner-products between the $r$ sections; no
dimension-only norm comparison follows from the upper construction. The obstruction in
Theorem~\ref{thm:yat-native-separation} concerns total IMQ coefficient variation, not
atom count, RKHS-ball radius, or generalization error.
\end{remark}

\subsection{Polynomial separation}

\begin{theorem}[Polynomial approximation gap on line-containing domains]
\label{thm:poly-gap}
Let \(\mathbf{w}_\star\neq 0\), \(b\ge 0\), \(\varepsilon>0\), and
\(g_\star(\mathbf{x})=k_{b,\varepsilon}(\mathbf{w}_\star,\mathbf{x})\).
Let \(\mathcal P_2=\{\mathbf{x}\mapsto \mathbf{x}^\top \mathbf{A}\mathbf{x}+\mathbf{a}^\top \mathbf{x}+c:\mathbf{A}=\mathbf{A}^\top\}\) be the space of degree-at-most-two polynomials. Suppose \(X\subset\R^d\) contains a nontrivial line segment \(\{t\mathbf{v}:t\in I\}\) with \(\mathbf{v}:=\mathbf{w}_\star/\|\mathbf{w}_\star\|\) and \(I\subset\R\) a compact interval with nonempty interior. Then
\[
\inf_{P\in\mathcal P_2}\|g_\star-P\|_{L^\infty(X)}>0.
\]
\end{theorem}

\begin{proof}
Let \(\rho:=\|\mathbf{w}_\star\|>0\). Along \(\mathbf{x}=t\mathbf{v}\), \(h(t):=g_\star(t\mathbf{v})=(\rho t+b)^2/((t-\rho)^2+\varepsilon)\).
Restrictions of \(\mathcal P_2\) to this line are univariate polynomials of degree at most two.
Suppose \(h(t)=p(t)\) for all \(t\in I\) with \(p\in\Pi_2\). By real-analyticity this extends globally:
\((\rho t+b)^2=p(t)((t-\rho)^2+\varepsilon)\). If \(p\neq 0\) has degree \(k\), the right side has degree
\(k+2\) (since \((t-\rho)^2+\varepsilon\) is monic quadratic), while the left has degree two, so \(k=0\), i.e.\ \(p=C\) constant. Then \((\rho t+b)^2=C((t-\rho)^2+\varepsilon)\). If \(C=0\): impossible since \(\rho>0\). If \(C\neq 0\): the left side has real double root \(-b/\rho\) while \((t-\rho)^2+\varepsilon\) has nonreal roots \(\rho\pm i\sqrt{\varepsilon}\); contradiction. Hence \(h\notin\Pi_2\), and since \(\Pi_2\) is closed in \(C(I)\), its distance from \(h\) is strictly positive.
\end{proof}

\section{Layer Norms, Capacity, and Composition}
\label{app:layer-theory}

\subsection{Layer-local squared Gram norm and Rademacher complexity}
\label{app:rademacher-proof}

\begin{proposition}[Closed-form squared RKHS norm]
\label{prop:rkhs-norm}
Let $b\ge 0$, $\varepsilon>0$, and let
$f(\cdot)=\sum_{j=1}^m\alpha_j\,k_{b,\varepsilon}(\mathbf{w}_j,\cdot)\in\mathcal{H}_{b,\varepsilon}$
be a finite kernel expansion with shared $(b,\varepsilon)$. Then
\begin{equation}
\label{eq:rkhs-norm}
\|f\|_{\mathcal{H}_{b,\varepsilon}}^{2}
\;=\; \boldsymbol{\alpha}^\top\mathbf{K}\boldsymbol{\alpha},
\qquad \mathbf{K}_{ij}=k_{b,\varepsilon}(\mathbf{w}_i,\mathbf{w}_j).
\end{equation}
The identity holds for any choice of $\{\mathbf{w}_j\}$, including repeated centers; in that case $\mathbf{K}$ may be singular, and although different coefficient vectors can represent the same function $f$, their difference lies in $\ker\mathbf{K}$ and the quadratic form $\boldsymbol{\alpha}^\top\mathbf{K}\boldsymbol{\alpha}$ is invariant across equivalent representations.
\end{proposition}

\begin{proof}
By bilinearity and the reproducing property,
\begin{align*}
\|f\|^2_{\mathcal{H}_{b,\varepsilon}}
&\;=\; \sum_{i,j}\alpha_i\alpha_j\,\bigl\langle k_{b,\varepsilon}(\mathbf{w}_i,\cdot),\,k_{b,\varepsilon}(\mathbf{w}_j,\cdot)\bigr\rangle_{\mathcal{H}_{b,\varepsilon}} \\
&\;=\;\sum_{i,j}\alpha_i\alpha_j\,k_{b,\varepsilon}(\mathbf{w}_i,\mathbf{w}_j)\;=\;\boldsymbol{\alpha}^\top\mathbf{K}\boldsymbol{\alpha}. \qedhere
\end{align*}
\end{proof}

\noindent
The identity $\|f\|_{\mathcal{H}_{b,\varepsilon}}^2=\boldsymbol{\alpha}^\top\mathbf{K}\boldsymbol{\alpha}$ remains exact when $\mathbf{K}$ is rank-deficient. If $\boldsymbol{\alpha}$ and $\boldsymbol{\alpha}'$ represent the same function, then $\mathbf{K}(\boldsymbol{\alpha}-\boldsymbol{\alpha}')=0$ and both quadratic forms equal $\langle f,f\rangle_{\mathcal{H}_{b,\varepsilon}}$.

For a PSD Gram matrix, $\mathbf{K}\mathbf{K}^\dagger\mathbf{K}=\mathbf{K}$ gives the equivalent expression
$(\mathbf{K}\boldsymbol{\alpha})^\top\mathbf{K}^\dagger(\mathbf{K}\boldsymbol{\alpha})=\boldsymbol{\alpha}^\top\mathbf{K}\boldsymbol{\alpha}$. Thus coincident or near-coincident centers affect numerical conditioning, not the RKHS identity. In floating-point computation, one may deduplicate identical centers and use a PSD-aware eigendecomposition or pseudoinverse. When $b>0$ and the retained centers are distinct, Corollary~\ref{cor:char-spd} guarantees a strictly positive-definite Gram matrix, for which a stable Cholesky factorization is available.

\begin{corollary}[Common-logit-invariant multiclass norm]
\label{cor:centered-class-norm}
Let $f_c=\sum_{j=1}^m\alpha_{jc}k_{b,\varepsilon}(\mathbf w_j,\cdot)$ for
$c=1,\ldots,C$, define
$\bar{\boldsymbol\alpha}=C^{-1}\sum_c\boldsymbol\alpha_c$, and let
$\bar f=C^{-1}\sum_c f_c$. Then
\[
\|f_c-\bar f\|_{\mathcal H_{b,\varepsilon}}^2
=(\boldsymbol\alpha_c-\bar{\boldsymbol\alpha})^\top
\mathbf K(\boldsymbol\alpha_c-\bar{\boldsymbol\alpha}).
\]
These $C$ quantities are unchanged when the same RKHS function is added to every
class logit. Each is the squared RKHS norm of a centered class component. Their sum
is the squared quotient norm on $\mathcal H_{b,\varepsilon}^C$ modulo common logits:
\[
\inf_{h\in\mathcal H_{b,\varepsilon}}\sum_{c=1}^C\|f_c-h\|_{\mathcal H_{b,\varepsilon}}^2
=\sum_{c=1}^C\|f_c-\bar f\|_{\mathcal H_{b,\varepsilon}}^2.
\]
\end{corollary}

\phantomsection
\label{app:centered-class-proof}
\begin{proof}[Proof of Corollary~\ref{cor:centered-class-norm}]
Proposition~\ref{prop:rkhs-norm} applied to $f_c-\bar f$ gives the identity.
Adding any common $h\in\mathcal H_{b,\varepsilon}$ changes the mean to $\bar f+h$,
leaving every centered function unchanged. Since $\sum_c(f_c-\bar f)=0$,
\[
\sum_c\|f_c-h\|^2=\sum_c\|f_c-\bar f\|^2+C\|\bar f-h\|^2.
\]
The minimum is attained at $h=\bar f$, which proves the quotient identity.
\end{proof}

\begin{theorem}[Single-layer Rademacher bound]
\label{thm:rademacher}
Fix $b\ge 0$, $\varepsilon>0$. Let
$\mathcal{F}_{B,b,\varepsilon}=\{f\in\mathcal{H}_{b,\varepsilon}:\|f\|_{\mathcal{H}_{b,\varepsilon}}\le B\}$
be the RKHS ball of radius $B$, and let $\{\mathbf{x}_i\}_{i=1}^n$ be i.i.d.\ samples
satisfying $\|\mathbf{x}_i\|\le R$ almost surely. With the convention
$\widehat{\mathrm{Rad}}_n(\mathcal{F}):=\mathbb{E}_{\boldsymbol{\sigma}}\sup_{f\in\mathcal{F}}\frac{1}{n}\sum_{i=1}^n\sigma_i f(\mathbf{x}_i)$,
$\sigma_i\stackrel{\mathrm{iid}}{\sim}\mathrm{Unif}\{\pm 1\}$,
\begin{equation}
\label{eq:rademacher}
\widehat{\mathrm{Rad}}_n(\mathcal{F}_{B,b,\varepsilon})
\;\le\; \frac{B(R^2+b)}{\sqrt{n}\,\sqrt{\varepsilon}}.
\end{equation}
\end{theorem}

\begin{proof}
For any $f\in\mathcal{F}_{B,b,\varepsilon}$ and $\mathbf{x}\in\R^d$, the reproducing property of the fixed kernel $k_{b,\varepsilon}$ gives $f(\mathbf{x})=\langle f, k_{b,\varepsilon}(\cdot,\mathbf{x})\rangle_{\mathcal{H}_{b,\varepsilon}}$. By Cauchy--Schwarz,
\begin{equation}
\label{eq:f-cs-app}
|f(\mathbf{x})| \;\le\; \|f\|_{\mathcal{H}_{b,\varepsilon}}\sqrt{k_{b,\varepsilon}(\mathbf{x},\mathbf{x})} \;\le\; B\sqrt{k_{b,\varepsilon}(\mathbf{x},\mathbf{x})}.
\end{equation}
The diagonal of the biased Yat kernel is
\[
k_{b,\varepsilon}(\mathbf{x},\mathbf{x})\;=\;\frac{(\|\mathbf{x}\|^2+b)^2}{0+\varepsilon}\;=\;\frac{(\|\mathbf{x}\|^2+b)^2}{\varepsilon},
\]
since $\mathbf{x}^\top\mathbf{x}+b=\|\mathbf{x}\|^2+b$ and $\|\mathbf{x}-\mathbf{x}\|^2+\varepsilon=\varepsilon$. For $\|\mathbf{x}\|\le R$,
\[
\sqrt{k_{b,\varepsilon}(\mathbf{x},\mathbf{x})} \;\le\; \frac{\|\mathbf{x}\|^2+b}{\sqrt{\varepsilon}} \;\le\; \frac{R^2+b}{\sqrt{\varepsilon}}.
\]
The standard RKHS Rademacher bound for a ball of radius $B$ follows from the reproducing property and Jensen's inequality (\citealt[Lemma~22]{bartlett2002}, with the one-sided convention $\widehat{\mathrm{Rad}}_n(\mathcal{F}):=\mathbb{E}_{\boldsymbol{\sigma}}\sup_{f\in\mathcal{F}}\frac{1}{n}\sum_i\sigma_i f(\mathbf{x}_i)$ used in~\eqref{eq:rademacher}). Writing $f(\mathbf{x}_i)=\langle f,k_{b,\varepsilon}(\cdot,\mathbf{x}_i)\rangle$, Hilbert-space duality gives the factor $B$. Jensen's inequality then gives $\mathbb{E}\|S\|\le\sqrt{\mathbb{E}\|S\|^2}$ for $S=\sum_i\sigma_i k_{b,\varepsilon}(\cdot,\mathbf{x}_i)$:
\[
\widehat{\mathrm{Rad}}_n(\mathcal{F}_{B,b,\varepsilon})
\;=\;\frac{B}{n}\,\mathbb{E}_{\boldsymbol{\sigma}}\Big\|\sum_{i}\sigma_i k_{b,\varepsilon}(\cdot,\mathbf{x}_i)\Big\|_{\mathcal{H}_{b,\varepsilon}}
\;\le\; \frac{B}{\sqrt{n}}\cdot\sqrt{\frac{1}{n}\sum_{i=1}^{n}k_{b,\varepsilon}(\mathbf{x}_i,\mathbf{x}_i)}.
\]
The constant is exactly $1$ (no symmetrization factor of $2$) under this one-sided convention. Bounding the per-sample term by $(R^2+b)^2/\varepsilon$ and taking the square root gives
\[
\widehat{\mathrm{Rad}}_n(\mathcal{F}_{B,b,\varepsilon})
\;\le\; B\sqrt{\frac{1}{n}\cdot\frac{(R^2+b)^2}{\varepsilon}}
\;=\; \frac{B(R^2+b)}{\sqrt{n}\,\sqrt{\varepsilon}}.
\]
\end{proof}

\begin{corollary*}[Unbiased case]
At $b=0$, $\widehat{\mathrm{Rad}}_n(\mathcal{F}_{B,0,\varepsilon})\le BR^2/(\sqrt{n}\,\sqrt{\varepsilon})$.
\end{corollary*}

\begin{proposition}[Uniform domination over learned kernel parameters]
\label{prop:hyperparameter-domination}
Fix $b_{\max}\ge0$ and $\varepsilon_{\min}>0$. For every
$0\le b\le b_{\max}$ and $\varepsilon\ge\varepsilon_{\min}$,
\[
k_{b,\varepsilon}\preceq
k_{b_{\max},\varepsilon_{\min}}.
\]
Consequently, every radius-$B$ ball in $\mathcal H_{b,\varepsilon}$ embeds
contractively into the radius-$B$ ball of the single reference RKHS
$\mathcal H_{b_{\max},\varepsilon_{\min}}$.
\end{proposition}

\phantomsection
\label{app:parameter-domination-proof}
\begin{proof}[Proof of Proposition~\ref{prop:hyperparameter-domination}]
Write $p_b(\mathbf x,\mathbf w)=(\mathbf x^\top\mathbf w+b)^2$ and
$h_\varepsilon(\mathbf x,\mathbf w)=D_\varepsilon(\mathbf x,\mathbf w)^{-1}$.
Both
$p_{b_{\max}}-p_b=2(b_{\max}-b)\mathbf x^\top\mathbf w
+(b_{\max}^2-b^2)$ and
$h_{\varepsilon_{\min}}-h_\varepsilon$ are PSD; for the latter, use the
Laplace--Gaussian mixture and the nonnegative weight
$e^{-t\varepsilon_{\min}}-e^{-t\varepsilon}$. Hence the Schur-product theorem gives
\[
k_{b_{\max},\varepsilon_{\min}}-k_{b,\varepsilon}
=(p_{b_{\max}}-p_b)h_{\varepsilon_{\min}}
+p_b(h_{\varepsilon_{\min}}-h_\varepsilon)\succeq0.
\]
The contractive RKHS embedding is the standard consequence of kernel order.
\end{proof}

\subsection{Data-dependent capacity refinement}
\label{app:proofs-capacity-extra}

\begin{corollary}[Data-dependent empirical bound]
\label{cor:rademacher-empirical}
For fixed shared $b\ge0$, $\varepsilon>0$, and any finite sample
$\mathbf x_1,\ldots,\mathbf x_n\in\R^d$, the radius-$B$ ball in
$\mathcal H_{b,\varepsilon}$ satisfies the data-dependent bound
\begin{equation}
\label{eq:rademacher-empirical}
\widehat{\mathrm{Rad}}_n(\mathcal{F}_{B,b,\varepsilon})
\;\le\;
\frac{B}{\sqrt{n\varepsilon}}
\sqrt{\frac{1}{n}\sum_{i=1}^n(\|\mathbf{x}_i\|^2+b)^2},
\end{equation}
computable from the sample alone and at most $B(R^2+b)/(\sqrt{n}\,\sqrt{\varepsilon})$
on any sample with $\max_i\|\mathbf{x}_i\|\le R$.
\end{corollary}

\begin{proof}[Proof of Corollary~\ref{cor:rademacher-empirical}]
The proof of Theorem~\ref{thm:rademacher} bounds $\widehat{\mathrm{Rad}}_n$ by
$(B/\sqrt{n})\sqrt{(1/n)\sum_i k_{b,\varepsilon}(\mathbf{x}_i,\mathbf{x}_i)}$ before applying
the worst-case bound on the diagonal. Using the Yat-kernel diagonal
$k_{b,\varepsilon}(\mathbf{x}_i,\mathbf{x}_i)=(\|\mathbf{x}_i\|^2+b)^2/\varepsilon$ directly
gives~\eqref{eq:rademacher-empirical}.
\end{proof}

\paragraph{Input scale and parameter selection.}
Corollary~\ref{cor:rademacher-empirical} replaces the worst-case factor $R^2+b$ by
$\sqrt{n^{-1}\sum_i(\|\mathbf{x}_i\|^2+b)^2}$. Under isotropic Gaussian inputs, this
empirical factor has leading scale $d+b$ by the law of large numbers and standard
chi-square concentration, so the empirical bound has leading scale
$B(d+b)/\sqrt{n\varepsilon}$. This is a probabilistic consequence of the empirical
corollary, not of the almost-sure radius assumption in Theorem~\ref{thm:rademacher};
the Gaussian RBF and IMQ diagonals are constant.

The Rademacher bound is layer-local: $B$ controls the norm of one
shared-$(b,\varepsilon)$ Yat layer and $R$ is a data-side radius on that layer's input.
If $(b,\varepsilon)$ is selected using the same sample, Theorem~\ref{thm:rademacher}
at the final learned values is not by itself a uniform generalization bound.
Proposition~\ref{prop:hyperparameter-domination} supplies a conservative repair when
training enforces $b\le b_{\max}$, $\varepsilon\ge\varepsilon_{\min}$, and a common
radius $B$: use the reference-kernel diagonal
$(R^2+b_{\max})^2/\varepsilon_{\min}$. A bound expressed in the final learned values
requires separate selection, covering, stability, or sample-splitting control. Likewise,
freezing a prefix after fitting it on the same data does not alone turn its pullback
kernel into a sample-independent population guarantee.

\subsection{Pullback RKHS Structure for Trained Prefixes}
\label{app:pullback-loewner}

\subsubsection{Prefix-conditioned pullback}
\label{subsec:pullback-rkhs}

This section proves a prefix-conditioned pullback statement: once the prefix of a trained
network is fixed, the next Yat layer induces a valid RKHS on the original input space.

Let $X\subset\R^{d_0}$ be compact. For a depth-$L$ Yat stack, write
\[
\Phi_0(\mathbf{x})=\mathbf{x},\qquad \Phi_\ell(\mathbf{x})=T_\ell(\Phi_{\ell-1}(\mathbf{x})),
\]
where layer $\ell$ has coordinates
$[T_\ell(\mathbf{z})]_c=\sum_{j=1}^{m_\ell}a_{\ell,j,c}k_{b_\ell,\varepsilon_\ell}(\mathbf{w}_{\ell,j},\mathbf{z})$.
Write $k_\ell(\mathbf{z},\mathbf{z}'):=k_{b_\ell,\varepsilon_\ell}(\mathbf{z},\mathbf{z}')$ and
$[\mathbf{K}_\ell]_{ij}:=k_\ell(\mathbf{w}_{\ell,i},\mathbf{w}_{\ell,j})$.

When learned centers $\mathbf{w}_{\ell,j}$ do not lie in the prefix range $\Phi_{\ell-1}(X)$, we
view each layer coordinate as the restriction to $\Phi_{\ell-1}(X)$ of a function in the
global RKHS $\mathcal{H}_{k_\ell}$ on $\R^{d_{\ell-1}}$. Equivalently, define the
extended domain $D_\ell:=\Phi_{\ell-1}(X)\cup\{\mathbf{w}_{\ell,1},\ldots,\mathbf{w}_{\ell,m_\ell}\}$, apply
the RKHS construction on $D_\ell$, and restrict back to $\Phi_{\ell-1}(X)$; the restriction
RKHS carries the quotient norm. The norm bound below holds in either view.

\begin{theorem}[Prefix-conditioned pullback RKHS: PSD, quotient norm, universality]
\label{thm:pullback-rkhs}
Let $Z$ be a Hausdorff topological space, let $k$ be a continuous PSD kernel on $Z$, and
let $\Phi:X\to Z$ be measurable. Define
$\widetilde{k}(\mathbf{x},\mathbf{x}'):=k(\Phi(\mathbf{x}),\Phi(\mathbf{x}'))$ and
$\mathcal{S}:=\overline{\mathrm{span}}\{k(\Phi(\mathbf{x}),\cdot):\mathbf{x}\in X\}\subset\mathcal{H}_k$.

\medskip\noindent\textbf{(a) PSD and norm bound.} $\widetilde{k}$ is PSD on $X$. For every $f\in\mathcal{H}_k$, $f\circ\Phi\in\mathcal{H}_{\widetilde{k}}$ with $\|f\circ\Phi\|_{\mathcal{H}_{\widetilde{k}}}\le\|f\|_{\mathcal{H}_k}$.

\medskip\noindent\textbf{(b) Exact quotient-norm characterization.}
\label{prop:pullback-quotient-norm}%
\(
\|f\circ\Phi\|_{\mathcal{H}_{\widetilde{k}}} = \inf\{\|g\|_{\mathcal{H}_k}:g\in\mathcal{H}_k,\;g|_{\Phi(X)}=f|_{\Phi(X)}\} = \|P_{\mathcal{S}}f\|_{\mathcal{H}_k},
\)
where $P_{\mathcal{S}}$ is orthogonal projection onto $\mathcal{S}$; the infimum is attained at $P_{\mathcal{S}}f$.

\medskip\noindent\textbf{(c) Universality propagates through injective continuous prefixes.}
\label{thm:pullback-universality}%
If $X$ is compact, $\Phi$ is continuous and injective, and $k$ is universal on $\Phi(X)$, then $\widetilde{k}$ is universal on $X$.

\medskip\noindent\textbf{Application to layer-$\ell$ Yat.} Specialising to $k=k_{b_\ell,\varepsilon_\ell}$, $\Phi=\Phi_{\ell-1}$, $\widetilde{k}=\widetilde{k}_\ell$, part (a) gives the per-layer norm bound $\|[T_\ell\circ\Phi_{\ell-1}]_c\|_{\mathcal{H}_{\widetilde{k}_\ell}}^2\le\mathbf{a}_{\ell,c}^\top \mathbf{K}_\ell\mathbf{a}_{\ell,c}$, with equality (via (b)) iff $f_{\ell,c}\in\mathcal{S}_\ell:=\overline{\mathrm{span}}\{k_\ell(\Phi_{\ell-1}(\mathbf{x}),\cdot):\mathbf{x}\in X\}$. For $b_\ell>0$ and continuous injective $\Phi_{\ell-1}$, part (c) together with Proposition~\ref{prop:singular} yields universality of $\widetilde{k}_\ell$ on $X$.
\end{theorem}

\begin{proof}
\textbf{(a)} For $\{\mathbf{x}_i\}\subset X$ and $\{\beta_i\}\subset\R$, $\sum\beta_i\beta_j\widetilde{k}(\mathbf{x}_i,\mathbf{x}_j)=\sum\beta_i\beta_j k(\Phi(\mathbf{x}_i),\Phi(\mathbf{x}_j))\ge 0$ by PSD of $k$. Let $\psi(\mathbf{z}):=k(\mathbf{z},\cdot)$ be the canonical feature map and $\widetilde\psi(\mathbf{x}):=\psi(\Phi(\mathbf{x}))$, so $\widetilde{k}(\mathbf{x},\mathbf{x}')=\langle\widetilde\psi(\mathbf{x}),\widetilde\psi(\mathbf{x}')\rangle$. The reproducing property gives $(f\circ\Phi)(\mathbf{x})=\langle f,\widetilde\psi(\mathbf{x})\rangle_{\mathcal{H}_k}$, so $f\circ\Phi\in\mathcal{H}_{\widetilde{k}}$ with the displayed norm bound. Specialising $f=f_{\ell,c}=\sum_j a_{\ell,j,c}k_\ell(\mathbf{w}_{\ell,j},\cdot)$ and using $\|f_{\ell,c}\|_{\mathcal{H}_\ell}^2=\mathbf{a}_{\ell,c}^\top\mathbf{K}_\ell\mathbf{a}_{\ell,c}$ delivers the per-layer Yat-Gram bound.

\textbf{(b)} Define $U$ on the finite span of kernel sections by
$U(\widetilde{k}(\mathbf{x},\cdot))=k(\Phi(\mathbf{x}),\cdot)$. Agreement of all
kernel-section inner products makes this definition independent of the chosen finite
representation and makes $U$ an isometry on that span. It therefore extends uniquely by
density to an isometric isomorphism $U:\mathcal{H}_{\widetilde{k}}\to\mathcal{S}$.
Under $U$, $f\circ\Phi$ corresponds to $P_{\mathcal{S}}f$, so
$\|f\circ\Phi\|_{\mathcal{H}_{\widetilde{k}}}=\|P_{\mathcal{S}}f\|_{\mathcal{H}_k}$.
The affine set $\{g\in\mathcal{H}_k:g|_{\Phi(X)}=f|_{\Phi(X)}\}$ equals
$P_{\mathcal{S}}f+\mathcal{S}^\perp$, whose unique minimum-norm element is
$P_{\mathcal{S}}f$.

\textbf{(c)} $X$ compact and $Z$ Hausdorff with $\Phi$ continuous injective makes $\Phi:X\to\Phi(X)$ a homeomorphism. For $g\in C(X)$, set $h:=g\circ\Phi^{-1}\in C(\Phi(X))$; by universality of $k$, find $f\in\mathcal{H}_k$ with $\sup_{\Phi(X)}|f-h|<\eta$. Then $f\circ\Phi\in\mathcal{H}_{\widetilde{k}}$ and $\sup_X|f\circ\Phi-g|=\sup_{\Phi(X)}|f-h|<\eta$.
\end{proof}

\paragraph{Why injectivity is necessary.}
For compact $X\subset\R^{d_0}$, suppose $\Phi(\mathbf{x})=\Phi(\mathbf{y})$ at
distinct points. The pullback sections at $\mathbf{x}$ and $\mathbf{y}$ coincide,
so the reproducing property gives $f(\mathbf{x})=f(\mathbf{y})$ for every
$f\in\mathcal H_{\widetilde k}$. For any $g\in C(X)$,
\[
\inf_{f\in\mathcal H_{\widetilde k}}\|f-g\|_{\infty,X}
\ge \tfrac12|g(\mathbf{x})-g(\mathbf{y})|.
\]
Indeed, $f(\mathbf{x})=f(\mathbf{y})$ and the triangle inequality give
$|g(\mathbf{x})-g(\mathbf{y})|\le2\|f-g\|_{\infty,X}$.
Taking $g(\mathbf{z})=(\mathbf{x}-\mathbf{y})^\top\mathbf{z}$ makes this lower
bound positive. Thus, when the base kernel is universal on the compact prefix image,
a continuous prefix gives a universal pullback if and only if it is injective.
This is a property of the full pullback RKHS; a particular finite set of trained
centers specifies only a finite-dimensional subspace.

\subsection{Ambient Sobolev-RKHS Containment for Deep Yat Stacks}
\label{app:global-sobolev-yat}

This appendix complements the exact single-layer squared Gram norm with a global ambient
space. Every bounded finite-depth Yat stack lies in one fixed Sobolev restriction RKHS
on the original input domain. Derivative bounds and a smooth cutoff provide the global
containment, while the prefix-conditioned pullback retains the exact layerwise geometry.

\subsubsection{Sobolev restriction RKHS}

Let $X\subset\R^{d_0}$ be compact and let $U\subset\R^{d_0}$ be a bounded
Lipschitz domain with $X\subset\overline{U}$. Here $H^s(U)$ is the restriction
of the Bessel-potential space $H^s(\R^{d_0})$, equipped with its quotient norm,
including for noninteger $s$. We first bound integer-order derivatives of the
network on a neighborhood of $\overline U$, then multiply by a smooth cutoff to
obtain an $H^s(\R^{d_0})$ extension. Fix $s>d_0/2$. Define
\[
\mathcal{H}_s(X):=\bigl\{f|_X:f\in H^s(U)\bigr\},
\qquad
\|g\|_{\mathcal{H}_s(X)}:=\inf\bigl\{\|f\|_{H^s(U)}:f|_X=g\bigr\}.
\]
By the Sobolev embedding $H^s(U)\hookrightarrow C(\overline{U})$ (valid for $s>d_0/2$),
point evaluation is continuous on $\mathcal{H}_s(X)$, making it an RKHS on $X$.

\begin{lemma}[Derivative bounds and Sobolev restriction]
\label{lem:sobolev-algebra-inverse-app}
Let $V\subset\R^{d_0}$ be a bounded open ball with $\overline U\subset V$,
let $k\ge0$ be an integer, and set
$\|f\|_{k,V}:=\max_{|\alpha|\le k}\sup_{x\in V}|\partial^\alpha f(x)|$
for smooth functions with bounded derivatives on $V$. Then:
\begin{enumerate}
\item $\|fg\|_{k,V}\le2^k\|f\|_{k,V}\|g\|_{k,V}$.
\item If $r\ge\lambda>0$ on $V$ and $\|r\|_{k,V}\le D$, then
$\|1/r\|_{k,V}\le\Gamma_k(\lambda^{-1},D)$, where the nondecreasing function
$\Gamma_k$ is defined by
\[
B_0=u,\qquad B_j=uD(2^j-1)\max_{0\le i<j}B_i\quad(1\le j\le k),
\qquad\Gamma_k(u,D)=\max_{0\le j\le k}B_j.
\]
\item If $0\le s\le k$, then
$\|f|_U\|_{H^s(U)}\le C_{s,k,U,V}\|f\|_{k,V}$.
\end{enumerate}
\end{lemma}

\begin{proof}
For (1), the multi-index Leibniz formula has coefficient sum
$\sum_{\beta\le\alpha}\binom\alpha\beta=2^{|\alpha|}\le2^k$.
For (2), write $v=1/r$. Differentiating $rv=1$ gives, for $|\alpha|=j\ge1$,
\[
\partial^\alpha v=-r^{-1}\sum_{0<\beta\le\alpha}
\binom\alpha\beta(\partial^\beta r)(\partial^{\alpha-\beta}v).
\]
Starting from $\|v\|_{0,V}\le\lambda^{-1}$, induction bounds each derivative
of order $j$ by $B_j$. The recurrence also proves monotonicity in $u,D$.
For (3), choose $\chi\in C_c^\infty(V)$ equal to one on a neighborhood of
$\overline U$. Extend $\chi f$ by zero outside $V$. Leibniz's formula and the
finite volume of $V$ bound its derivatives through order $k$ in $L^2$ by a
constant times $\|f\|_{k,V}$. Plancherel's identity and
$(1+\|\xi\|^2)^s\le(1+\|\xi\|^2)^k$ give
\[
\|f|_U\|_{H^s(U)}\le\|\chi f\|_{H^s(\R^{d_0})}
\le\|\chi f\|_{H^k(\R^{d_0})}\le C_{s,k,U,V}\|f\|_{k,V}.
\]
\end{proof}

\subsubsection{Network class and parameter budgets}

For a depth-$L$ Yat stack, write $\mathbf{z}_0(\mathbf{x})=\mathbf{x}\in\R^{d_0}$ and, for
$\ell=1,\dots,L$ and $c=1,\dots,d_\ell$,
\[
z_{\ell,c}(\mathbf{x})=\sum_{j=1}^{m_\ell}\alpha_{\ell,j,c}\,g_{\ell,j}(\mathbf{x}),
\qquad
g_{\ell,j}(\mathbf{x})=\frac{(\mathbf{w}_{\ell,j}^\top \mathbf{z}_{\ell-1}(\mathbf{x})+b_\ell)^2}
                     {\|\mathbf{z}_{\ell-1}(\mathbf{x})-\mathbf{w}_{\ell,j}\|^2+\varepsilon_\ell}.
\]
Assume the uniform budgets
\[
\|\mathbf{w}_{\ell,j}\|_2\le W_\ell,\qquad
|b_\ell|\le B_\ell,\qquad
0<\varepsilon_{\min}\le\varepsilon_\ell\le\varepsilon_{\max},\qquad
\sum_{j=1}^{m_\ell}|\alpha_{\ell,j,c}|\le A_\ell.
\]
Note $|b_\ell|\le B_\ell$ (not $b_\ell\ge0$): the Sobolev containment holds for any
bounded real bias; $b_\ell\ge0$ is required only for the Mercer/RKHS regime of the
single-layer results elsewhere in the paper.
The width $m_\ell$ enters only through $A_\ell$; if per-atom bounds
$|\alpha_{\ell,j,c}|\le a_\ell$ are assumed, take $A_\ell=m_\ell a_\ell$.

\subsubsection{Proof of global Sobolev-RKHS containment}

\begin{theorem}[Global ambient RKHS containment]
\label{thm:global-sobolev-yat}
Let $X\subset\R^{d_0}$ be compact and $U\supset X$ a bounded Lipschitz domain. Fix
$s>d_0/2$ and define $\mathcal{H}_s(X):=\{f|_X:f\in H^s(U)\}$ with the quotient norm.
Under the budgets $\|\mathbf{w}_{\ell,j}\|\le W_\ell$, $|b_\ell|\le B_\ell$,
$0<\varepsilon_{\min}\le\varepsilon_\ell\le\varepsilon_{\max}$,
$\sum_j|\alpha_{\ell,j,c}|\le A_\ell$, every coordinate $z_{\ell,c}|_X$ of every
admissible Yat stack belongs to $\mathcal{H}_s(X)$, with
$\|z_{\ell,c}\|_{\mathcal{H}_s(X)}\le M_\ell$ for finite constants $M_0,\dots,M_L$
depending only on $X,U,s,L,\{d_\ell,m_\ell,A_\ell,W_\ell,B_\ell\},\varepsilon_{\min},\varepsilon_{\max}$
and not on the trained parameters. The reproducing kernel of $\mathcal{H}_s(X)$ is
universal on $X$, hence characteristic.
\end{theorem}

\begin{proof}
Choose a fixed bounded ball $V$ with $\overline U\subset V$ and an integer
$k\ge s$. All network coordinates are smooth on $\R^{d_0}$ by induction,
since every denominator is at least $\varepsilon_{\min}$. We bound
$\|z_{\ell,c}\|_{k,V}$ inductively. Write $C_{\mathrm{alg}}=2^k$.

\emph{Base case ($\ell=0$).}
Each coordinate $z_{0,i}(\mathbf{x})=x_i$ has
$\|x_i\|_{k,V}\le N_0$ for a finite constant depending only on $V$ and $k$.

\emph{Inductive step.}
Assume $\|z_{\ell-1,i}\|_{k,V}\le N_{\ell-1}$
for all $i=1,\dots,d_{\ell-1}$.

\emph{Linear factor.}
$q_{\ell,j}:=\mathbf{w}_{\ell,j}^\top \mathbf{z}_{\ell-1}+b_\ell$ satisfies
\[
\|q_{\ell,j}\|_{k,V}\le\sqrt{d_{\ell-1}}W_\ell N_{\ell-1}+B_\ell
=:Q_\ell(N_{\ell-1}).
\]
By Lemma~\ref{lem:sobolev-algebra-inverse-app}(1),
$\|q_{\ell,j}^2\|_{k,V}\le C_{\mathrm{alg}}Q_\ell(N_{\ell-1})^2$.

\emph{Denominator.}
$r_{\ell,j}:=\|\mathbf{z}_{\ell-1}-\mathbf{w}_{\ell,j}\|^2+\varepsilon_\ell=\sum_i(z_{\ell-1,i}-w_{\ell,j,i})^2+\varepsilon_\ell$.
Algebra bounds give
$\|r_{\ell,j}\|_{k,V}\le d_{\ell-1}C_{\mathrm{alg}}(N_{\ell-1}+W_\ell)^2+\varepsilon_{\max}=:D_\ell(N_{\ell-1})$.
Pointwise, $r_{\ell,j}(\mathbf{x})\ge\varepsilon_{\min}>0$.
By Lemma~\ref{lem:sobolev-algebra-inverse-app}(2),
$\|1/r_{\ell,j}\|_{k,V}\le\Gamma_k(\varepsilon_{\min}^{-1},D_\ell(N_{\ell-1}))$.

\emph{Atom.}
$g_{\ell,j}=q_{\ell,j}^2\cdot(1/r_{\ell,j})$ obeys, by another application of (1),
\[
\|g_{\ell,j}\|_{k,V}
\le C_{\mathrm{alg}}^2 Q_\ell(N_{\ell-1})^2
\,\Gamma_k(\varepsilon_{\min}^{-1},D_\ell(N_{\ell-1}))
=:G_\ell(N_{\ell-1}).
\]

\emph{Layer coordinate.}
$\|z_{\ell,c}\|_{k,V}\le\sum_j|\alpha_{\ell,j,c}|\|g_{\ell,j}\|_{k,V}
\le A_\ell G_\ell(N_{\ell-1})=:N_\ell$.
This closes the derivative induction. Part~(3) of the lemma and the quotient
norm give
\[
\|z_{\ell,c}|_X\|_{\mathcal H_s(X)}
\le\|z_{\ell,c}|_U\|_{H^s(U)}\le C_{s,k,U,V}N_\ell=:M_\ell.
\]
The ball, cutoff and integer $k$ are fixed before training, so all constants
depend only on the stated data and budgets. Universality and characteristicness
follow from Corollary~\ref{cor:sobolev-ambient-universal}, whose density argument
depends only on the ambient restriction space.
\end{proof}

\begin{remark}[The ambient norm bound grows rapidly with depth]
\label{rem:sobolev-constants-growth}
The recurrence gives finite constants at every fixed depth, but the resulting bound is
not quantitatively useful for a deep stack. To make the growth statement precise,
suppose the derivative estimates above, after conversion to $M_\ell$, bound the update by
\[
1+M_\ell\le C(1+M_{\ell-1})^q,
\qquad C\ge1,\quad q>1,
\]
uniformly over layers. Iteration gives
\[
1+M_L\le C^{(q^L-1)/(q-1)}(1+M_0)^{q^L}=\exp\!\bigl(O(q^L)\bigr).
\]
Thus this standard coarse polynomial recursion is double-exponential in depth. Theorem
~\ref{thm:global-sobolev-yat} is a qualitative containment result, not a practical
deep-capacity bound; sharper composition estimates or a depth-adapted ambient function
space would be needed for useful depth dependence.
\end{remark}

\begin{corollary}[Universality and characteristicness of $\mathcal{H}_s(X)$]
\label{cor:sobolev-ambient-universal}
$\mathcal{H}_s(X)$ is universal on $X$, hence characteristic.
\end{corollary}

\begin{proof}
Every polynomial has bounded derivatives on the ball $V$, so
Lemma~\ref{lem:sobolev-algebra-inverse-app}(3), with an integer $k\ge s$,
places its restriction to $U$ in $H^s(U)$ and its restriction to $X$ in
$\mathcal{H}_s(X)$. By
Stone--Weierstrass, polynomials are dense in $C(X)$. Characteristicness follows from
universality on compact $X$~\citep{sriperumbudur2011universality}.
\end{proof}

\begin{remark}[Exact local norm and ambient global norm]
The identity
$\|\sum_j\boldsymbol{\alpha}_j k_{b,\varepsilon}(\mathbf{w}_j,\cdot)\|_{\mathcal{H}_{b,\varepsilon}}^2=\boldsymbol{\alpha}^\top\mathbf{K}\boldsymbol{\alpha}$
gives the exact single-layer norm. For a depth-$L$ stack, the exact induced kernel on the original input
space is the prefix-dependent pullback $\widetilde{k}_\ell(\mathbf{x},\mathbf{x}')=k_\ell(\mathbf{z}_{\ell-1}(\mathbf{x}),\mathbf{z}_{\ell-1}(\mathbf{x}'))$;
the fixed Sobolev RKHS $\mathcal{H}_s(X)$ provides a parameter-independent global
container with a coarser norm bound.
\end{remark}

\begin{remark}[Sign of bias and Mercer regime]
The proof requires only $|b_\ell|\le B_\ell$; the sign of $b_\ell$ is irrelevant for
Sobolev containment. The condition $b_\ell\ge0$ is required only for the Mercer/RKHS
structure of the single-layer Yat kernel (Theorem~\ref{thm:psd-mercer}). In the
common case where a single shared bias $b_\ell=b\ge0$ is used, the theorem applies
with $B_\ell=b$.
\end{remark}


\section{Experimental Protocols and Results}
\label{app:experiments}

\subsection{Gaussian and fixed-norm spectral diagnostic protocol.}
\label{app:spectral-protocol}
Figure~\ref{fig:hd-spectrum} uses $m=8$ orthonormal centers, dimensions
$d\in\{16,32,\ldots,2048\}$, $b=\varepsilon=1$, and $10^5$ independent inputs for
each of five seeds. To sample exactly without materializing all $d$ coordinates, the
first $m$ Gaussian coordinates are drawn explicitly and the remaining squared norm is
drawn from $\chi^2_{d-m}$. Gaussian inputs use these coordinates and their full random
norm. Fixed-norm inputs divide the same Gaussian vector by its norm and multiply by
$\sqrt d$. For the centered activation covariance $\widehat\Sigma$, the
reported effective rank is
$\operatorname{tr}(\widehat\Sigma)^2/\operatorname{tr}(\widehat\Sigma^2)$. The
supplement contains the implementation and every reported mean and standard deviation.

\subsection{Frozen-CLIP Diagnostics and Experimental Record}
\label{app:clip}

\paragraph{Experimental question.}
The frozen-CLIP evidence consists of a no-training scale/bandwidth sweep and a
trained-head norm--accuracy association. The supplement provides the complete
configurations, implementations, fitted state, and per-configuration results.

The trained diagnostic uses frozen CLIP ViT-B/32 ImageNet-1k features ($d=512$), a
width-$128$ Yat head, fixed $b=\varepsilon=1$, batch size $1024$, Adam at learning rate
$0.1$, seed $42$, and $30$ epochs. Optimization uses all $1{,}281{,}167$ examples in
the public \texttt{train} split. The reported $73.244\%$ accuracy and all per-class
correlations use the disjoint public \texttt{validation} split of $50{,}000$ examples.
This is a within-run geometric measurement; no examples are reassigned between splits.
The retained record fixes this configuration but does not contain a separate
hyperparameter-selection trace. The seven-point $(b,\varepsilon)$
stability grid repeats the same width, optimizer, epoch count, and single seed while
holding each selected parameter pair fixed.

\paragraph{Frozen-representation scale and bandwidth diagnostic.}
We test the concentration mechanism without fitting a classifier. The public CLIP
features are loaded as float32 and the kernel matrices are evaluated in float64; no
additional feature normalization is applied. From the frozen CLIP validation
representation, we select $12{,}000$ points and $64$ disjoint centers without
replacement at deterministic seed $20260804$. Let
$\varepsilon_0=\operatorname{median}_{i,j}\|\mathbf{x}_i-\mathbf{w}_j\|^2=1.026$.
We first scale both points and centers by
$s\in\{0.25,0.5,1,2,4\}$ at fixed $\varepsilon_0$. We then fix $s=1$ and sweep
$\varepsilon/\varepsilon_0\in\{0.01,0.1,1,10,100\}$. For each kernel, we report the
median coefficient of variation across units and the participation-ratio effective
rank of the sample-centered $64\times64$ activation covariance, using divisor $n-1$.
The RBF convention is
$\exp[-\|\mathbf x-\mathbf w\|^2/(2\varepsilon)]$. A unit enters the CV median only
when the absolute sample mean exceeds $10^{-12}$; the count of such units is recorded
to expose numerical underflow.

\begin{figure}[h]
\centering
\includegraphics[width=0.92\linewidth]{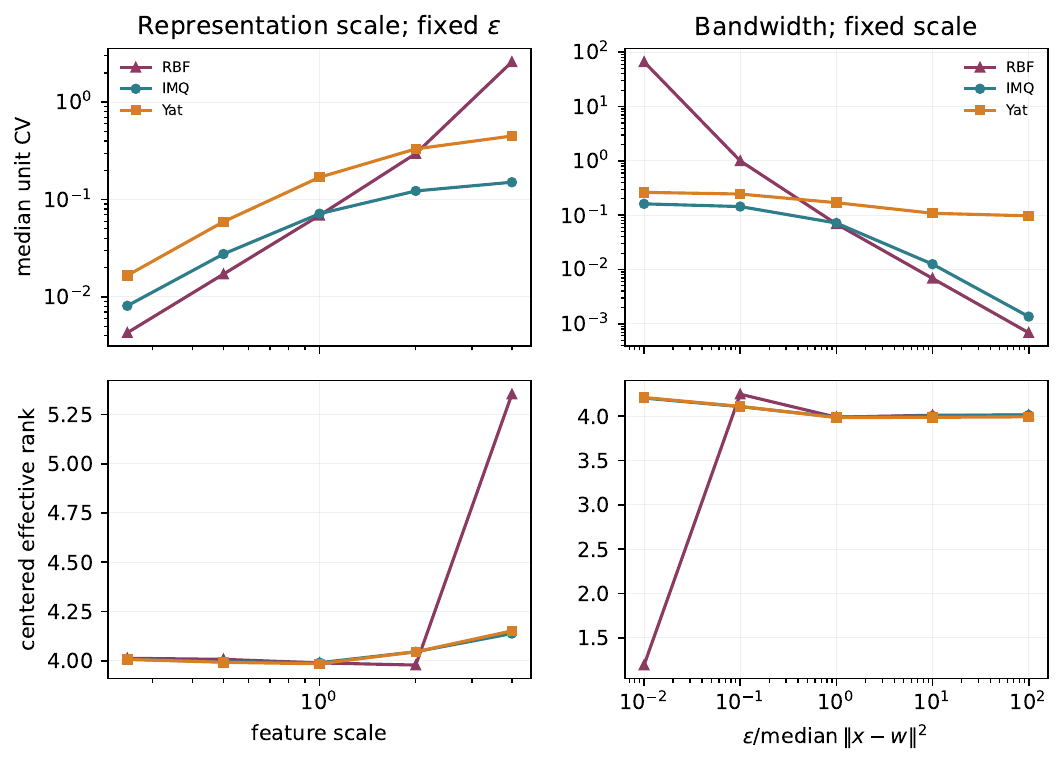}
\caption{\textbf{Purely radial units are sensitive to representation scale and
bandwidth on frozen CLIP features.} No classifier is trained. At
$\varepsilon/\varepsilon_0=100$, median unit CV is $6.85\times10^{-4}$ for RBF,
$1.36\times10^{-3}$ for IMQ, and $9.72\times10^{-2}$ for Yat. At
$\varepsilon/\varepsilon_0=0.01$, RBF has median CV $66.4$, centered effective rank
$1.19$, and only $43$ of $64$ units with numerically nonzero mean. The effective rank
of IMQ and Yat remains near $4$ across most of the sweep, showing that this learned
representation is far from the isotropic orthogonal-center model of
Theorem~\ref{thm:activation-covariance}. The measured conclusion is the relative
variation supplied by the alignment numerator.}
\label{fig:clip-scale-spectrum}
\end{figure}

The four-order-of-magnitude bandwidth sweep displays both radial failure modes described
in the introduction: broad kernels concentrate their activations, while a narrow RBF
can become numerically sparse. Yat is also scale-dependent, but its numerator supplies
a larger relative signal throughout the tested sweep. The diagnostic measures relative
activation variation with centers sampled from frozen features. Classifier-accuracy
attribution is a separate experiment. The supplement includes the CPU implementation
and complete per-configuration statistics.

\paragraph{Per-class squared RKHS norm correlation.}
For the recorded Yat head trained at learning rate $10^{-1}$, the
per-class closed-form squared RKHS norm $\boldsymbol{\alpha}_c^\top\mathbf{K}\boldsymbol{\alpha}_c$
(Proposition~\ref{prop:rkhs-norm}) has Spearman correlation $\rho=0.5628$ with
per-class top-1 accuracy across $n=1000$ classes (one observation per class, $50$
validation examples each). Raw class norms change when one common function is added to
every logit, although softmax predictions do not. We therefore also report the
common-logit-invariant statistic from Corollary~\ref{cor:centered-class-norm},
\[
(\boldsymbol\alpha_c-\bar{\boldsymbol\alpha})^\top
\mathbf K(\boldsymbol\alpha_c-\bar{\boldsymbol\alpha}),
\qquad
\bar{\boldsymbol\alpha}=1000^{-1}\sum_{c=1}^{1000}\boldsymbol\alpha_c.
\]
Its correlation is $\rho=0.5571$. The association therefore persists after removing
the common-logit ambiguity. We treat both correlations as descriptive statistics:
the $50$ examples per class and the shared trained head do not define independent
replicates for a conventional correlation test. Seven additional fixed-parameter,
single-seed runs give raw-norm correlations in $[0.537,0.557]$ and centered-norm
correlations in $[0.530,0.553]$ across the recorded $(b,\varepsilon)$ grid. These
descriptive associations measure fitted class geometry rather than the Rademacher bound.
Both norms are computed from trained centers and readout weights without evaluation
features.

The same Gram formula applies to learned-center RBF and IMQ heads. Yat contributes the
diagonal $(\|\mathbf{x}\|^2+b)^2/\varepsilon$, the directional trace, and the bias
finite-difference identity; the reproducing-property calculation is shared across
fixed-kernel heads.

\subsection{Directional-Tail Benchmark: Protocol and Full Results}
\label{app:directional}

We illustrate Corollary~\ref{cor:directional-gap} with a synthetic experiment in
$d=2$. The target is the single Yat atom
$g_\varepsilon(\mathbf{x};\mathbf{w}^*,b^*)$ with $\mathbf{w}^*=[1,0]$,
$b^*=1$, $\varepsilon=1$.  Along the ray $\mathbf{u}=[1,0]$ the alignment
factor satisfies $a=(\mathbf{u}^\top\mathbf{w}^*)^2=1$, so
Corollary~\ref{cor:directional-gap} predicts
$|g_\varepsilon(r\mathbf{u})-F(r\mathbf{u})|\to 1$ for every finite IMQ
combination $F$. The experiment evaluates this predicted asymptotic behavior; it is
an empirical visualization of the theorem's far-field mechanism.

\paragraph{Setup.}
We draw $N=4{,}000$ training points uniformly (area-measure) from the annulus
$\|\mathbf{x}\|\in[50,100]$, labels $y=g_\varepsilon(\mathbf{x})$.
Two IMQ models are trained: $M{=}50$ and $M{=}200$ learned centers,
shared bandwidth $\varepsilon_{\mathrm{IMQ}}$ trained jointly, Adam for
$5{,}000$ full-batch optimizer updates at learning rates $10^{-3}$ and $5{\times}10^{-4}$.
The training seed is $42$; centers are initialized uniformly on $[-100,100]^2$,
coefficients at zero, and bandwidth at one with parameterization
$\varepsilon_{\mathrm{IMQ}}=\exp(\eta)$. A separate $4{,}000$-point area-uniform
annulus sample uses seed $20260905$ for held-out evaluation.
Both models are evaluated along the ray $r\mathbf{u}$ for $r\in[0,500]$.
The implementation is provided in the supplementary
material; the MLX rerun used $4.24$ and $8.08$\,s of training time, respectively,
on the local Apple Silicon GPU.

\paragraph{Results.}
\begin{table}[h]
\centering
\caption{Mean and max pointwise error on the far tail ($r\ge 400$) for finite
IMQ expansions trained to fit a single Yat atom on an annulus.
The asymptotic limit predicted by Corollary~\ref{cor:directional-gap} is $a=1$.}
\label{tab:directional-tail}
\setlength{\tabcolsep}{6pt}
\begin{tabular}{lccc}
\toprule
Model & $M$ & Mean $|g-F|$ at $r\ge400$ & Max $|g-F|$ at $r\ge400$ \\
\midrule
IMQ-50  & 50  & $1.008$ & $1.008$ \\
IMQ-200 & 200 & $1.007$ & $1.007$ \\
\midrule
Predicted ($a=1$) & --- & $1.000$ & $1.000$ \\
\bottomrule
\end{tabular}
\end{table}

Both fitted expansions approach error $1$ on the evaluated far-field interval,
consistent with Corollary~\ref{cor:directional-gap}. Increasing the width from
$50$ to $200$ leaves the far-field limit unchanged, as the corollary predicts.
The rerun also measures fit quality and parameter scale:
\begin{center}
\small
\begin{tabular}{lrrrrrr}
\toprule
Model & Train MSE & Held-out MSE & Held-out MAE & $\varepsilon_{\IMQ}$ &
$\max_j\|\mathbf v_j\|$ & $\sum_j|a_j|$\\
\midrule
IMQ-50 & $0.2290$ & $0.2318$ & $0.3767$ & $4.2171$ & $130.05$ & $265.24$\\
IMQ-200 & $0.1591$ & $0.1856$ & $0.3145$ & $1.0560$ & $132.46$ & $457.61$\\
\bottomrule
\end{tabular}
\end{center}
These fits retain appreciable error inside the annulus. The experiment establishes
the predicted tail behavior for the fitted expansions; improved annulus fitting is
needed to study extrapolation after accurate interpolation.

\subsection{Causal-LM Depth Diagnostic: Protocol and Full Results}
\label{app:lm-poc}

The exact squared RKHS norm in Sections~\ref{sec:psd}--\ref{sec:capacity} applies to one
shared-$(b,\varepsilon)$ layer. For a fixed preceding network, the pullback theorem
(Theorem~\ref{thm:pullback-rkhs}) defines the corresponding RKHS on the original
input domain. This appendix evaluates whether Yat layers can be optimized in matched
causal language models at $261$M parameters, measuring end-to-end
optimization, held-out behavior, and throughput.

\paragraph{Setup.}
Using the externally public JAX/Flax training codebase described in the provenance
paragraph below, we replace the standard GELU MLP block of a
$261$M decoder-only causal language model ($12$ layers,
$n_{\text{embd}}{=}768$, context $1024$) with a
$\mathbf{x}\to\text{Yat}\to\text{Linear}$ block at the same parameter count.
Both models are trained for $19{,}920$ optimizer steps on $5.22$B C4
tokens~\citep{raffel2019t5,hoffmann2022chinchilla}, with global batch $256$.
We use AdamW~\citep{loshchilov2019adamw} with weight decay $0.01$, $500$
linear-warmup steps, and cosine decay to $5\%$ of the peak learning rate.
The retained main-run configurations use
$\mathrm{lr}{=}0.01$ for GELU and $\mathrm{lr}{=}0.03$ for Yat. We evaluate the
full $\{$pn,sb$\}\times\{+\alpha,+\text{ca}\}$ Yat grid:
\emph{pn} = per-neuron bias, \emph{sb} = bias shared within a layer,
$+\alpha$ = learned output scaling, and $+\text{ca}$ = constant output scaling.
The released implementation represents each hidden coordinate as
\[
u_j(\mathbf x)=\alpha\,
\frac{(\mathbf w_j^\top\mathbf x+b_j)^2}
{\|\mathbf x-\mathbf w_j\|^2+\varepsilon},
\qquad
b_j=\operatorname{softplus}(\beta_j),\quad
\varepsilon=\operatorname{softplus}(\eta),
\]
where the sb variants constrain all $\beta_j$ to one shared scalar, the pn variants
learn one $\beta_j$ per hidden coordinate, and $\eta$ is shared within the layer.
Distance evaluation and the positive transforms are performed in float32, with squared
distances clamped below at zero before adding $\varepsilon$; the activation is then
cast back to the model dtype. The raw bias is initialized at zero, so the initial
positive bias is $\log 2$; $\eta$ is initialized to give
$\varepsilon=10^{-3}$. The $+\alpha$ variants learn the scaling and the
$+\mathrm{ca}$ variants use a fixed scalar. Its numerical value must be matched to
the historical implementation when regenerating the corresponding table rows.
A zero-initialized, bias-free output projection
maps the hidden coordinates back to the residual width.

The sb variants have shared positive $(b,\varepsilon)$ and are therefore finite
expansions of one Mercer kernel at each layer; Proposition~\ref{prop:rkhs-norm} and
the layerwise pullback interpretation apply directly to them. The pn variants use
atom-dependent bias and are not one fixed Mercer-kernel expansion, so the shared-kernel
Gram norm is not assigned to those rows. The scalar scaling is absorbed into the
linear readout coefficients for the shared-kernel representation. The Sobolev theorem
analyzes the explicitly defined Yat stacks; the full transformer additionally contains
attention, normalization, embeddings, and residual connections.
All four variants remain matched empirical
comparisons. Each row is mean$\,\pm\,$sample standard deviation over $3$
independent training seeds. The recovered evaluation notebook implements the
zero-shot task scoring directly; the task-metric histories retain the scores and
sample counts.

\begin{table}[h]
\centering\small
\caption{\textbf{$261$M causal-LM evaluation, $5.2$B C4 tokens, mean$\,\pm\,$sample standard deviation over $3$ training seeds.} Wiki PPL scores $29{,}970$ next tokens; long-range PPL scores $16{,}380$ next tokens in four $4096$-token chunks. The recovered evaluator uses WikiText-2 for the former and prioritizes PG-19 for the latter~\citep{rae2019compressive}; source-selection details are given below. Perplexity is lower-is-better; zero-shot accuracy is higher-is-better.}
\label{tab:lm-poc-261M}
\setlength{\tabcolsep}{7pt}
\textit{Perplexity} \\[-2pt]
\begin{tabular}{lcc}
\toprule
Variant & Wiki PPL $\downarrow$ & Long-range PPL $\downarrow$ \\
\midrule
GELU                 & $44.23_{\pm 1.44}$           & $114.08_{\pm 2.82}$ \\
Yat (pn$+\alpha$)    & $\mathbf{39.04_{\pm 1.03}}$ & $104.20_{\pm 3.40}$ \\
Yat (sb$+\alpha$)    & $39.07_{\pm 0.39}$           & $\mathbf{101.09_{\pm 3.89}}$ \\
Yat (pn$+\mathrm{ca}$) & $42.18_{\pm 0.28}$         & $108.62_{\pm 3.51}$ \\
Yat (sb$+\mathrm{ca}$) & $42.09_{\pm 1.96}$         & $106.51_{\pm 4.08}$ \\
\bottomrule
\end{tabular}

\medskip
\textit{Zero-shot accuracy: language and ARC tasks} \\[-2pt]
\begin{tabular}{lccc}
\toprule
Variant & LAMBADA $\uparrow$ & ARC-E $\uparrow$ & ARC-C $\uparrow$ \\
\midrule
GELU & $23.59_{\pm0.87}\%$ & $\mathbf{33.53_{\pm1.12}\%}$ & $\mathbf{22.47_{\pm0.90}\%}$ \\
Yat (pn$+\alpha$) & $\mathbf{24.16_{\pm0.43}\%}$ & $33.01_{\pm0.80}\%$ & $22.01_{\pm1.67}\%$ \\
Yat (sb$+\alpha$) & $23.80_{\pm1.20}\%$ & $32.62_{\pm0.84}\%$ & $21.39_{\pm0.81}\%$ \\
Yat (pn$+\mathrm{ca}$) & $23.79_{\pm0.71}\%$ & $33.02_{\pm0.43}\%$ & $22.07_{\pm0.79}\%$ \\
Yat (sb$+\mathrm{ca}$) & $23.43_{\pm0.83}\%$ & $32.98_{\pm1.47}\%$ & $22.44_{\pm0.31}\%$ \\
\bottomrule
\end{tabular}

\medskip
\textit{Zero-shot accuracy: commonsense tasks} \\[-2pt]
\begin{tabular}{lccc}
\toprule
Variant & HellaSwag $\uparrow$ & OpenBookQA $\uparrow$ & PIQA $\uparrow$ \\
\midrule
GELU & $29.00_{\pm0.05}\%$ & $27.87_{\pm1.63}\%$ & $61.77_{\pm0.65}\%$ \\
Yat (pn$+\alpha$) & $\mathbf{29.62_{\pm0.52}\%}$ & $28.20_{\pm2.23}\%$ & $\mathbf{62.44_{\pm0.56}\%}$ \\
Yat (sb$+\alpha$) & $29.52_{\pm0.18}\%$ & $\mathbf{28.40_{\pm0.35}\%}$ & $62.37_{\pm0.74}\%$ \\
Yat (pn$+\mathrm{ca}$) & $29.20_{\pm0.39}\%$ & $27.47_{\pm1.40}\%$ & $61.26_{\pm1.07}\%$ \\
Yat (sb$+\mathrm{ca}$) & $\mathbf{29.62_{\pm0.18}\%}$ & $27.20_{\pm0.20}\%$ & $61.50_{\pm0.69}\%$ \\
\bottomrule
\end{tabular}
\end{table}

\paragraph{Reading.}
Across the six zero-shot tasks, the reported means and standard deviations show no
uniform separation between the four Yat variants and GELU. The two $+\alpha$ variants have lower WikiText
perplexity (39.0 versus 44.2) and lower long-range perplexity (101 versus 114) than
GELU in these three-seed runs. Throughput is approximately 655K tokens/s for Yat and
715K tokens/s for GELU, measured at the training batch size on a TPU v6e-8 slice and averaged
over the full run. Thus the tested Yat parameterizations can be optimized across
12 layers, with an approximately 8\% throughput reduction. For the shared-parameter
rows, Proposition~\ref{prop:rkhs-norm} remains applicable separately at each layer.

\paragraph{Compute resources and implementation.} Each $261$M variant was trained on
$5.22\mathrm{B}$ C4 tokens, approximately $20$ training tokens per parameter. Per-variant training
used bfloat16 model computation in JAX/Flax on a Google Cloud TPU v6e-8 slice
(TPU Research Cloud), while numerically sensitive reductions such as logits and
the loss were accumulated in float32. Compute below is reported in TPU-v6e-8
slice-hours: one hour on the full eight-chip slice counts as one slice-hour.
The reported token counts and rounded throughput imply about $2.03$ slice-hours per
GELU run and $2.21$ per Yat run, or $32.65$ across the fifteen main runs.
This is a throughput-based estimate. Measured elapsed times for every main run and
the complete tuning campaign are needed to determine total training compute.
The supplementary validation record checks the throughput arithmetic.
\ifarxiv
The public FlaxChat repository provides the base training and evaluation implementation.
The release's declarative launch specification records the configurations used for the
reported rows, and the public artifact map freezes the $15$ exact checkpoint/run mappings.
The separate evaluation project is
\url{https://wandb.ai/irf-sic/flaxchat-evals}.
\else
The anonymous artifact contains the declarative launch specification; the public build
provides the identity-bearing checkpoint/run map.
\fi
The supplement includes the $15$ per-run task-metric histories for this comparison. Selected complete
eight-task passes provide all $40$ displayed task-metric cells. We compute
the mean and sample standard deviation ($n-1$ denominator) from full-precision seed
metrics and round only the final values to two decimal places.
For $261$M, these are the second complete pass for seed zero and the first complete
pass for seeds one and two. These whole passes match the earlier table's
recorded scores; the historical reason for selecting them remains undocumented.
The recovered notebook resumes runs by model repository and sample limit, without
including a checkpoint revision in the run identifier or model-loading call.
Rerunning the models additionally requires the original training revision and environment
lockfile.

\paragraph{Evaluation protocol.}
The histories record $29{,}970$ scored tokens for Wiki PPL and $16{,}380$ for
long-range PPL, with the latter using four $4096$-token chunks. The recovered notebook
computes Wiki PPL from the first $30{,}000$ tokens of WikiText-2 test, in nonoverlapping
$1024$-token chunks. For long-range PPL it tries PG-19 validation, then PG-19 test,
then WikiText-103 test. The selected source and executed notebook revision were not
recorded with the task scores, so the long-range column is identified by its logged
context protocol rather than assigned a verified dataset split.
The zero-shot counts are $5153$ LAMBADA examples, $2376$ ARC-Easy, $1172$ ARC-Challenge,
$10{,}042$ HellaSwag, $500$ OpenBookQA, and $1838$ PIQA.

\paragraph{Learning-rate sweep.}
The recovered $261$M Yat precursor sweep tests ten peak rates,
$\{10^{-4},3\!\times\!10^{-4},10^{-3},3\!\times\!10^{-3},10^{-2},0.03,0.1,0.3,1,3\}$.
Each completed candidate uses $1000$ optimizer updates, global batch $256$ and context
$1024$: $262{,}144{,}000$ token slots, with padding excluded from the loss.
AdamW uses weight decay $0.01$, gradient clipping at norm $1$, $100$ linear-warmup
updates, and cosine decay toward $5\%$ of the peak rate. Among these candidates,
$0.03$ has the lowest recorded final-minibatch loss ($4.5621$, versus $4.8306$ at
$0.01$), consistent with the retained Yat main-run rate. The script also records
an exponential loss average with coefficient $0.95$; it does not encode an automatic
selection rule. The GELU main-run rate is $0.01$; the recovered sweep evidence does
not establish its selection history or a separate sweep for each Yat variant.

\paragraph{Checkpoint availability.}
External Hugging Face repositories archive all three $261$M checkpoints for every table
row; an external W\&B project retains the training
curves and logged metrics.
\ifarxiv
The public code, model collection, and training histories are available at
\url{https://github.com/mlnomadpy/flaxchat}, \url{https://huggingface.co/mlnomad}, and
\url{https://wandb.ai/irf-sic/flaxchat}, respectively. The supplementary artifact
records immutable row-level links and their verification status.
\else
Identity-bearing code, model, and W\&B links will be supplied in the non-anonymous version.
\fi

\ifarxiv
\clearpage
\section{Public Resources and Model Index}
\label{app:public-resources}
\noindent
The \href{https://github.com/mlnomadpy/flaxchat}{FlaxChat codebase},
\href{https://huggingface.co/mlnomad}{Hugging Face model hub}, and
\href{https://wandb.ai/irf-sic/flaxchat-evals}{evaluation project}
provide entry points to the released language-model resources.
Table~\ref{tab:public-model-index} links the fifteen retained runs individually.
Checkpoint links identify immutable released revisions; training and evaluation
links identify distinct records. The historical linkage between an evaluation pass
and its executed checkpoint remains as described in Appendix~\ref{app:lm-poc}.

\begin{table}[ht]
\centering
\caption{Public resource index for the five $261$M configurations and three seeds.
Each checkpoint label is the first seven characters of its immutable revision.
Training and evaluation links open the corresponding W\&B run.}
\label{tab:public-model-index}
\small
\setlength{\tabcolsep}{9pt}
\begin{tabular}{llccc}
\toprule
Variant & Seed & Checkpoint & Training & Evaluation\\
\midrule
GELU & 0 & \href{https://huggingface.co/mlnomad/gelu-d12-chinchilla-261M-pytorch/tree/116f05b7e907d6048ef21bb04e306065083721dd}{\texttt{116f05b}} & \href{https://wandb.ai/irf-sic/flaxchat/runs/gd6xv37e}{run} & \href{https://wandb.ai/irf-sic/flaxchat-evals/runs/95f2607b68}{history}\\
GELU & 1 & \href{https://huggingface.co/mlnomad/gelu-d12-chinchilla-261M-seed1-pytorch/tree/3d3ab0b762800e2af029abe3d0ce954d44457b4c}{\texttt{3d3ab0b}} & \href{https://wandb.ai/irf-sic/flaxchat/runs/qmdz0itw}{run} & \href{https://wandb.ai/irf-sic/flaxchat-evals/runs/f9ec1fbf94}{history}\\
GELU & 2 & \href{https://huggingface.co/mlnomad/gelu-d12-chinchilla-261M-seed2-pytorch/tree/df539cd1f6f47a74cabb01f9efa5520a0a89080d}{\texttt{df539cd}} & \href{https://wandb.ai/irf-sic/flaxchat/runs/myfqxkpn}{run} & \href{https://wandb.ai/irf-sic/flaxchat-evals/runs/f6f2a63083}{history}\\
\addlinespace
Yat (pn$+\alpha$) & 0 & \href{https://huggingface.co/mlnomad/yatnmn-softplus-d12-chinchilla-261M-pytorch/tree/e813df4e386c2bc15fddb822d7ddb7d1df3b8b8c}{\texttt{e813df4}} & \href{https://wandb.ai/irf-sic/flaxchat/runs/uesk2fq3}{run} & \href{https://wandb.ai/irf-sic/flaxchat-evals/runs/666c1a2a9f}{history}\\
Yat (pn$+\alpha$) & 1 & \href{https://huggingface.co/mlnomad/yatnmn-softplus-d12-chinchilla-261M-seed1-pytorch/tree/74d3eb04a1d7c0297cc4164ddf0e6aae507932b9}{\texttt{74d3eb0}} & \href{https://wandb.ai/irf-sic/flaxchat/runs/ijoumh6z}{run} & \href{https://wandb.ai/irf-sic/flaxchat-evals/runs/e10b4463ea}{history}\\
Yat (pn$+\alpha$) & 2 & \href{https://huggingface.co/mlnomad/yatnmn-softplus-d12-chinchilla-261M-seed2-pytorch/tree/f56e95238223c9f19424d4263c61c47dc11c1ccf}{\texttt{f56e952}} & \href{https://wandb.ai/irf-sic/flaxchat/runs/d6k18imn}{run} & \href{https://wandb.ai/irf-sic/flaxchat-evals/runs/0b648d3e09}{history}\\
\addlinespace
Yat (sb$+\alpha$) & 0 & \href{https://huggingface.co/mlnomad/yatnmn-softplus-sb-d12-chinchilla-261M-pytorch/tree/677ea39d4a992825a9d4b1df98ea369fa726c5aa}{\texttt{677ea39}} & \href{https://wandb.ai/irf-sic/flaxchat/runs/b1ool4vs}{run} & \href{https://wandb.ai/irf-sic/flaxchat-evals/runs/7f7b10ef9b}{history}\\
Yat (sb$+\alpha$) & 1 & \href{https://huggingface.co/mlnomad/yatnmn-softplus-sb-d12-chinchilla-261M-seed1-pytorch/tree/c0f5ea1785c3dba6e18701b8e8c9be138eb4e81c}{\texttt{c0f5ea1}} & \href{https://wandb.ai/irf-sic/flaxchat/runs/5h37m8in}{run} & \href{https://wandb.ai/irf-sic/flaxchat-evals/runs/71c8b2304d}{history}\\
Yat (sb$+\alpha$) & 2 & \href{https://huggingface.co/mlnomad/yatnmn-softplus-sb-d12-chinchilla-261M-seed2-pytorch/tree/da2ca189597a2500eda167d1ba702f7d7a37c584}{\texttt{da2ca18}} & \href{https://wandb.ai/irf-sic/flaxchat/runs/vl7wj9a4}{run} & \href{https://wandb.ai/irf-sic/flaxchat-evals/runs/936d1174e3}{history}\\
\addlinespace
Yat (pn$+\mathrm{ca}$) & 0 & \href{https://huggingface.co/mlnomad/yatnmn-softplus-ca-d12-chinchilla-261M-pytorch/tree/0074b2b24815201a6aea848a5b5ed500aafffc33}{\texttt{0074b2b}} & \href{https://wandb.ai/irf-sic/flaxchat/runs/vojh88ls}{run} & \href{https://wandb.ai/irf-sic/flaxchat-evals/runs/fad8fc8c45}{history}\\
Yat (pn$+\mathrm{ca}$) & 1 & \href{https://huggingface.co/mlnomad/yatnmn-softplus-ca-d12-chinchilla-261M-seed1-pytorch/tree/03655837e5dff7a45f55a4f1fd97781fb456319f}{\texttt{0365583}} & \href{https://wandb.ai/irf-sic/flaxchat/runs/qoi4e2hs}{run} & \href{https://wandb.ai/irf-sic/flaxchat-evals/runs/8aca7502f2}{history}\\
Yat (pn$+\mathrm{ca}$) & 2 & \href{https://huggingface.co/mlnomad/yatnmn-softplus-ca-d12-chinchilla-261M-seed2-pytorch/tree/a5fe50fb8b83e9ce553abca0668ad3d37c7a0aca}{\texttt{a5fe50f}} & \href{https://wandb.ai/irf-sic/flaxchat/runs/9owo2b9y}{run} & \href{https://wandb.ai/irf-sic/flaxchat-evals/runs/d7b955d495}{history}\\
\addlinespace
Yat (sb$+\mathrm{ca}$) & 0 & \href{https://huggingface.co/mlnomad/yatnmn-softplus-sb-ca-d12-chinchilla-261M-pytorch/tree/105af05258cbb8c8d63c33a7773499c6872b365a}{\texttt{105af05}} & \href{https://wandb.ai/irf-sic/flaxchat/runs/j57dlijv}{run} & \href{https://wandb.ai/irf-sic/flaxchat-evals/runs/e756f707ca}{history}\\
Yat (sb$+\mathrm{ca}$) & 1 & \href{https://huggingface.co/mlnomad/yatnmn-softplus-sb-ca-d12-chinchilla-261M-seed1-pytorch/tree/9203b20a512b70fa16f9a266a9e60df940650727}{\texttt{9203b20}} & \href{https://wandb.ai/irf-sic/flaxchat/runs/i25lwkhh}{run} & \href{https://wandb.ai/irf-sic/flaxchat-evals/runs/d65dbede87}{history}\\
Yat (sb$+\mathrm{ca}$) & 2 & \href{https://huggingface.co/mlnomad/yatnmn-softplus-sb-ca-d12-chinchilla-261M-seed2-pytorch/tree/65a88a12956843b6e88de2ab5db0425c3ed17138}{\texttt{65a88a1}} & \href{https://wandb.ai/irf-sic/flaxchat/runs/3rxabp1q}{run} & \href{https://wandb.ai/irf-sic/flaxchat-evals/runs/608823698a}{history}\\
\bottomrule
\end{tabular}
\end{table}

\paragraph{Using the resources.}
The released configurations and metric histories support inspection of the reported
training setup and reconstruction of the table aggregates. A new evaluation should
pin model, tokenizer, evaluator, and dataset revisions and record them together.
The kernel proofs and notation are self-contained in Appendices~\ref{app:background}--\ref{app:layer-theory}.

\acksection
We acknowledge the Google TPU Research Cloud program for providing computational resources.
\else

\section*{NeurIPS Paper Checklist}

\begin{enumerate}

\item {\bf Claims}

\item[] Question: Do the main claims made in the abstract and introduction accurately reflect the paper's contributions and scope?

\item[] Answer: \answerYes{}

\item[] Justification: The introduction's three numbered kernel contributions are
PSD/Mercer structure with Loewner-domination universality, nonradial alignment via the
directional trace, covariance limit, and coefficient-budget obstruction, and exact IMQ recovery with
three-atom tightness. They correspond to Proposition~\ref{prop:inherited-guarantees},
Theorem~\ref{thm:activation-covariance}, Theorem~\ref{thm:imq-embed}, and
Theorem~\ref{thm:main-atom-separation}. The subsequent trained-layer
control is Proposition~\ref{prop:rkhs-norm} and Theorem~\ref{thm:rademacher}. For deep
stacks, the paper proves the prefix-conditioned pullback and ambient Sobolev containment.

\item {\bf Limitations}

\item[] Question: Does the paper discuss the limitations of the work performed by the authors?

\item[] Answer: \answerYes{}

\item[] Justification: Section~\ref{sec:discussion} states the domain of every central
result: shared $(b,\varepsilon)$ for the exact squared RKHS norm, qualitative
fixed-kernel universality, coefficient variation for the IMQ obstruction, Gaussian and
fixed-norm spherical inputs for the covariance theorems, and exterior shells for extrapolation. It
also identifies faithful intervention and matched-quality systems evaluation as the two
next empirical questions. The artifact section records checkpoint and training-run
coverage for all fifteen runs.

\item {\bf Theory assumptions and proofs}

\item[] Question: For each theoretical result, does the paper provide the full set of assumptions and a complete (and correct) proof?

\item[] Answer: \answerYes{}

\item[] Justification: Every theorem, proposition, lemma, and corollary states its assumptions inline ($b\ge 0$, $\varepsilon>0$, compactness of $\mathcal{X}$, etc.); notation is consolidated in Appendix~\ref{app:notation}, generic RKHS background in Appendix~\ref{sec:prelim}, and named results from the literature in Appendix~\ref{app:bg}. Proofs and supporting results are organized by distance (Appendix~\ref{app:distance-proofs}), alignment (Appendix~\ref{app:alignment-proofs}), constructive representation (Appendix~\ref{app:representation-proofs}), and layer theory (Appendix~\ref{app:layer-theory}).

\item {\bf Experimental result reproducibility}

\item[] Question: Does the paper fully disclose all the information needed to reproduce the main experimental results of the paper?

\item[] Answer: \answerNo{}

\item[] Justification: The Gaussian/fixed-norm spectrum, frozen-CLIP scale, fixed-parameter CLIP
norm--accuracy, and directional-tail diagnostics have supplied implementations and
machine-readable result records.
The causal-LM appendix reports the model architecture, optimizers, schedules, seed counts,
token budgets, datasets, JAX/Flax implementation, TPU v6e-8 slice, and bfloat16
model-compute precision. The public artifact map freezes $15$ exact checkpoint/run
mappings, and the anonymous-safe launch specification records the reported
configurations. The $15$ evaluation histories reproduce all $40$ displayed
task-metric cells. End-to-end reruns still require the historical code and environment,
and verified evaluation revision and checkpoint linkage.

\item {\bf Open access to data and code}

\item[] Question: Does the paper provide open access to the data and code, with sufficient instructions to faithfully reproduce the main experimental results, as described in supplemental material?

\item[] Answer: \answerNo{}

\item[] Justification: Reproduction scripts are supplied for the Gaussian/fixed-norm spectrum,
frozen-CLIP scale, fixed-parameter CLIP norm--accuracy, and directional-tail diagnostics,
and the CLIP features come from a public ViT-B/32 checkpoint and the public ImageNet-1k
validation set. The CLIP scripts fail loudly on dataset errors; synthetic smoke tests
require an explicit flag and mark their outputs as non-evidentiary. The causal-LM
training codebase, checkpoints, and W\&B metric histories are externally available.
Identity-bearing links and the $15$-run mapping are supplied only with the public
build; the anonymous supplement retains the declarative configuration record and
per-run task-metric histories.

\item {\bf Experimental setting/details}

\item[] Question: Does the paper specify all the training and test details (e.g., data splits, hyperparameters, how they were chosen, type of optimizer, etc.) necessary to understand the results?

\item[] Answer: \answerNo{}

\item[] Justification: Appendix~\ref{app:spectral-protocol} specifies the
Gaussian/fixed-norm spectrum diagnostic, and Appendix~\ref{app:clip} specifies the two retained
frozen-feature diagnostics, including the trained head's actual width, fixed kernel
parameters, optimizer, epoch count, seed, and evaluation set.
Appendix~\ref{app:lm-poc} and the supplement specify the reported language-model
configurations. Exact historical reconstruction additionally requires
the original code revision, environment, evaluated-checkpoint linkage, and
evaluation pass-selection rationale. Per-seed task-metric histories are supplied.

\item {\bf Experiment statistical significance}

\item[] Question: Does the paper report error bars suitably and correctly defined or other appropriate information about the statistical significance of the experiments?

\item[] Answer: \answerNo{}

\item[] Justification: The retained CLIP norm--accuracy result is a single recorded run, and its $\rho\in[0.54,0.56]$ parameter-grid range reports parameter sensitivity. The $261$M causal-LM tables report mean $\pm$ standard deviation across three seeds.

\item {\bf Experiments compute resources}

\item[] Question: For each experiment, does the paper provide sufficient information on the computer resources (type of compute workers, memory, time of execution) needed to reproduce the experiments?

\item[] Answer: \answerNo{}

\item[] Justification: The Gaussian and fixed-norm spectrum diagnostic uses CPU;
the frozen-CLIP scale diagnostic used a CPU-only Kaggle worker. The directional-tail
rerun took $4.24$ and $8.08$ seconds of training on Apple Silicon for its two models.
The trained CLIP probe uses one Apple Silicon accelerator at $d=512$.
The causal-LM studies use TPU v6e-8 slices and bfloat16 model computation.
Rounded throughput implies approximately $32.65$ slice-hours across the fifteen
$261$M main runs; measured elapsed times and total tuning compute remain unresolved.
Exact reconstruction also requires the historical software environment.

\item {\bf Code of ethics}

\item[] Question: Does the research conducted in the paper conform, in every respect, with the NeurIPS Code of Ethics?

\item[] Answer: \answerYes{}

\item[] Justification: The paper is a theoretical contribution about a kernel construction; no human subjects, no personal data, no harmful applications.

\item {\bf Broader impacts}

\item[] Question: Does the paper discuss both potential positive societal impacts and negative societal impacts of the work performed?

\item[] Answer: \answerNA{}

\item[] Justification: The paper develops mathematical foundations for a hidden-unit primitive; broader societal impact is not directly applicable to this contribution.

\item {\bf Safeguards}

\item[] Question: Does the paper describe safeguards that have been put in place for responsible release of data or models with a high risk for misuse?

\item[] Answer: \answerNA{}

\item[] Justification: The release consists of research-scale checkpoints trained on
public corpora and contains no restricted or personal dataset.

\item {\bf Licenses for existing assets}

\item[] Question: Are the creators or original owners of assets (e.g., code, data, models), used in the paper, properly credited and are the license and terms of use explicitly mentioned and properly respected?

\item[] Answer: \answerYes{}

\item[] Justification: The CLIP ViT-B/32 model~\citep{radford2021clip} is
MIT-licensed, and ImageNet-1k is used under its research-only terms; we use
precomputed embeddings and do not redistribute raw images. C4~\citep{raffel2019t5}
is used under ODC-BY together with the applicable Common Crawl terms, WikiText
under CC BY-SA, and PG-19~\citep{rae2019compressive} under Apache-2.0. Each asset is cited
at its point of use, and no third-party raw data or model checkpoint is redistributed.

\item {\bf New assets}

\item[] Question: Are new assets introduced in the paper well documented and is the documentation provided alongside the assets?

\item[] Answer: \answerYes{}

\item[] Justification: The authors release trained causal-LM checkpoints, public
training histories, row-level immutable checkpoint/run mappings, an aggregate manifest,
a declarative launch specification, a validator, and diagnostic code. The public
artifact statement records all $15$ paper-time checkpoint mappings and the inputs
remaining for exact historical regeneration; model cards document the released checkpoints.

\item {\bf Crowdsourcing and research with human subjects}

\item[] Question: For crowdsourcing experiments and research with human subjects, does the paper include the full text of instructions given to participants and screenshots, if applicable, as well as details about compensation (if any)?

\item[] Answer: \answerNA{}

\item[] Justification: No human subjects.

\item {\bf Institutional review board (IRB) approvals or equivalent for research with human subjects}

\item[] Question: Does the paper describe potential risks incurred by study participants, whether such risks were disclosed to the subjects, and whether Institutional Review Board (IRB) approvals (or an equivalent approval/review based on the requirements of your country or institution) were obtained?

\item[] Answer: \answerNA{}

\item[] Justification: No human subjects.

\item {\bf Declaration of LLM usage}

\item[] Question: Does the paper describe the usage of LLMs if it is an important, original, or non-standard component of the core methods in this research?

\item[] Answer: \answerNA{}

\item[] Justification: LLMs are not part of the methodology of this paper.

\end{enumerate}
\fi

\end{document}